\documentclass{article}

\usepackage[final]{neurips_2023}

\usepackage{natbib}
\setcitestyle{authoryear}
\usepackage{hyperref}
\hypersetup{
  colorlinks   = true, %Colours links instead of ugly boxes
    linkcolor=red,
    citecolor={blue!50!black},
    urlcolor={blue!80!black}
}

\usepackage{multirow}
\usepackage{makecell}

\usepackage[utf8]{inputenc} % allow utf-8 input
\usepackage[T1]{fontenc}    % use 8-bit T1 fonts
\usepackage{hyperref}       % hyperlinks
\usepackage{url}            % simple URL typesetting
\usepackage{booktabs}       % professional-quality tables
\usepackage{amsfonts}       % blackboard math symbols
\usepackage{nicefrac}       % compact symbols for 1/2, etc.
\usepackage{microtype}      % microtypography
\usepackage{xcolor}         % colors
\usepackage{algorithm}
\usepackage{algorithmic}
\usepackage{bm}
\usepackage{makecell}
\usepackage{amsthm}
\usepackage{amssymb}
\usepackage{amsmath}

\usepackage{wrapfig}
\usepackage{graphicx}
\usepackage{subfigure}
\theoremstyle{plain}
\newtheorem{theorem}{Theorem}

\newtheorem{lemma}{Lemma}
\newtheorem{corollary}{Corollary}
\theoremstyle{definition}
\newtheorem{definition}{Definition}

\newtheorem{problem}{Problem}
\usepackage{tcolorbox}
\theoremstyle{remark} 
\newtheorem{remark}[theorem]{Remark}
\DeclareMathOperator*{\argmin}{arg\,min}

\usepackage{colortbl,color}
\definecolor{greyC}{RGB}{180,180,180}
\definecolor{greyL}{RGB}{235,235,235}

\definecolor{citeColor}{RGB}{0,20,115}
\hypersetup{colorlinks,linkcolor={red},citecolor={citeColor},urlcolor={citeColor}}

\title{Learning to Augment Distributions for Out-of-Distribution Detection}

\author{%
Qizhou Wang$^{1}$\thanks{Equal contributions.} 
\quad Zhen Fang$^{2}$\footnotemark[1] \quad Yonggang Zhang$^{1}$ \quad Feng Liu$^{3}$ \quad {Yixuan Li$^{4}$ \quad Bo Han$^{1}\thanks{Correspondence to Bo Han (bhanml@comp.hkbu.edu.hk).}$ } \\
  $^1$Department of Computer Science, Hong Kong Baptist University \\
  $^2$Australian Artificial Intelligence Institute, University of Technology Sydney \\
  $^3$School of Computing and Information Systems, The University of Melbourne \\
  $^4$Department of Computer Sciences, University of Wisconsin-Madison \\ 
  \textnormal{\{csqzwang, csygzhang, bhanml\}@comp.hkbu.edu.hk} \\  \textnormal{zhen.fang@uts.edu.au} \quad \textnormal{fengliu.ml@gmail.com} \quad \textnormal{sharonli@cs.wisc.edu}
}

\begin{document}

\maketitle

\begin{abstract}
Open-world classification systems should discern out-of-distribution (OOD) data whose labels deviate from those of in-distribution (ID) cases, motivating recent studies in OOD detection. Advanced works, despite their promising progress, may still fail in the open world, owing to the lack of knowledge about unseen OOD data in advance. Although one can access auxiliary OOD data (distinct from unseen ones) for model training, it remains to analyze how such auxiliary data will work in the open world. 
To this end, we delve into such a problem from a learning theory perspective, finding that the distribution discrepancy between the auxiliary and the unseen real OOD data is the key to affecting the open-world detection performance. Accordingly, we propose \emph{Distributional-Augmented OOD Learning} (DAL), alleviating the OOD distribution discrepancy by crafting an \emph{OOD distribution set} that contains all distributions in a Wasserstein ball centered on the auxiliary OOD distribution. We justify that the predictor trained over the worst OOD data in the ball can shrink the OOD distribution discrepancy, thus improving the open-world detection performance given only the auxiliary OOD data. We conduct extensive evaluations across representative OOD detection setups, demonstrating the superiority of our DAL over its advanced counterparts.  The code is publicly available at: \url{https://github.com/tmlr-group/DAL}.
\end{abstract}

\section{Introduction}

Deep learning in the open world often encounters {out-of-distribution} (OOD) data of which the label space is disjoint with that of the {in-distribution} (ID) cases~\citep{hendrycks2016baseline,fang2022learnable}. It leads to the well-known {OOD detection} problem, where the predictor should make accurate predictions for ID data and detect anomalies from OOD cases~\citep{bulusu2020anomalous,yang2021generalized}. Nowadays, OOD detection has attracted intensive attention in reliable machine learning due to its integral role in safety-critical applications~\citep{cao2020benchmark,shen2021towards}.

OOD detection remains challenging since predictors can make over-confidence predictions for OOD data~\citep{HendrycksMD19}, motivating recent studies towards effective OOD detection. Therein, {outlier exposure}~\citep{HendrycksMD19,MingFL22} is among the most potent ones, learning from \emph{auxiliary OOD data} to discern ID and OOD patterns. However, due to the openness of the OOD task objective~\citep{wang2023out}, auxiliary OOD data can arbitrarily differ from the (unseen) real OOD data in the open world. So, to formally understand their consequences, we model the difference

\begin{wrapfigure}{R}{0.6\linewidth}
    \centering
    \includegraphics[width=.95\linewidth]{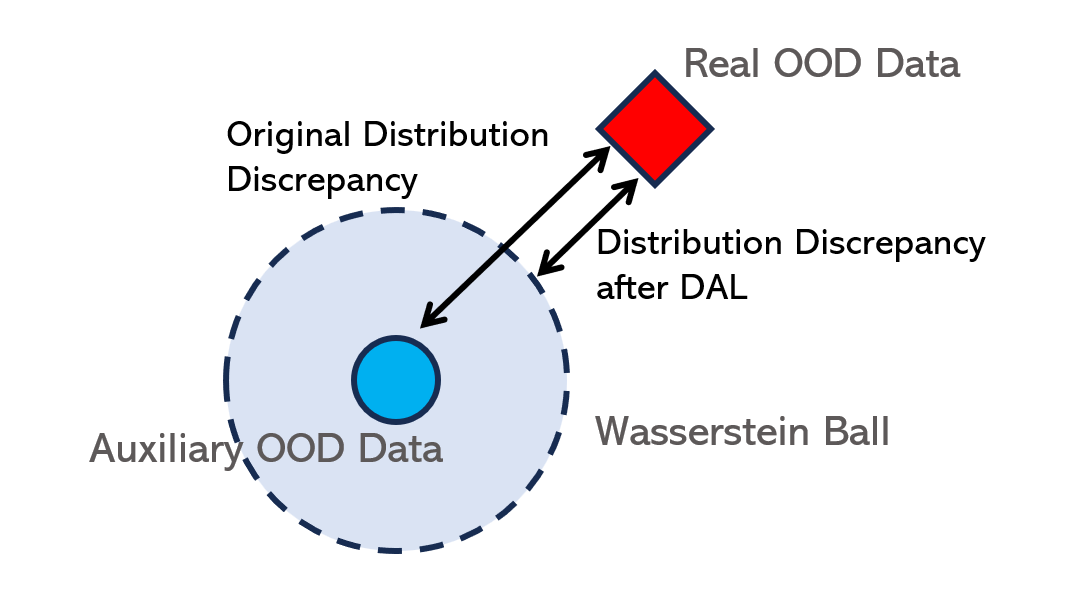}
    % \vspace{-1.5em}
    \vspace{1.5pt}
    \caption{A heuristic illustration for our DAL. A large distribution discrepancy between the auxiliary and the unseen OOD data will hurt the real detection effectiveness. However, by ensuring uniformly well performance inside the Wasserstein ball, we can mitigate the distribution discrepancy and thus improve the detection power in the open world. } \label{fig: motivation}
\end{wrapfigure}

between auxiliary and real OOD data by their distribution discrepancy, measured by the Wasserstein distance~\citep{Villani2003TopicsIO,Villani2008OptimalTO}. Then, we reveal the negative impacts of such OOD distribution discrepancy on the real detection power, with a larger distribution discrepancy indicating a lower performance on real OOD data, cf., Eq.~\eqref{eq: motivate}.

The OOD distribution discrepancy threatens the open-world detection performance for outlier exposure. Therefore, we raise a natural question in \emph{how to alleviate such an OOD distribution discrepancy}. Hence, this paper establishes a promising learning framework named \emph{Distributional-Augmented OOD Learning} (DAL). Therein, we augment the auxiliary OOD distribution by crafting an \emph{OOD distribution set} containing all distributions in a Wasserstein ball~\citep{Villani2003TopicsIO,Villani2008OptimalTO}, centered on the auxiliary OOD distribution. Then, by making the predictor learn from the worst OOD distribution in the set, cf., Eq.~\eqref{Eq::Objective}, one can alleviate the negative impacts of the distribution discrepancy. Moreover, our proposed framework enjoys the learning guarantees towards the expected risk with respect to the real OOD distribution, making OOD detection stay effective when facing unseen data (cf., Theorem~\ref{T3}). Figure~\ref{fig: motivation} provides a conceptual explanation: learning from the worst OOD distribution ensures the uniformly well performance inside the Wasserstein ball, enlarging the influence of the auxiliary OOD distribution. Thus, one can shrink the OOD distribution discrepancy between the auxiliary and the real OOD data and improve OOD detection.

In realization, the primal learning objective in Eq.~\eqref{Eq::Objective} is generally intractable due to the infinite-dimensional optimization for the worst OOD distribution search. Instead, we adopt the dual form with respect to the original learning problem (cf., Theorem~\ref{T-dual}), transforming it into a tractable problem of the worst OOD data search in a finite-dimensional space. Furthermore, following~\cite{du2022vos,mehra2022certifying}, the data search procedure is conducted in the embedding space, which can benefit the open-world performance of OOD detection with decent costs of additional computation.

We conduct extensive experiments over representative OOD detection setups, revealing the open-world performance of our method toward effective OOD detection. For example, our DAL reduces the average FPR95 by {1.99 to 13.46} on CIFAR benchmarks compared with the conventional outlier exposure~\citep{HendrycksMD19}. Overall, we summarize our contributions into three folds:
\begin{itemize}
    \item We measure the difference between the auxiliary and the real OOD data by the Wasserstein distance, and establish an effective learning framework, named DAL, to mitigate the OOD distribution discrepancy issue. We further guarantee our performance with respect to unseen real OOD data via Theorem~\ref{T3}, which is new to previous works. 
    
    \item DAL leads to a practical method in Algorithm~\ref{alg: DOE}, learning from the worst cases in the Wasserstein ball to improve the open-world detection performance. Overall, our method solves the dual problem, which performs the worst-case search in the embedding space, which is simple to compute yet effective in OOD detection.
    
    \item We conduct extensive experiments in Section~\ref{sec: experiment} to evaluate our effectiveness, ranging from the well-known CIFAR benchmarks to the challenging ImageNet settings. The empirical results comprehensively demonstrate our superiority over advanced counterparts, and the improvement is mainly attributed to our distributional-augmented learning framework.  
\end{itemize}

A detailed overview of existing OOD detection methods and theories can be found in
Appendix~\ref{app: related works}, and a summary of the important notations can be found in Appendix~\ref{app: notations}.

\section{Outlier Exposure} \label{sec: Setup}

%According to \citet{fang2022learnable}, 
Let $\mathcal{X}$ denote the feature space and $\mathcal{Y}=\{1,\ldots, C\}$ denote the label space with respect to the ID distribution. We consider the ID distribution $D_{X_{\rm I}Y_{\rm I}}$, a joint distribution defined over 
$\mathcal{X}\times \mathcal{Y}$, where $X_{\rm I}$ and ${Y}_{\rm I}$ are random variables whose outputs are from spaces $\mathcal{X}$ and $\mathcal{Y}$.  We also have an OOD joint distribution  $D_{X_{\rm O}Y_{\rm O}}$,  where $X_{\rm O}$ is a random variable from $\mathcal{X}$, but $Y_{\rm O}$ is a random variable whose outputs do not belong to $\mathcal{Y}$, i.e., $Y_{\rm O}\notin \mathcal{Y}$ \citep{fang2022learnable}. %Our goal is find the proper model $f = h \circ g:\mathcal{X}\rightarrow\mathbb{R}^C$ parameterized by $\mathcal{W}$, where $g$ is the feature extractor and $h$ is the classifier. 

The classical OOD detection \citep{hendrycks2016baseline,yang2021generalized} typically considers an open-world setting, where the real OOD data drawn from  $D_{X_{\rm O}Y_{\rm O}}$ are unseen during training. Recently, \citet{fang2022learnable}
have provided several \textit{strong} conditions necessary to ensure the success of the classical OOD setting. Furthermore, to increase the possibility of success for OOD detection and weaken the strong conditions proposed by \citet{fang2022learnable}, advanced works~\citep{HendrycksMD19,chen2021atom} introduce a promising approach named \emph{outlier exposure}, where a set of auxiliary OOD data is employed as a surrogate of real OOD data. Here, we provide a formal definition.

% To increase the possibility of success for OOD detection, researchers~\cite{HendrycksMD19,chen2021robustifying} have considered a novel setting called outlier exposure, where the {auxiliary} OOD data are available during the training. Here, we provide a formal definition. 

\begin{problem}[OOD Detection with Outlier Exposure]\label{P1}
Let $D_{X_{\rm I}Y_{\rm I}}$, $D_{X_{\rm O}}$, and $D_{X_{\rm A}}$ be the ID joint distribution, the OOD distribution, and the auxiliary OOD distribution, respectively. Given the sets of samples called the ID and the auxiliary OOD data, namely,
\begin{equation*}
\begin{split}
S &= \{(\mathbf{x}_{\rm I}^1,y_{\rm I}^1),...,(\mathbf{x}_{\rm I}^n,y_{\rm I}^n)\}\sim D^n_{X_{\rm I}Y_{\rm I}}, ~i.i.d.,
~~~~~
T = \{\mathbf{x}_{\rm A}^1,...,\mathbf{x}_{\rm A}^m\}\sim D^m_{X_{\rm A}}, ~i.i.d.,
\end{split}
\end{equation*}
outlier exposure trains a predictor $\mathbf{f}$ by using the training data $S$ and $T$, such that for any test data $\mathbf{x}$:
1) if $\mathbf{x}$ is an observation from $D_{X_{\rm I}}$, the predictor $\mathbf{f}$ can classify $
\mathbf{x}$ into its correct ID label; otherwise
2) if $\mathbf{x}$ is an observation from $D_{X_{\rm O}}$, the predictor $\mathbf{f}$ can detect $
\mathbf{x}$ as an OOD case.
\end{problem}

\textbf{OOD Scoring.} Many existing methods detect OOD data by using various score-based strategies \citep{hendrycks2016baseline,lee2018simple,liu2020energy,SunM0L22}. In general, given a  model $\mathbf{f}: \mathcal{X}\rightarrow \mathbb{R}^C$ and a scoring function $s(\cdot ;\mathbf{f}):\mathcal{X} \rightarrow \mathbb{R}$, the OOD detector $g_{\lambda}$ is given by:
\begin{equation*}
  g_{\lambda}(\mathbf{x}) = {\rm ID},~\textnormal{ if}~s(\mathbf{x};\mathbf{f})\geq \lambda;~\text{otherwise}, ~
    g_{\lambda}(\mathbf{x}) = {\rm OOD},
\end{equation*}
%\begin{equation*}
 % g_{\lambda}(\mathbf{x}) = \left \{
 % \begin{aligned}
%   &~~~~~~~~~{\rm ID},~~\textnormal{ if}~s(\mathbf{x};\mathbf{f})\geq \lambda\\ 
%   & ~~~~ {\rm OOD},~~\textnormal{ if}~s(\mathbf{x};\mathbf{f})< \lambda
 %   \end{aligned}
%    \right.~,
%\end{equation*}
where $\lambda$ is a given threshold. For example, as a well-known baseline scoring function, the maximum softmax prediction (MSP)~\citep{hendrycks2016baseline} is given by:
\begin{equation}
    s_\text{MSP}(\mathbf{x};\mathbf{f}) = \max_{k\in \mathcal{Y}}~\texttt{softmax}_k~\mathbf{f} (\mathbf{x}),  \label{eq: softmax score}
\end{equation}
with $\texttt{softmax}_k (\cdot)$ denoting the $k$-th dimension of the softmax output.

\textbf{Model and Risks.} We denote $\mathbf{f}_\mathbf{w}:\mathcal{X}\rightarrow \mathbb{R}^C$ the predictor with parameters $\mathbf{w}\in\mathcal{W}$, with $\mathcal{W}$ the parameter space. We consider the loss functions $\ell$ and $\ell_{\rm OE}$ w.r.t. the ID and the OOD cases, respectively. Then, the expected and the empirical \emph{ID risks} of the model $\mathbf{f}_{\mathbf{w}}$ can be written as:
\begin{equation*}
R_{{\rm I}} (\mathbf{w}) = \mathbb{E}_{(\mathbf{x},y)\sim D_{X_{\rm I}Y_{\rm I}}} \ell(\mathbf{f}_{\mathbf{w}}; \mathbf{x},y)~~\text{and}~~\widehat{R}_{{\rm I}} (\mathbf{w}) = \frac{1}{n}\sum_{i=1}^n\ell(\mathbf{f}_{\mathbf{w}}; \mathbf{x}_{\rm I}^i,y_{\rm I}^i).
\end{equation*}
The expected and the empirical \emph{auxiliary OOD risks} are then given by
\begin{equation*}
R_{{\rm A}} (\mathbf{w}) = \mathbb{E}_{\mathbf{x}\sim D_{X_{\rm A}}} \ell_{\rm OE}(\mathbf{f}_{\mathbf{w}}; \mathbf{x})~~\text{and}~~\widehat{R}_{{\rm A}} (\mathbf{w}) = \frac{1}{m} \sum_{i=1}^m \ell_{\rm OE}(\mathbf{f}_{\mathbf{w}}; \mathbf{x}_{\rm A}^i),
\end{equation*}
and the expected \emph{real OOD risk} is given by
$
R_{{\rm O}} (\mathbf{w}) = \mathbb{E}_{\mathbf{x}\sim D_{X_{\rm O}}} \ell_{\rm OE}(\mathbf{f}_{\mathbf{w}}; \mathbf{x}).
$
Accordingly, we can define the expected {\emph{detection risk}} with respect to real OOD data, following
\begin{equation}\label{alpharisk}
R_{D}(\mathbf{w}) = R_{{\rm I}} (\mathbf{w}) + \alpha R_{{\rm O}}(\mathbf{w}),
\end{equation}
where $\alpha$ is the trade-off parameter. 

\textbf{Learning Strategy.} After the scoring function is selected, one can obtain the OOD detector if the model $\mathbf{f}_\mathbf{w}$ is given. Under the Problem~\ref{P1} of outlier exposure, a common learning strategy is to optimize the empirical ID and auxiliary OOD risk jointly \citep{HendrycksMD19}, namely,
\begin{equation}
    \min_{\mathbf{w}\in \mathcal{W}} ~\big [ \widehat{R}_{{\rm I}} (\mathbf{w}) + \alpha \widehat{R}_{{\rm A}}(\mathbf{w})\big ]. \label{eq::oe}
\end{equation}
Note that the auxiliary OOD data are employed in Eq.~\eqref{eq::oe}, which can arbitrarily differ from the real OOD cases. Then, it is generally expected that the predictor $\mathbf{f}_\mathbf{w}$, trained over the auxiliary OOD data, can perform well even on unseen OOD data, i.e., a small value of $R_{D}(\mathbf{w})$ is expected.

\section{Motivation} \label{sec: motivation}
To the general learning strategy in Eq. \eqref{eq::oe}, intuitively, if the auxiliary data are sampled from a distribution similar to real ones, the predictor will perform well for real OOD data. However, auxiliary and real OOD data differ in practice, posing us to suspect their open-world detection performance. To formally study the problem, we measure the difference between auxiliary and real OOD data in the distribution level, motivating our discussion of \emph{OOD distribution discrepancy}. 

\textbf{Distribution Discrepancy.} In this paper, we adopt a classical measurement for the distribution discrepancy---Optimal Transport Cost \citep{SinhaND18,mehra2022certifying}.

\begin{definition}[Optimal Transport Cost and Wasserstein-1 Distance \citep{Villani2003TopicsIO,Villani2008OptimalTO}]\label{D1}
Given a cost function $c:\mathcal{X}\times \mathcal{X} \rightarrow \mathbb{R}_{+}$, the \textit{optimal transport cost} between two distributions $D$ and $D'$ is
\begin{equation*}
 {\rm W}_c(D,D') = \inf_{\pi\in\Pi(D,D')}\mathbb{E}_{(\mathbf{x}, \mathbf{x}')\sim \pi} c(\mathbf{x}, \mathbf{x}'),
\end{equation*}
where $\Pi(D,D')$ is the space of all couplings for $D$ and $D'$. Furthermore, if the cost $c$ is a \textit{metric}, then the optimal transport cost is also called the  \textit{Wasserstein-1} distance.
\end{definition}

Based on Definition~\ref{D1}, we use the distribution discrepancy to measure the difference between the auxiliary and the real OOD data, namely, ${\rm W}_c(D_{X_{\rm O}},D_{X_{\rm A}})$. Then, we can formally study the impacts of such a discrepancy on the detection performance of the predictor. Under certain assumptions (cf., Corollary \ref{T2}), we can prove that with high probability, the following generalization bound holds:
\begin{equation}
R_D(\widehat{\mathbf{w}}) \leq \min_{\mathbf{w} \in \mathcal{W}}\left({R}_{{\rm I}} (\mathbf{w}) + \alpha {R}_{{\rm A}}(\mathbf{w})\right) + \alpha L_{c} {\rm W}_c(D_{X_{\rm O}},D_{X_{\rm A}}) + {\mathcal{O}}(1/ \sqrt{n})+{\mathcal{O}}(1/ \sqrt{m}), \label{eq: motivate}
\end{equation}
where $\widehat{\mathbf{w}}$ is the parameter learned by Eq. \eqref{eq::oe}, i.e.,  $\widehat{\mathbf{w}} \in \argmin_{\mathbf{w}\in \mathcal{W}} ~\widehat{R}_{{\rm I}} (\mathbf{w}) + \alpha \widehat{R}_{{\rm A}}(\mathbf{w})$,  $L_c$ is the Lipschitz constant of $\ell_{\rm OE}$ w.r.t. the cost function $c(\cdot,\cdot)$ (see Theorem \ref{T3}). In general, the expected detection risk $R_D(\widehat{\mathbf{w}})$ measures the expected performance on unseen OOD data given the predictor trained on the auxiliary OOD data. Then, due to the upper bound, the impacts of the OOD distribution discrepancy are reflected by the Wasserstein-1 distance between the auxiliary and the real OOD data, i.e., ${\rm W}_c(D_{X_{\rm O}},D_{X_{\rm A}})$. Therefore, although classical outlier exposure can improve OOD detection to some extent, it fails to ensure reliable detection of unseen OOD data, in that a larger distribution discrepancy generally indicates a worse guarantee for open-world OOD detection. %Eq.~\eqref{eq: motivate} can easily be derived from a special case of Theorem~\ref{T3}, given $\rho=0$. 

% The key challenge of OOD detection lies in knowing nothing about OOD data during the training. Thus, as aforementioned, a promising strategy is to employ a set of auxiliary OOD data as a surrogate~\citep{HendrycksMD19}. Accordingly, it is expected that trained models, i.e., following Eq.~\eqref{eq::oe}, over the auxiliary OOD and the ID data can perform well on the unseen real OOD data.

% The key challenge of OOD detection lies in the fact that OOD data are unavailable. Thus, as aforementioned, a promising strategy is to employ a set of auxiliary OOD data as a surrogate~\citep{HendrycksMD19,yang2021generalized}. Here, an implicit assumption is that training models (i.e., Eq. \eqref{eq::oe}) over the auxiliary OOD and the ID data can perform well on the unseen real OOD data.

The key to improve the detection performance is mitigating the negative impact induced by the OOD distribution discrepancy. To tackle this problem, a simple lemma inspires us:
\begin{lemma}\label{T0}
Let $d(\cdot,\cdot)$ be the distance to measure the discrepancy between distributions. Given a space $\mathfrak{D}$ consisting of some OOD distributions, if $D_{X_{\rm A}}\in \mathfrak{D}$, then 
\begin{equation}\label{Eq::T0.0}
\inf_{D_{X'} \in \mathfrak{D} }d(D_{X'},D_{X_{\rm O}}) \leq d(D_{X_{\rm A}},D_{X_{\rm O}}).
\end{equation}
If $d(\cdot,\cdot)$ is the Optimal Transport Cost in Definition~\ref{D1}, the cost function $c$ is a continuous metric, and $\mathfrak{D}$ is the Wasserstein-1 ball with a radius $\rho>0$, i.e., $\mathfrak{D} = \{D_{X'}: {\rm W}_c(D_{X'}, D_{X_{\rm A}})\leq \rho \}$, then 
\begin{equation}\label{Eq::T0.1}
\inf_{D_{X'} \in \mathfrak{D} }{\rm W}_{\rm c}(D_{X'},D_{X_{\rm O}})\leq \max\{{\rm W}_{c}(D_{X_{\rm A}},D_{X_{\rm O}})-\rho, 0\}.
\end{equation}
\end{lemma}
In the light of Lemma~\ref{T0}, we introduce a specific set of distributions $\mathfrak{D}$, augmented around the auxiliary OOD distribution. It makes it possible to mitigate the distribution discrepancy, following Eqs.~\eqref{Eq::T0.0} and~\eqref{Eq::T0.1}. Therefore, instead of choosing a model $\mathbf{f}_{\mathbf{w}}$ that directly minimizes the empirical risk in Eq.~\eqref{eq::oe}, we target augmenting the auxiliary OOD data within the distribution space $\mathfrak{D}$, namely,
\begin{equation}\label{Eq::DROrisk}
\min_{\mathbf{w}\in \mathcal{W}} \big [\widehat{R}_{\rm I}(\mathbf{w}) + \alpha \sup_{D_{X'} \in \mathfrak{D}} \mathbb{E}_{\mathbf{x}\sim D_{X'} } \ell_{\rm OE}(\mathbf{f}_{\mathbf{w}};\mathbf{x})\big ],~\text{subject~to~}\widehat{D}_{X_{\rm A}}\in \mathfrak{D},
\end{equation}
where $\widehat{D}_{X_{\rm A}}$ is the empirical form of ${D}_{X_{\rm A}}$, i.e.,
$\widehat{D}_{X_{\rm A}} = \frac{1}{m} \sum_{i=1}^m \delta_{\mathbf{x}_{\rm A}^i}$
{and}~$\delta_{\mathbf{x}_{\rm A}^i}$~\text{is~the~dirac~measure}.

%In this paper, we use the Wasserstein-1 distance to measure the discrepancy between distributions. Wassserstein-1 distance has its superiority in measuring the distance between probability distributions. For example, compared with KL-divergence, Wasserstein-1 distance covers a wider range of diverse distributions, and compared with the maximum mean discrepancy \cite{DBLP:journals/jmlr/GrettonBRSS12} that plainly measures mean features between distributions, Wasserstein-1 distance leads to more solid theoretical conclusions~\citep{wainwright2019high}. 

\section{Learning Framework}

This section proposes a general learning framework to mitigate the OOD distribution discrepancy. As aforementioned, we consider an augmented set of OOD distributions to improve OOD detection, thus named \emph{Distributional-Augmented OOD Learning} (DAL).

To begin with, we need to select a suitable distribution space $\mathfrak{D}$ for the tractable solutions of Eq.~\eqref{Eq::DROrisk}. Generally, the choice of $\mathfrak{D}$ influences both the richness of the auxiliary data as well as the tractability of the resulting optimization problem. Previous works have developed a series of distribution spaces, e.g., the distribution ball based on $f$-divergences \citep{NamkoongD16,DBLP:conf/iclr/MichelHN21} and maximum mean discrepancy (MMD) \citep{Staib2019DistributionallyRO}. However, there are several drawbacks for the distribution balls based on $f$-divergences and MMD: 1) any $ f$-divergence-based space $\mathfrak{D}$ contains only distributions within the same support set as $\widehat{D}_{X_{\rm A}}$; and 2) the effective solutions in the MMD-based space have not been provided~\citep{Staib2019DistributionallyRO}. 

Instead, motivated by~\cite{SinhaND18,mehra2022certifying,dai2023moderately}
and Theorem \ref{T0}, we consider the Wasserstein ball. For any $\rho>0$, we define the augmented OOD distribution set as
\begin{equation*}
\mathfrak{D} = \{D_{X'}: {\rm W}_c(D_{X'}, \widehat{D}_{X_{\rm A}})\leq \rho \},
\end{equation*}
and consider the following optimization problem: 
\begin{equation}\label{Eq::Objective}
\begin{split}
&\min_{\mathbf{w}\in \mathcal{W}} \widehat{R}_D(\mathbf{w};\rho) = \min_{\mathbf{w}\in \mathcal{W}} \big [ \widehat{R}_{\rm I}(\mathbf{w}) + \alpha \widehat{R}_{{\rm O}}(\mathbf{w};\rho)\big ],
\end{split}
\end{equation}
where %$\alpha > 0$ is the trade-off hyper-parameter, and
\begin{equation}
\widehat{R}_{{\rm O}}(\mathbf{w};\rho) = \sup_{{\rm W}_{c}(D_{X'},\widehat{D}_{X_{\rm A}}) \leq \rho} \mathbb{E}_{\mathbf{x}\sim D_{X'}} \ell_{\rm OE}(\mathbf{f}_{\mathbf{w}}; \mathbf{x}). \label{eq: dro_oe2}
\end{equation}
However, the optimization problem in Eq.~\eqref{Eq::Objective} is intractable due to the infinite-dimensional search for the distribution $D_{X'}$. Fortunately, the following dual theorem provides a solution:

\begin{theorem}[\citet{Blanchet2016QuantifyingDM}]\label{T-dual}
Let $c(\cdot,\cdot)$ be a continuous metric and 
$
\phi_{\gamma}(\mathbf{w};\mathbf{x})= \sup_{\mathbf{x}'\in \mathcal{X}} \{ \ell(\mathbf{f}_{\mathbf{w}};\mathbf{x}') - \gamma c(\mathbf{x}',\mathbf{x})\}
$
be the robust surrogate function. Then, for any $\rho>0$, 
\begin{equation}\label{Eq::dual}
 \widehat{R}_D(\mathbf{w};\rho)=\widehat{R}_{\rm I}(\mathbf{w}) + \alpha \inf_{\gamma \geq 0}\big \{\gamma \rho + \frac{1}{m}\sum_{i=1}^m \phi_{\gamma}(\mathbf{w}; \mathbf{x}_{\rm A}^i) \big \}.
\end{equation}
\end{theorem}
Theorem~\ref{T-dual} provides a feasible surrogate for the original optimization problem in  Eq.~\eqref{Eq::Objective}, transforming the infinite-dimensional problem to its finite counterpart, i.e., the data feature search. We use Eq.~\eqref{Eq::dual} to design our algorithm, cf., Section~\ref{sec: alg}.

\vspace{-3pt}
\subsection{Theoretical Supports}\label{sec: theory}
\vspace{-3pt}
This section provides the theoretical support for our DAL. Specifically, 1) Theorem~\ref{T1} shows that the empirical model given by Eq.~\eqref{Eq::Objective} can achieve consistent learning performance, and 2) Theorem \ref{T3} further demonstrates the expected detection risk estimation, i.e., $R_D(\mathbf{w})$, with respect to the empirical model given by Eq.~\eqref{Eq::Objective}. All the proofs can be found in Appendix~\ref{app: proofs}. To state our theoretical results, we use the notation ${R}_D(\mathbf{w};\rho)$ to represent the ideal form of $\widehat{R}_D(\mathbf{w};\rho)$, which is defined by
\begin{equation*}
{R}_D(\mathbf{w};\rho)= {R}_{\rm I}(\mathbf{w}) + \alpha {R}_{{\rm O}}(\mathbf{w};\rho),
\end{equation*}
where
\begin{equation*}
{R}_{{\rm O}}(\mathbf{w};\rho) = \sup_{{\rm W}_{c}(D_{X'},D_{X_{\rm A}}) \leq \rho} \mathbb{E}_{\mathbf{x}\sim D_{X'}} \ell_{\rm OE}(\mathbf{f}_\mathbf{w}; \mathbf{x}).
\end{equation*}
Similar to~\citet{SinhaND18}, our results rely on the covering number (cf., Appendix~\ref{app: cn}) for the model classes $\mathcal{F}=\{\ell(\mathbf{f}_{\mathbf{w}};\cdot): \mathbf{w} \in \mathcal{W}\}$ and $\mathcal{F}_{\rm OE} = \{\ell_{\rm OE}(\mathbf{f}_{\mathbf{w}};\cdot): \mathbf{w} \in \mathcal{W}\}$ to represent their complexity. Intuitively, the covering numbers $\mathcal{N}(\mathcal{F},\epsilon,L^{\infty})$ and $\mathcal{N}(\mathcal{F}_{\rm OE},\epsilon,L^{\infty})$ are the minimal numbers of $L^{\infty}$ balls of radius $\epsilon > 0$ needed to cover the model classes $\mathcal{F}$ and $\mathcal{F}_{\rm OE}$, respectively. Now, we demonstrate that DAL can achieve consistent performance under mild assumptions. 
%formally justify that our proposed learning framework can lead to generalizable OOD detection. For both the ID and the auxiliary OOD cases, we only have the finite sample sizes for model training. Hence, we begin by verifying that minimizing the worst-case empirical risk leads to minimizing the associated expected one. 
\begin{theorem}[Excess Generalization Bound]\label{T1} Assume that $0\leq \ell(\mathbf{f}_{\mathbf{w}};\mathbf{x},y)\leq M_{\ell}$, $0\leq \ell_{{{\rm OE}}}(\mathbf{f}_{\mathbf{w}};\mathbf{x})\leq M_{\ell_{\rm OE}}$, and $c(\cdot,\cdot):\mathcal{X}\times \mathcal{X} \rightarrow \mathbb{R}_{+}$ is a continuous metric.
Let $\widehat{\mathbf{w}}$ be the optimal solution of Eq. \eqref{Eq::Objective}, i.e.,
$\widehat{\mathbf{w}} \in \argmin_{\mathbf{w}\in \mathcal{W}}  \widehat{R}_D({\mathbf{w}}; \rho)$. Then with the probability at least $1-4e^{-t}>0$,
\begin{equation}\label{T1.01}
\begin{split}
&R_D(\widehat{\mathbf{w}};\rho) - \min_{\mathbf{w}\in \mathcal{W}} R_D(\mathbf{w};\rho)\leq \epsilon(n,m;t),
\end{split}
\end{equation}
for any $\rho>0$, where
\begin{equation*}
\begin{split}
{\epsilon(n,m;t)}=& \frac{b_0 M_{\ell}}{\sqrt{n}}  \int_{0}^{1} \sqrt{\log\mathcal{N}(\mathcal{F},M_{\ell}\epsilon,L^{\infty})}d\epsilon + 2M_{\ell}\sqrt{\frac{2t}{n}}
\\
+ & \alpha b_1 \sqrt{\frac{M_{\ell_{\rm OE}}^3}{\rho^2 m}}\int_{0}^{1} \sqrt{ \log \mathcal{N}(\mathcal{F}_{\rm OE},M_{\ell_{\rm OE}}\epsilon,L^{\infty})}{\rm d} \epsilon
+ \alpha b_{2}M_{\ell_{\rm OE}} \sqrt{\frac{2t}{m}},
\end{split}
\end{equation*}
where $b_0$, $b_1$ and $b_2$ are uniform constants.
\end{theorem}

Furthermore, under proper conditions, one can show that the bound in Eq. \eqref{T1.01} can attain ${\mathcal{O}}(1/ \sqrt{n})+{\mathcal{O}}(1/ \sqrt{m})$, i.e.,
$R_D(\widehat{\mathbf{w}};\rho) - \min_{\mathbf{w}\in \mathcal{W}} R_D(\mathbf{w};\rho) \leq {\mathcal{O}}(1/ \sqrt{n})+{\mathcal{O}}(1/ \sqrt{m})$.
Corollary \ref{T2} in Appendix \ref{C1inApp} gives an {example} to support the above claim. Next, we give a learning bound to estimate the expected detection risk in Eq.~\eqref{alpharisk} w.r.t. the model $\mathbf{f}_{\widehat{\mathbf{w}}}$ given by Eq. \eqref{Eq::Objective}. 

\begin{theorem}[Risk Estimation]\label{T3}
Given the same conditions in Theorem \ref{T1} and let $\widehat{\mathbf{w}}$ be the solution of Eq. \eqref{Eq::Objective}, which is given by
$\widehat{\mathbf{w}} \in \argmin_{\mathbf{w}\in \mathcal{W}}  \widehat{R}_D({\mathbf{w}}; \rho).$ If $\ell_{\rm OE}(\mathbf{f}_{\mathbf{w}}; \mathbf{x})$ is  $L_{c}$-Lipschitz w.r.t. $c(\cdot,\cdot)$, i.e., $|\ell_{\rm OE}(\mathbf{f}_{\mathbf{w}}; \mathbf{x})-\ell_{\rm OE}(\mathbf{f}_{\mathbf{w}}; \mathbf{x}')|\leq L_{c}c(\mathbf{x},\mathbf{x}')$, then with the probability at least $1-4e^{-t}>0$,
\begin{equation*}
R_D(\widehat{\mathbf{w}}) - \overbrace{{\min_{\mathbf{w} \in \mathcal{W}} R_D({\mathbf{w}};\rho)}}^{\text{approximate risk}} \leq  \underbrace{\alpha L_{c} \max\{{\rm W}_c(D_{X_{\rm O}},D_{X_{\rm A}})-\rho,0\}+ \epsilon(n,m;t)}_{\text{estimation error}},
\end{equation*}
for any $\rho>0$, where $\epsilon(n,m;t)$ is defined in Theorem~\ref{T1}.
\end{theorem}

The bias term $\alpha L_{c}  \max\{{\rm W}_c(D_{X_{\rm O}},D_{X_{\rm A}})-\rho,0\}=0$ when $\rho$ is large enough. Hence, a large $\rho$ implies a small estimation error. Although a larger $\rho$ leads to better generalization ability, the approximate risk $\min_{\mathbf{w} \in \mathcal{W}} R_D(\mathbf{w};\rho)$ may become larger. It implies that for practical effectiveness, i.e., small $R_D(\widehat{\mathbf{w}})$, there is a {trade-off} between the approximate risk $\min_{\mathbf{w} \in \mathcal{W}} R_D(\mathbf{w};\rho)$ and the bias $\alpha L_{c}  \max\{{\rm W}_c(D_{X_{\rm O}},D_{X_{\rm A}})-\rho,0\}$ across different choices of $\rho$. Hence, we need to choose a proper $\rho$ for open-world detection with unseen data {(cf., Section~\ref{sec: ablation})}.

% Furthermore, when we enlarge the radius $\rho$ for the Wasserstein ball, the bias term shrinks and the real-world detection power improves. If the value of $\rho$ is sufficiently large, we can cover the real OOD distribution in our learning framework, and the generalizable OOD detection is certifiable. 

\subsection{Proposed Algorithm} \label{sec: alg}

\begin{algorithm}[!t]
   \caption{Distributional-Augmented OOD Learning (DAL)}
   \label{alg: DOE}
%   \vspace{1mm}
\begin{algorithmic}
  \STATE {\bfseries Input:} ID and OOD samples from $D_{X_{\rm I}Y_{\rm I}}$ and $D_{X_{\rm A}}$;
  \FOR{$\texttt{st}=1$ \textbf{to} \texttt{num\_step}}
  \STATE Sample $S_{\rm B}$ and $T_{\rm B}$ from $D_{X_{\rm I}Y_{\rm I}}$ and $D_{X_{\rm A}}$;
  \STATE Initialize $\mathbf{p}^i\sim \mathcal{N}(\mathbf{0},\sigma I ),~\forall~i\in\{1,\ldots,|T_{\rm B}|\}$;
  \FOR{$\texttt{se}=1$ \textbf{to} \texttt{num\_search}}
        \STATE $\boldsymbol{\psi}^i=\nabla_{\mathbf{p}^i}\left[\ell_\text{OE}\left(\textbf{h}(\textbf{e}(\mathbf{x}_{\rm A}^i+\mathbf{p}^i);\textbf{e}(\mathbf{x}_{\rm A}^i)\right)-\gamma \left\| \mathbf{p}^i \right\|_1\right],~\forall~i\in\{1,\ldots,|T_{\rm B}|\}$;
        \STATE $\mathbf{p}^i\leftarrow \mathbf{p}^i + \texttt{ps} \boldsymbol{\psi}^i,~\forall~i\in\{1,\ldots,|T_{\rm B}|\}$
  \ENDFOR 
 \STATE $\gamma\leftarrow\min\big(\max\big(\gamma-\beta(\rho-\frac{1}{|T_{\rm B}|}\sum_{i=1}^{|T_{\rm B}|} \| \mathbf{p}^i \|, \gamma_\text{max}\big),0\big)$; 
 \STATE $\mathbf{w}\leftarrow \mathbf{w} - \texttt{lr} \nabla_\mathbf{w}\big[\frac{1}{|T_{\rm B}|}\sum_{i=1}^{|T_{\rm B}|} \ell_\text{OE}(\textbf{h}(\textbf{g}(\mathbf{x}_{\rm A}^i)+\mathbf{p}^i)) + \alpha \frac{1}{|S_{\rm B}|}\sum_{i=1}^{|S_{\rm B}|} \ell(\textbf{f}_\mathbf{w};\mathbf{x}_{\rm I}^i,y_{\rm I}^i)\big)]$;
 \ENDFOR
 \STATE {\bfseries Output:} model parameter $\mathbf{w}$.
\end{algorithmic}
\end{algorithm}

In this section, we introduce the algorithm design for DAL, summarized in Algorithm~\ref{alg: DOE}. Due to the space limit, we provide further discussions in Appendix~\ref{app:alg}. 

\textbf{Losses and Cost Function.} Following~\citet{HendrycksMD19}, we adopt the cross entropy loss to realize $\ell$ and the KL-divergence between model predictions and uniform distribution for $\ell_{\rm OE}$. We also define the cost function $c$ by the $l_1$ norm, namely, $c(\mathbf{x},\mathbf{x}') = \|\mathbf{x}-\mathbf{x}'\|_{1}$.

\textbf{Algorithm Design.} By Theorem~\ref{T-dual}, we can address the primary problem in Eq. \eqref{Eq::Objective} by the dual problem in Eq.~\eqref{eq: dro_oe2}. Additionally, following~\citet{du2022vos}, we perturb for the worst OOD data in the embedding space. Denote the model $\mathbf{f}_\mathbf{w}=\mathbf{h}\circ \mathbf{e}$ with $\mathbf{h}$ the classifier and $\mathbf{e}$ the feature extractor, we find the perturbation $\mathbf{p}$ for the embedding features, i.e., $\mathbf{e}(\mathbf{x})$, of the associated data $\mathbf{x}$. The perturbation $\mathbf{p}$ should lead to the worst OOD case for the surrogate function in Theorem \ref{T-dual}, namely,
\begin{equation*} %\label{eq: surrogate_emb}
\begin{split}
    & \phi_{\gamma}(\mathbf{w}; \mathbf{e}{(\mathbf{x})})= \sup_{\mathbf{p}\in\mathcal{E}} \left\{ 
 \ell_{\rm OE}(\mathbf{h}(\mathbf{e}(\mathbf{x})+\mathbf{p}); \mathbf{e}(\mathbf{x}))-\gamma  \|\mathbf{p}\|_1\right\},
\end{split}
\end{equation*}
where $\mathcal{E}$ denotes the space of embedding features. Note that we abuse the definition of  $\ell_{\rm OE}$, emphasizing that we perturb the embedding features of $\mathbf{e}(\mathbf{x})$ by $\mathbf{p}$. 

\textbf{Training and Inference.} Our definition of $\phi_{\gamma}(\mathbf{w}; \mathbf{e}{(\mathbf{x})})$ leads to a particular realization of Eq.~\eqref{Eq::dual}, which is the learning objective of our DAL. It can be solved by stochastic gradient optimization for deep models, e.g., mini-batch stochastic gradient descent. After training, we use the MSP scoring function by default and discuss the possibility of other scoring functions in Appendix~\ref{app: other scoring}. 

% we adopt  (ReAct)~\citep{sun2021react} as the scoring function, which is an advanced improving method over MSP. Overall, ASH suggests the post-hoc modification for the penultimate layer of the model, clamping the activation patterns that can mislead OOD detection, namely,
% \begin{equation}
    % s_\text{ReAct}(\mathbf{x};\mathbf{f}) = \max_{k\in \mathcal{Y}}~\texttt{softmax}_k~\mathbf{h}(\max(\mathbf{e}(\mathbf{x}),\tau)),
 %\end{equation}
% where $\tau$ is the rectifying threshold applied element-wise to the feature vector $\mathbf{e}(\mathbf{x})$. In particular, following~\citep{sun2021react}, $\tau$ is chosen based on the $\zeta$-th percentile of activation outputs estimated on ID training data. In \textcolor{blue}{Appendix~\ref{}}, we also provide the results for other choices of scoring functions. 

\textbf{Stochastic Realization.}
Algorithm~\ref{alg: DOE} gives a stochastic realization of DAL, where ID and auxiliary OOD mini-batches are randomly sampled in each stochastic iteration, denoted by $S_{\rm B}$ and $T_{\rm B}$ respectively. Therein, we first find the perturbation $\mathbf{p}$ that leads to the maximal $\phi_\gamma(\mathbf{w},\mathbf{e}(\mathbf{x}))$. The value of $\mathbf{p}$ is initialized by random Gaussian noise with the standard deviation $\sigma$ and updated by gradient ascent for $\texttt{num}\_\texttt{search}$ steps with the perturbation strength $\texttt{ps}$. Then we update $\gamma$ by one step of gradient descent with the learning rate $\beta$, and further clipping between $0$ and $\gamma_\text{max}$ to avoid extreme values. Finally, given the proper perturbations for the auxiliary OOD data in $T_{\rm B}$, we update the model parameter $\mathbf{w}$ by one step of mini-batch gradient descent. 

\vspace{-6pt}
\section{Experiments} \label{sec: experiment}

In this section, we mainly test DAL on the CIFAR~\citep{krizhevsky2009learning} benchmarks (as ID datasets). To begin with, we introduce the evaluation setups. 

\vspace{-1pt}
\textbf{OOD Datasets.} We adopt the 80 Million Tiny Images~\citep{torralba200880} as the auxiliary OOD dataset; Textures~\citep{cimpoi2014describing}, SVHN~\citep{netzer2011reading}, Places$365$~\citep{ZhouLKO018}, LSUN~\citep{yu2015lsun}, and iSUN~\citep{xu2015turkergaze} as the (test-time) real OOD datasets. We eliminate those data whose labels coincide with ID cases.

\vspace{-1pt}
\textbf{Pre-training Setups.} We employ Wide ResNet-40-2~\citep{zagoruyko2016wide} trained for $200$ epochs via empirical risk minimization, with a batch size $64$, momentum $0.9$, and initial learning rate $0.1$. The learning rate is divided by $10$ after $100$ and $150$ epochs. 

\vspace{-1pt}
\textbf{Hyper-parameters Tuning Strategy.} The hyper-parameters are tuned based on the validation data, separated from the training ID  and auxiliary OOD data, which is a common strategy in OOD detection with outlier exposure field~\citep{HendrycksMD19,chen2021atom}. Specifically, we fix $\sigma=0.001$, $\texttt{num}\_\texttt{search}=10$, and adopt the grid search to choose $\gamma_\text{max}$ from $\{0.1,0.5,1,5,10,50\}$; $\beta$ from $\{1e^{-3},5e^{-3},1e^{-2},5e^{-2},1e^{-1},5e^{-1},1,5\}$; $\rho$ from $\{1e^{-2},1e^{-1},1,10,100\}$; $\texttt{ps}$ from $\{1e^{-3},1e^{-2},1e^{-1},1,10,100\}$; $\alpha$ from $\{0.1,0.5,1.0,1.5,2.0\}$. 

 \vspace{-1pt}
\textbf{Hyper-parameters Setups.} 
For CIFAR-10, DAL is run for $50$ epochs with the ID batch size $128$, the OOD batch size $256$, the initial learning rate $0.07$, $\gamma_\text{max}=10$, $\beta=0.01$, $\rho=10$, $\texttt{ps}=1$, and $\alpha=1$. 
For CIFAR-100, DAL is run for $50$ epochs with the ID batch size $128$, the OOD batch size $256$, the initial learning rate $0.07$, $\gamma_\text{max}=10$, $\beta=0.005$, $\rho=10$, and $\texttt{ps}=1$, and $\alpha=1$.  
For both cases, we employ cosine decay~\citep{LoshchilovH17} for the model learning rate. 

\vspace{-1pt}
\textbf{Baseline Methods.} We compare DAL with representative methods, including MSP~\citep{hendrycks2016baseline}, Free Energy~\citep{liu2020energy}, ASH~\citep{djurisic2022extremely}, ReAct~\citep{sun2021react}, Mahalanobis~\citep{lee2018simple}, KNN~\citep{SunM0L22}, KNN+~\citep{SunM0L22}, CSI~\citep{Tack20CSI}, VOS~\citep{du2022vos}, Outlier Exposure (OE)~\citep{HendrycksMD19}, Energy-OE~\citep{liu2020energy}, ATOM~\citep{chen2021atom}, DOE~\citep{wang2023out}, and POEM~\citep{MingFL22}. We adopt their suggested setups but unify the backbones for fairness. 

% MaxLogit~\citep{hendrycks2022scaling}\SL{This is a theoretically misleading method that I'd suggest removing. MaxLogit doesn't have an interpretation and connection to $\log p(x)$. A similar argument is made here \url{https://arxiv.org/abs/2112.00787} on why we should use \text{log sum exp} rather than max.}
\vspace{-1pt}
\textbf{Evaluation Metrics.} The detection performance is evaluated via two representative metrics, which are both threshold-independent: the false positive rate of OOD data when the true positive rate of ID data is at $95\%$ (FPR$95$); and the {area under the receiver operating characteristic curve} (AUROC), which can be viewed as the probability of the ID case having greater score than that of the OOD case. 

\vspace{-1pt}
Due to the space limit, we test our DAL with more advanced scoring strategies in Appendix~\ref{app: other scoring} and conduct experiments on the more complex ImageNet~\citep{deng2009imagenet} dataset in Appendix~\ref{app: imagenet results}.

\begin{table*}[t]
\caption{Comparison between our method and advanced methods on the CIFAR benchmarks.  $\downarrow$ (or $\uparrow$) indicates smaller (or larger) values are preferred, and a bold font indicates the best result in a column. Methods are grouped based on 1) using ID data only and 2) using additional information about auxiliary OOD data. Two groups are separated by the horizontal line for each ID case.} \label{tab: full} \vspace{4pt}
\resizebox{\linewidth}{!}{
\begin{tabular}{c|cccccccccccc}
\toprule[1.5pt]
\multirow{2}{*}{Method} & \multicolumn{2}{c}{SVHN} & \multicolumn{2}{c}{LSUN} & \multicolumn{2}{c}{iSUN} & \multicolumn{2}{c}{Textures} & \multicolumn{2}{c}{Places365} & \multicolumn{2}{c}{\textbf{Average}} \\
\cline{2-13}
& FPR95 $\downarrow$ & AUROC $\uparrow$ & FPR95 $\downarrow$ & AUROC $\uparrow$ & FPR95 $\downarrow$ & AUROC $\uparrow$ & FPR95 $\downarrow$ & AUROC $\uparrow$ & FPR95 $\downarrow$ & AUROC $\uparrow$ & FPR95 $\downarrow$ & AUROC $\uparrow$ \\
\midrule[1pt]
\multicolumn{13}{c}{\cellcolor{greyC} CIFAR-10} \\
\midrule[0.6pt]
\multicolumn{13}{c}{Using ID data only} \\ \hline
MSP         & 48.89 & 91.97 & 25.53 & 96.49 & 56.44 & 89.86 & 59.68 & 88.42 & 60.19 & 88.36 & 50.15 & 91.02 \\
Free Energy & 35.21 & 91.24 &  4.42 & 99.06 & 33.84 & 92.56 & 52.46 & 85.35 & 40.11 & 90.02 & 33.21 & 91.64 \\
ASH         & 33.98 & 91.79 &  4.76 & 98.98 & 34.38 & 92.64 & 50.90 & 86.07 & 40.89 & 89.79 & 32.98 & 91.85 \\
Mahalanobis & 12.21 & 97.70 & 57.25 & 89.58 & 79.74 & 77.87 & 15.20 & 95.40 & 68.81 & 82.39 & 46.64 & 88.59 \\
KNN         & 26.56 & 95.93 & 27.52 & 95.43 & 33.55 & 93.15 & 37.62 & 93.07 & 41.67 & 91.21 & 33.38 & 93.76 \\
KNN+        &  3.28 & 99.33 &  2.24 & 98.90 & 17.85 & 95.65 & 10.87 & 97.72 & 30.63 & 94.98 & 12.97 & 97.32 \\
CSI         & 17.37 & 97.69 &  6.75 & 98.46 & 12.58 & 97.95 & 25.65 & 94.70 & 40.00 & 92.05 & 20.47 & 96.17 \\
VOS         & 36.55 & 93.30 &  9.98 & 98.03 & 28.93 & 94.25 & 52.83 & 85.74 & 39.56 & 89.71 & 33.57 & 92.21  \\
\midrule[0.6pt]
\multicolumn{13}{c}{Using ID data and auxiliary OOD data} \\ \hline
OE          &  2.36 & 99.27 &  1.15 & \textbf{99.68} &  2.48 & 99.34 &  5.35 & 98.88 & 11.99 & 97.23 &  4.67 & 98.88 \\
Energy-OE   &  0.97 & 99.54 &  1.00 & 99.15 &  2.32 & 99.27 &  3.42 & \textbf{99.18} &  9.57 & 97.44 &  3.46 & 98.91 \\
ATOM        &  1.00 & 99.59 &  0.61 & 99.53 &  2.15 & \textbf{99.40} &  2.52 & 99.10 &  7.93 & 97.27 &  2.84 & 98.97 \\
DOE         &  1.80 & 99.37 &  \textbf{0.25} & 99.65 &  2.00 & 99.36 &  5.65 & 98.75 & 10.15 & 97.28 &  3.97 & 98.88 \\
POEM        &  1.20 & 99.53 &  0.80 & 99.10 &  \textbf{1.47} & 99.26 &  2.93 & 99.13 &  7.65 & 97.35 &  2.81 & 98.87 \\
% GOOD        & \\
\rowcolor[HTML]{EFEFEF} DAL & \textbf{0.80} & \textbf{99.65} & 0.90 & 99.46 & 1.70 & 99.34 & \textbf{2.30} & 99.14 & \textbf{7.65} & \textbf{97.45} & \textbf{2.68} & \textbf{99.01}  \\
\midrule[1.5pt]
\multicolumn{13}{c}{{\cellcolor{greyC}  CIFAR-100}} \\
\midrule[0.8pt]
\multicolumn{13}{c}{Using ID data only} \\ \hline
MSP         & 84.39 & 71.18 & 60.36 & 85.59 & 82.63 & 75.69 & 83.32 & 73.59 & 82.37 & 73.69 & 78.61 & 75.95 \\
Free Energy & 85.24 & 73.71 & 23.05 & 95.89 & 81.11 & 79.02 & 79.63 & 76.35 & 80.18 & 75.65 & 69.84 & 80.12 \\
ASH         & 70.09 & 83.56 & 13.20 & 97.71 & 69.87 & 82.56 & 63.69 & 83.59 & 79.70 & 74.87 & 59.31 & 84.46 \\
Mahalanobis & 51.00 & 88.70 & 91.60 & 69.69 & 38.48 & 91.86 & 47.07 & 89.09 & 82.70 & 74.18 & 72.37 & 82.70 \\
KNN         & 52.10 & 88.83 & 68.82 & 79.00 & 42.17 & 90.59 & 42.79 & 89.07 & 92.21 & 61.08 & 59.62 & 81.71 \\
KNN+        & 32.50 & 93.86 & 47.41 & 84.93 & 39.82 & 91.12 & 43.05 & 88.55 & 63.26 & 79.28 & 45.20 & 87.55 \\
CSI         & 64.50 & 84.62 & 25.88 & 95.93 & 70.62 & 80.83 & 61.50 & 86.74 & 83.08 & 77.11 & 61.12 & 95.05 \\
VOS         & 78.06 & 92.59 & 40.40 & 92.90 & 85.77 & 70.20 & 82.46 & 77.22 & 82.31 & 75.47 & 73.80 & 91.67 \\
\midrule[0.8pt]
\multicolumn{13}{c}{Using ID data and auxiliary OOD data} \\ \hline
OE          & 46.73 & 90.54 & 16.30 & 96.98 & 47.97 & 88.43 & 50.39 & 88.27 & 54.30 & 87.11 & 43.14 & 90.27  \\
Energy-OE   & 35.34 & 94.74 & 16.27 & 97.25 & 33.21 & 93.25 & 46.13 & 90.62 & 50.45 & 90.04 & 36.28 & 93.18  \\
ATOM        & 24.80 & 95.15 & 17.83 & 96.76 & 47.83 & 91.06 & 44.86 & \textbf{91.80} & 53.92 & 88.88 & 37.84 & 92.73  \\
DOE         & 43.10 & 91.83 & 13.95 & \textbf{97.56} & 47.25 & 87.88 & 49.40 & 88.62 & 51.05 & 88.08 & 40.95 & 90.79  \\
POEM        & 22.27 & \textbf{96.28} & \textbf{13.66} & 97.52 & 42.46 & 91.97 & 45.94 & 90.42 & 49.50 & 90.21 & 34.77 & 93.28 \\
% GOOD        & \\
\rowcolor[HTML]{EFEFEF}  DAL      & \textbf{19.35} & 96.21 & 16.05 & 96.78 & \textbf{26.05} & \textbf{94.23} & \textbf{37.60} & 91.57 & \textbf{49.35} & \textbf{90.81} & \textbf{29.68} & \textbf{93.92}  \\
\bottomrule[1.5pt]   
\end{tabular}}
\end{table*}

\vspace{-5pt}
\subsection{Main Results}

The main results are summarized in Table~\ref{tab: full}, where we report the detailed results across the considered real OOD datasets. First, we reveal that using auxiliary OOD data can generally lead to better results than using only ID information, indicating that outlier exposure remains a promising direction worth studying. However, as demonstrated in Section~\ref{sec: motivation},  the OOD distribution discrepancy can hurt its open-world detection power, while previous works typically oversee such an important issue. Therefore, our DAL, which can alleviate the OOD distribution discrepancy, reveals a large improvement over the original outlier exposure. Specifically, comparing with the conventional outlier exposure, our method reveals {1.99 and 0.13} average improvements w.r.t. FPR95 and AUROC on the CIFAR-10 dataset, and  {13.46 and 3.65} of the average improvements on CIFAR-100 dataset. For advanced works that consider the OOD sampling strategies, e.g., ATOM and POEM, DAL can achieve much better results, especially for the CIFAR-100 case. The reason is that these methods mainly consider the situations where the model capacity is not enough to learn from all the auxiliary OOD data, deviating from our considered issue in OOD distribution discrepancy. Moreover, for the previous works that adopt similar concepts in the worst-case OOD learning, e.g., VOS and DOE, DAL also reveals better results, with {1.29 and 30.89} improvements on the CIFAR-10 dataset and {11.27 and 44.12} improvements on the CIFAR-100 dataset w.r.t. FPR95. It indicates that our theoretical-driven scheme can also guide the algorithm designs with practical effectiveness. Note that many previous works use advanced scoring strategies other than MSP, and thus our experiment above is not completely fair to us. Therefore, in Appendix~\ref{app: other scoring}, we also combine DAL with many advanced scoring strategies other than MSP, which can further improve our performance. 

% \SL{For Table 2, Should we use consistent baselines as Table 1?}
%\vspace{-5pt}
\begin{wraptable}{r}{0.6\textwidth}
\vspace{-0.5cm}
\caption{Comparison between our method and advanced methods on hard OOD detection.  $\downarrow$ (or $\uparrow$) indicates smaller (or larger) values are preferred, and a bold font indicates the best result in a column.} \label{tab: cifar hard}
\vspace{7pt}
\centering\resizebox{\linewidth}{!}{
\begin{tabular}{c|cc|cc|cc}
\toprule[1.5pt]
\multirow{2}{*}{Methods} & \multicolumn{2}{c|}{LSUN-Fix}            & \multicolumn{2}{c|}{ImageNet-Resize} & \multicolumn{2}{c}{CIFAR-100}           \\
 \cline{2-7} 
 & FPR$95$ $\downarrow$ & AUROC $\uparrow$ & FPR$95$ $\downarrow$ & AUROC $\uparrow$ & FPR$95$ $\downarrow$ & AUROC $\uparrow$ \\
\midrule[1pt]
\multicolumn{7}{c}{Using ID data only} \\ \hline
Free Energy &  6.42 & 98.85 & 46.46 & 89.02 & 50.47 & 87.08 \\
ASH         &  4.00 & 98.20 & 46.18 & 88.85 & 54.31 & 83.71 \\
KNN+        & 24.88 & 95.75 & 30.52 & 94.85 & 40.00 & 89.11 \\
CSI         & 39.79 & 93.63 & 37.47 & 93.93 & 45.64 & 87.64 \\
\midrule[0.8pt]
\multicolumn{7}{c}{Using ID data and auxiliary OOD data} \\ \hline
OE          &  1.75 & 99.47 &  6.76 & 98.58 & 29.40 & 94.20 \\
DOE         &  1.97 & 98.71 &  5.98 & 98.75 & 29.75 & 94.24 \\
POEM        &  \textbf{1.24} & 98.93 &  6.56 & 98.37 & 35.11 & 91.80 \\
\rowcolor[HTML]{EFEFEF}  DAL         &  1.39 & \textbf{99.47} &  \textbf{5.60} & \textbf{98.80} & \textbf{25.45} & \textbf{94.34} \\
 \bottomrule[1.5pt]
\end{tabular}}
\end{wraptable}

\subsection{Hard OOD Detection}
We further consider hard OOD scenarios~\citep{SunM0L22}, of which the test OOD data are very similar to that of the ID cases. Following the common setup~\citep{SunM0L22} with the CIFAR-$10$ dataset being the ID case, we evaluate our DAL on three hard OOD datasets, namely, LSUN-Fix~\citep{yu2015lsun}, ImageNet-Resize~\citep{deng2009imagenet}, and CIFAR-$100$. Note that data in ImageNet-Resize ($1000$ classes) with the same semantic space as Tiny-ImageNet ($200$ classes) are removed. We select a set of strong baselines that are competent in hard OOD detection, summarizing the experiments in Table~\ref{tab: cifar hard}. As we can see, our method can beat these advanced methods across the considered datasets, even for the challenging CIFAR-$10$ versus CIFAR-$100$ setting. The reason is that our distributional augmentation directly learns from OOD data close to ID pattern, which can cover hard OOD cases.

\subsection{Ablation Study} \label{sec: ablation}

We further conduct an ablation study to demonstrate two mechanisms that mainly contribute to our open-world effectiveness, namely, OOD data generation and Wasserstein ball constraint.

\textbf{OOD Data Generation.} DAL learns from the worst OOD data to mitigate the OOD distribution discrepancy. To understand such an OOD generation scheme,  we employ the t-SNE visualization~\citep{van2008visualizing} for the ID, the auxiliary OOD, and the worst OOD data. Figure~\ref{fig:tsne} summarizes the results before and after DAL training. Before training, the ID and auxiliary OOD data overlap largely, indicating that the original model is not effective at distinguishing between them. Then, DAL does not directly train the model on auxiliary OOD data but instead perturbs it to further confuse the model beyond the overlap region. After DAL training, the overlap region between ID and auxiliary OOD data shrinks. Additionally, perturbing the original OOD data becomes more difficult, indicating that the model has learned to handle various worst-case OOD scenarios.

\begin{figure}
   \begin{minipage}{0.49\textwidth}
     \centering
     \includegraphics[width=\linewidth]{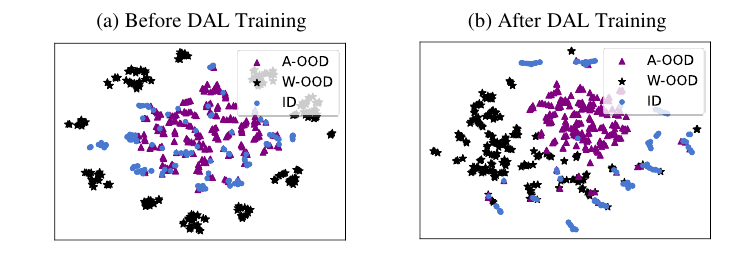}
     \caption{{Illustrations of embedding features for the ID, the auxiliary OOD (A-OOD), and the worst OOD (W-OOD) data. We adopt the t-SNE visualization on the CIFAR-10 dataset and illustrate the results before and after DAL training.}}\label{fig:tsne}
   \end{minipage}\hfill
   \begin{minipage}{0.49\textwidth}
     \centering
     \includegraphics[width=\linewidth]{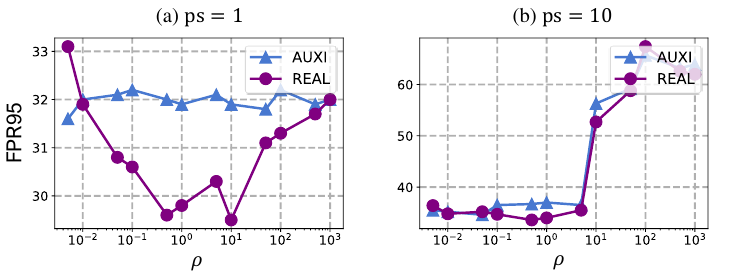}
     \caption{FPR95 curves across various $\rho$ on the CIFAR-100. We report both the results for the average real OOD data (REAL) and the auxiliary OOD data (AUXI), where we consider two hyper-parameter setups, i.e., $\texttt{ps}=1$ and $\texttt{ps}=10$.}\label{fig:my_label}
   \end{minipage}
\end{figure}

\textbf{Wasserstein Ball Constraint.} The choice of $\rho$ determines the radius of the Wasserstein ball. Larger values of $\rho$ reduce estimation error and improve model generalization, as stated in Theorem~\ref{T3}. However, larger values of $\rho$ also increase the approximate risk $\min_{\mathbf{w} \in \mathcal{W}} R_D(\mathbf{w};\rho)$ as it becomes more challenging to ensure uniform model performance with increased distributional perturbation. Figure~\ref{fig:my_label} shows the FPR95 curves on the CIFAR-100 dataset for both the real and the surrogate OOD data, revealing the trade-off in selecting $\rho$. Here, we consider two setups of $\rho$, i.e., $\texttt{ps}=1$ (default) and $\texttt{ps}=10$ (large perturbation strength). First, when the perturbation strength is very large (i.e., $\texttt{ps}=10$), the model can easily fail for training if the value of $\rho$ is also large (e.g., $\rho=100$), indicating that large value of $\rho$ can lead to a large approximation error. However, such an issue can be overcome by selecting a relatively small value of $\rho$ (e.g., $\rho=0.5$). % Furthermore, we can observe a trade-off between estimation error and the approximate risk for the both cases, reflecting by the U-shape curves for the real OOD data that corresponds to the expected risk $R_D(\widehat{\mathbf{w}})$.

\section{Conclusion}
Outlier exposure is one of the most powerful methods in OOD detection, but the discrepancy between the auxiliary and (unseen) real OOD data can hinder its practical effectiveness. To address such an issue, we have formalized it as the OOD distribution discrepancy and developed an effective learning framework to mitigate its negative impacts. Specifically, we consider a specific distribution set that contains all distributions in a Wasserstein ball centered on the auxiliary OOD distribution. Then, models trained over worst-case OOD data in the ball can ensure improved performance toward open-world OOD detection. Overall, as pioneers in critically analyzing the open-world setting with theoretical analysis, we are committed to raising attention to the OOD distribution discrepancy issue and encouraging further research in this direction.

\begin{ack}
QZW, YGZ, and BH were supported by the NSFC Young Scientists Fund No. 62006202, NSFC General Program No. 62376235, Guangdong Basic and Applied Basic Research Foundation No. 2022A1515011652, HKBU Faculty Niche Research Areas No. RC-FNRA-IG/22-23/SCI/04, and HKBU CSD Departmental Incentive Scheme. FL was supported by Australian Research Council (ARC) under Award No. DP230101540, and by NSF and CSIRO Responsible AI Program under Award No. 2303037. YXL was supported by the AFOSR Young Investigator Program under award number FA9550-23-1-0184, National Science Foundation (NSF) Award No. IIS-2237037 \& IIS-2331669, and Office of Naval Research under grant number N00014-23-1-2643. 
\end{ack}

{
\bibliography{reference}
\bibliographystyle{unsrtnat}
}

%%%%%%%%%%%%%%%%%%%%%%%%%%%%%%%%%%%%%%%%%%%%%%%%%%%%%%%%%%%%

\clearpage
\appendix
\onecolumn

\section{Related Works}
\label{app: related works}
In this section, we discuss the related studies in OOD detection. 

\textbf{OOD Detection Methods.} Existing works in OOD detection can be mainly categorized into three categories, namely, the post-hoc methods, the representation-based methods, and the outlier exposure. For the post-hoc methods, they believe a well-trained ID classifier can already lead to effective OOD detection~\citep{hendrycks2016baseline}, given proper scoring functions to indicate ID and OOD cases. Existing scoring functions are built upon the classifiers, taking logit outputs~\citep{hendrycks2016baseline,LiangLS18,liu2020energy,sun2021react,wang2021can,lakshminarayanan2017simple,wang2021can,HuangL21}, embedding features~\citep{lee2018simple,sastry2019detecting,wang2022vim,lin2021mood,SunM0L22,morteza2022provable,10107906}, or gradient information~\citep{huang2021importance,LiangLS18,igoe2022useful} as its inputs and returning a score value to indicate the confidence for an ID case. Recent works focus on adaptation strategies for specific tasks~\citep{huang2021importance,zhu2023unleashing} and non-parametric approaches~\citep{SunM0L22}, which may motivate future works.

Other works believe that training procedures are indispensable in OOD detection. For representation-based methods, researchers assume that a good ID representation is all we need for effective OOD detection. Therein, researchers study contrastive learning methods~\citep{SehwagCM21,wang2022partial}, data augmentation~\citep{Tack20CSI,zheng2023out},  constraints on embedding features~\citep{du2022siren, ming2022cider,zaeemzadeh2021ood} or model output~\citep{wei2022mitigating}. However, some of the adopted scoring functions in representation-based methods are complex. It can make us overestimate the true effects of representation learning, which may require further studies. For outlier exposure, related methods directly make the model learn from OOD data with low OOD score predictions~\citep{HendrycksMD19,liu2020energy,huang2023harnessing}. Related works studies sampling strategies~\citep{zhu2023diversified,MingFL22,chen2021atom}, adversarial robust learning~\citep{li2020background,lee2018simple,0001AB19}, meta learning~\citep{jeong2020ood}, and regularization strategies~\citep{van2020uncertainty}. Other works consider the situations where OOD data are inaccessible, studying various outlier synthesis strategies~\citep{LeeLLS18,vernekar2019out,du2022vos,tao2023nonparametric}. Although outlier exposure typically reveals promising results, the difference between auxiliary and real OOD data largely hinders its real-world detection power, similar to conclusions in domain adaptation~\citep{10107906,LuoWHB20}.

\textbf{OOD Detection Theory.} \citet{zhang2021understanding} gives an explanation of why there exist OOD data that have higher probabilities or densities
than the data from the ID distribution in the deep generative models. \citet{zhang2021understanding} understands OOD detection via goodness-of-fit tests and points out that  OOD detection should be defined based on the
data distribution’s typical set if we hope OOD detection can be successful. \citet{morteza2022provable} develops a novel unified framework that helps researchers to understand the theoretical connections among some representative OOD detection methods.
\citet{DBLP:conf/icml/FangLLL021, fang2022learnable} develop the probably approximately correct
(PAC) learning theory for OOD detection and gives a series of sufficient and necessary conditions for the PAC learnability of OOD detection. \citet{fang2022learnable} has proven that although OOD detection cannot be PAC learnable in the distribution-free case, OOD detection can be successful in many practical scenarios. Note that \citet{zhang2021understanding}, \citet{morteza2022provable}, and \citet{fang2022learnable} all consider the case that the auxiliary OOD data are unavailable. Therefore, to ensure the leanability of OOD detection, some strong conditions are necessary \citep{zhang2021understanding, fang2022learnable}. To explore the outlier exposure case in OOD detection, \citet{DBLP:conf/icml/BitterwolfMA022} shows that several representative OOD detection methods that
optimize an objective that includes predictions on auxiliary OOD data are equivalent to the binary discriminator. Compared to 
\citet{zhang2021understanding,morteza2022provable,fang2022learnable}, our paper mainly focuses on the case that the auxiliary OOD data are available. Using the auxiliary OOD data, we can weaken the strong conditions proposed by \citet{fang2022learnable} and provide more reasonable and practical learning bounds for OOD detection. Compared to \citet{DBLP:conf/icml/BitterwolfMA022}, our theory mainly focuses on the learnability of OOD detection in the outlier exposure case and provides theoretical support to our practical method.

\newpage
\section{Notations}\label{app: notations}

In this section, we summarize the important notations in Table~\ref{tab: notation}. 

\begin{table}[h]
    \centering
    \small
    \caption{Main notations and their descriptions.}
    \begin{tabular}{cl}
    \toprule[1.5pt]
         \multicolumn{1}{c}{Notation} & \multicolumn{1}{c}{Description} \\
    \midrule[1pt]
    \multicolumn{2}{c}{\cellcolor{greyC} Spaces} \\
    $\mathcal{X}$ and $\mathcal{Y}$     & the feature space and the ID label space $\{1,\dots,C\}$ \\
    $\mathcal{W}$ & the parameter space \\ 
    $\mathfrak{D}$ & the distribution space \\
    $\mathcal{E}$ & the embedding space \\
    \multicolumn{2}{c}{\cellcolor{greyC} Distributions} \\
    $X_{\rm I}, X_{\rm A}, X_{\rm O}$ & ID feature, auxiliary OOD feature, and real OOD feature \\
    $Y_{\rm I}$ and $Y_{\rm O}$ & ID label and OOD label random variable \\
    $D_{X_{\rm I}Y_{\rm I}}$ and $D_{X_{\rm O}Y_{\rm O}}$  & ID joint distribution and OOD joint distribution \\
    $D_{X_{\rm A}}$  & the auxiliary OOD distribution \\
    $\delta$ & the dirac measure \\
    \multicolumn{2}{c}{\cellcolor{greyC} Data and Models} \\
    $S$ and $T$ & ID training data and auxiliary OOD training data \\
    $n$ and $m$ & the number of ID data and the number of auxiliary OOD data \\
    $\mathbf{x}_{\rm I}$ and $\mathbf{x}_{\rm A}$ & ID data and auxiliary OOD data\\
    $y_{\rm I}$ & ID label\\
    $\mathbf{f}_{\mathbf{w}}$ & the model: $\mathcal{X}\rightarrow\mathbb{R}^C$, parameterized by $\mathbf{w}\in\mathcal{W}$ \\
    $\mathbf{e}$ and $\mathbf{h}$ & the feature extractor and the classifier \\
    $s(\cdot;\mathbf{f})$ & the scoring function: $\mathcal{X}\rightarrow\mathbb{R}$ \\
    $g_\lambda(\cdot)$ & the OOD detector: $\mathcal{X}\rightarrow\{\text{ID, OOD}\}$, with threshold $\lambda$ \\
    \multicolumn{2}{c}{\cellcolor{greyC} Distances} \\
    $c(\cdot,\cdot)$ & the cost function: $\mathcal{X}\times\mathcal{X}\rightarrow\mathbb{R}_+$ \\
    $d(\cdot,\cdot)$ & the distance between two distributions: $\mathfrak{D}\times\mathfrak{D}\rightarrow\mathbb{R}_+$ \\
    ${\rm W}_c$ & the Wasserstein-1 distance: $\mathfrak{D}\times\mathfrak{D}\rightarrow\mathbb{R}_+$ \\
    $\rho$ & the radius of the Wasserstein ball \\
    $\|\cdot\|_p$ & $l_p$ norm \\
    \multicolumn{2}{c}{\cellcolor{greyC} Loss and Risk} \\
    $\ell$ and $\ell_{\rm OE}$ & ID loss function and OOD loss function  \\
    $R_{\rm I}(\mathbf{w})$ and $\widehat{R}_{\rm I}(\mathbf{w})$ & the expected risk and the empirical risk corresponding to $D_{X_{\rm I}Y_{\rm I}}$\\
    $R_{\rm A}(\mathbf{w})$ and $\widehat{R}_{\rm A}(\mathbf{w})$ & the expected risk and the empirical risk corresponding to $D_{X_{\rm A}}$\\
    $R_{\rm O}(\mathbf{w})$ & the expected risk corresponding to $D_{X_{\rm O}}$\\
    $R_{D}(\mathbf{w})$ & the real detection risk  corresponding to $D$\\
    $\phi_\gamma(\mathbf{w};\mathbf{x})$ & the surrogate function \\
    $R_{\rm O}(\mathbf{w};\rho)$ and $\widehat{R}_{\rm O}(\mathbf{w};\rho)$ & the expected DAL risk and the empirical one corresponding to $D_{X_{\rm A}}$\\
    $R_{D}(\mathbf{w};\rho)$ and $\widehat{R}_{D}(\mathbf{w};\rho)$ & the expected DAL risk and the empirical one corresponding to $D$\\
    \multicolumn{2}{c}{\cellcolor{greyC} Hypothesis Space} \\
    $\mathcal{F}$ and $\mathcal{F}_\text{OE}$ & the model classes with respect to $\ell$ and $\ell_{\rm OE}$ \\
    $\mathcal{N}(\cdot,\epsilon,L^{\infty})$ & the covering number \\
    \bottomrule[1.5pt]
    \end{tabular}
    
    \label{tab: notation}
\end{table}

\clearpage

%understands the OOD detection via goodness-of-fit tests and typical set hypothesis, and
% discover that relevant OOD distributions can lie in the high likelihood regions of a data distribution. \citep{zhang2021understanding} 
%argues that minimal density estimation errors can lead to OOD detection failures without assuming an overlap between ID and OOD distributions. 
% understands the OOD detection via goodness-of-fit tests and typical set hypothesis, and discovers that relevant OOD distributions can lie in the high likelihood regions of a data distribution.

%Compared to \citep{zhang2021understanding}, our theory focuses on the PAC learnable theory of OOD detection. If detectors are generated by FCNN, our theory (Theorem~\ref{overlapcase}) shows that the overlap is the sufficient condition to the failure of learnability of OOD detection, which is complementary to \citep{zhang2021understanding}. In addition, we identify several necessary and sufficient conditions for the learnability of OOD detection, which opens a door to studying OOD detection in theory.
%Beyond \citep{zhang2021understanding}, \citep{morteza2022provable} paves a new avenue to designing provable OOD detection algorithms. Compared to \citep{morteza2022provable}, our paper aims to characterize the learnability of OOD detection to answer the question: is OOD detection PAC learnable?
\newpage

\section{Proofs of Theorems}
\label{app: proofs}

We provide the detailed proofs for our theoretical results in Sections~\ref{sec: theory}.  

\subsection{Covering Number} \label{app: cn}
We use the covering number for the model classes in our derivation. Here, we give the formal definition. 

\begin{definition}[$\epsilon$-covering \citep{Vershynin2018HighDimensionalP}]
    Let $(V,\|\cdot\|)$ be a normed space, $\Theta\subset V$, and $B(\cdot,\epsilon)$ the ball of radius $\epsilon$. Then $\{V_1,\dots,V_N\}$ is an $\epsilon$-covering of $\Theta$ if $\Theta\subset\bigcup_{i=1}^N B(V_i,\epsilon)$, or equivalently, $\forall\theta\in\Theta$, $\exists i$ such that $\|\theta-V_i\|\le\epsilon$.
\end{definition}

Upon our definition of $\epsilon$-covering, the covering number is the minimal number of $\epsilon$-balls one needs to cover $\Theta$. 
\begin{definition}[Covering Number \citep{Vershynin2018HighDimensionalP}]
\begin{equation*}
    \mathcal{N}(\Theta,\|\cdot\|,\epsilon)=\min\{n: \exists ~ \epsilon\text{-covering over}~\Theta~\text{of size}~n\}.
    \end{equation*}
\end{definition}
\subsection{Proof of Lemma~\ref{T0}}
\begin{proof}[Proof of Lemma~\ref{T0}]
Because of $D_{X_{\rm A}} \in \mathfrak{D}$, according to the definite of infimum, it is clear that
\begin{equation*}
\inf_{D_{X'} \in \mathfrak{D} }d(D_{X'},D_{X_{\rm O}}) \leq d(D_{X_{\rm A}},D_{X_{\rm O}}).
\end{equation*}
To prove the second result, we consider a special distribution $D'$, which is defined as follows: for some $u\in [0,1]$,
\begin{equation*}
D' = (1-u) D_{X_{\rm O}} + u D_{X_{\rm A}}.
\end{equation*}
Because $c(\cdot,\cdot)$ is a continuous metric, Kantorovich–Rubinstein duality \citep{Villani2003TopicsIO} implies that
\begin{equation*}
\begin{split}
{\rm W}_c(D', D_{X_{\rm O}}) = &\sup_{\|f\|_{\rm Lip}\leq 1} \int_{\mathcal{X}} f(\mathbf{x}) {\rm d}D'(\mathbf{x}) - \int_{\mathcal{X}} f(\mathbf{x}) {\rm d}D_{X_{\rm O}}(\mathbf{x}) 
\\ = & u \sup_{\|f\|_{\rm Lip}\leq 1} \int_{\mathcal{X}} f(\mathbf{x}) {\rm d}D_{X_{\rm A}}(\mathbf{x}) - \int_{\mathcal{X}} f(\mathbf{x}) {\rm d}D_{X_{\rm O}}(\mathbf{x}) 
\\
=&  u {\rm W}_c(D_{X_{\rm A}},D_{X_{\rm O}})
\end{split}
\end{equation*}

Similarly, we can obtain that
\begin{equation*}
\begin{split}
&{\rm W}_c(D', D_{X_{\rm A}}) =  (1-u) {\rm W}_c(D_{X_{\rm A}},D_{X_{\rm O}})
\end{split}
\end{equation*}
\textbf{Case 1.} If ${\rm W}_c(D_{X_{\rm A}},D_{X_{\rm O}})\leq \rho$, then it is clear that
\begin{equation*}
\begin{split}
&\inf_{D_{X'} \in \mathfrak{D} }{\rm W}_{\rm c}(D_{X'},D_{X_{\rm O}}) \leq \max\{{\rm W}_{c}(D_{X_{\rm A}},D_{X_{\rm O}})-\rho, 0\}.
\end{split}
\end{equation*}

\textbf{Case 2.} If ${\rm W}_c(D_{X_{\rm A}},D_{X_{\rm O}})> \rho$, then we set $u =1- \rho / {\rm W}_c(D_{X_{\rm A}},D_{X_{\rm O}})$. Therefore,
\begin{equation*}
{\rm W}_c(D', D_{X_{\rm A}}) = \rho,~~~~~~{\rm W}_c(D', D_{X_{\rm O}})={\rm W}_c(D_{X_{\rm A}},D_{X_{\rm O}})-\rho,
\end{equation*}
which implies that
\begin{equation*}
\inf_{D_{X'} \in \mathfrak{D} }{\rm W}_{\rm c}(D_{X'},D_{X_{\rm O}}) \leq {\rm W}_c(D', D_{X_{\rm O}}) = {\rm W}_c(D_{X_{\rm A}},D_{X_{\rm O}})-\rho \leq \max\{{\rm W}_{c}(D_{X_{\rm A}},D_{X_{\rm O}})-\rho, 0\}.
\end{equation*}
We have completed this proof by combining Cases 1 and 2.
\end{proof}

%By Kantorovich–Rubinstein duality \cite{Villani2003TopicsIO}, we also obtain that
%\begin{equation*}
%\begin{split}
%&\sup_{{\rm W}_{c}(D_{X'},D_{X_{\rm A}}) \leq \rho+\delta} \mathbb{E}_{\mathbf{x}\sim D_{X'}} \ell_{\rm OE}(\theta; \mathbf{x}) - \sup_{{\rm W}_{c}(D_{X'},D_{X_{\rm A}}) \leq \rho} \mathbb{E}_{\mathbf{x}\sim D_{X'}} \ell_{\rm OE}(\theta; \mathbf{x}) \\ \leq &
 %\mathbb{E}_{\mathbf{x}\sim D_{X'}^{\delta,\epsilon}} \ell_{\rm OE}(\theta; \mathbf{x}) - \mathbb{E}_{\mathbf{x}\sim D_{X'}'} \ell_{\rm OE}(\theta; \mathbf{x})+\epsilon
 %\\ 
 %\leq & L_c \delta +\epsilon,
%\end{split}
%\end{equation*}
%which implies that
%\begin{equation*}
%\begin{split}
%\sup_{{\rm W}_{c}(D_{X'},D_{X_{\rm A}}) \leq \rho+\delta} \mathbb{E}_{\mathbf{x}\sim D_{X'}} \ell_{\rm OE}(\theta; \mathbf{x}) - \sup_{{\rm W}_{c}(D_{X'},D_{X_{\rm A}}) \leq \rho} \mathbb{E}_{\mathbf{x}\sim D_{X'}} \ell_{\rm OE}(\theta; \mathbf{x})
 %\leq  L_c \delta.
%\end{split}
%\end{equation*}

\subsection{Proof of Theorem \ref{T-dual}}
\begin{proof}[Proof of Theorem \ref{T-dual}]
One can find a similar proof from \citet{Blanchet2016RobustWP,Blanchet2016QuantifyingDM,SinhaND18}. We omit it here.
\end{proof}

\subsection{Proof of Theorem~\ref{T1}}
\begin{proof}[Proof of Theorem \ref{T1}] We first recall the notations as follows:
\begin{equation*}
R_{{\rm I}}({\mathbf{w}}) = \mathbb{E}_{(\mathbf{x},y)\sim D_{X_{\rm I}Y_{\rm I}}} \ell(\mathbf{f}_{\mathbf{w}}; \mathbf{x},y),
\end{equation*}
\begin{equation*}
\widehat{R}_{{\rm I}}(\mathbf{w}) = \frac{1}{n} \sum_{i=1}^n\ell(\mathbf{f}_{\mathbf{w}}; \mathbf{x}_{\rm I}^i,y_{\rm I}^i),
\end{equation*}
\begin{equation*}
R_{{\rm O}}(\mathbf{w};\rho) = \sup_{{\rm W}_{c}(D_{X'},D_{X_{\rm A}}) \leq \rho} \mathbb{E}_{\mathbf{x}\sim D_{X'}} \ell_{\rm OE}(\mathbf{f}_{\mathbf{w}}; \mathbf{x}),
\end{equation*}
\begin{equation*}
\widehat{R}_{{\rm O}}(\mathbf{w};\rho) = \sup_{{\rm W}_{c}(D_{X'},\widehat{D}_{X_{\rm A}}) \leq \rho} \mathbb{E}_{\mathbf{x}\sim D_{X'}} \ell_{\rm OE}(\mathbf{f}_{\mathbf{w}}; \mathbf{x}).
\end{equation*}
Let $\mathbf{w}^*$ be the solution of $\min_{\mathbf{w} \in \mathcal{W}} {R}_D(\mathbf{w}; \rho)$. Then
\begin{equation}\label{T1.1}
\begin{split}
& {R}_D(\widehat{\mathbf{w}}; \rho)- {R}_D(\mathbf{w}^*; \rho) \\
 \leq & {R}_D(\widehat{\mathbf{w}}; \rho) - \widehat{R}_D(\widehat{\mathbf{w}}; \rho)+\widehat{R}_D(\widehat{\mathbf{w}}; \rho)-{R}_D(\mathbf{w}^*; \rho)+\widehat{R}_D({\mathbf{w}}^*; \rho)-\widehat{R}_D({\mathbf{w}}^*; \rho)\\ \leq &
 [R_{\rm I}(\widehat{\mathbf{w}})-R_{\rm I}(\mathbf{w}^*)] + \alpha [R_{\rm O}(\widehat{\mathbf{w}};\rho)-R_{\rm O}(\mathbf{w}^*;\rho)] - [\widehat{R}_{\rm I}(\widehat{\mathbf{w}})-\widehat{R}_{\rm I}(\mathbf{w}^*)] - \alpha [\widehat{R}_{\rm O}(\widehat{\mathbf{w}};\rho)-\widehat{R}_{\rm O}(\mathbf{w}^*;\rho)]
 \\
 = & 
 [R_{\rm I}(\widehat{\mathbf{w}})-\widehat{R}_{\rm I}(\widehat{\mathbf{w}})] + \alpha [R_{\rm O}(\widehat{\mathbf{w}};\rho)-\widehat{R}_{\rm O}(\widehat{\mathbf{w}};\rho)]  - [R_{\rm I}(\mathbf{w}^*)-\widehat{R}_{\rm I}(\mathbf{w}^*)] - \alpha [{R}_{\rm O}(\mathbf{w}^*;\rho)-\widehat{R}_{\rm O}(\mathbf{w}^*;\rho)].
\end{split}
\end{equation}

By Lemmas \ref{L3} and \ref{L8}, we have that with the probability at least $1-2e^{-t}>0$, for any $\rho>0$,
\begin{equation}\label{T1.2}
\begin{split}
&[R_{\rm I}(\widehat{\mathbf{w}})-\widehat{R}_{\rm I}(\widehat{\mathbf{w}})] +\alpha [R_{\rm O}(\widehat{\mathbf{w}};\rho)-\widehat{R}_{\rm O}(\widehat{\mathbf{w}};\rho)] 
\\
\leq &  \frac{b_0 M_{\ell}}{\sqrt{n}}  \int_{0}^{1} \sqrt{\log\mathcal{N}(\mathcal{F},M_{\ell}\epsilon,L^{\infty})}d\epsilon 
\\
+ & \alpha b_1 \sqrt{\frac{M_{\ell_{\rm OE}}^3}{\rho^2 m}}\int_{0}^{1} \sqrt{ \log \mathcal{N}(\mathcal{F}_{\rm OE},M_{\ell_{\rm OE}}\epsilon,L^{\infty})}{\rm d} \epsilon + \alpha b_2 M_{\ell_{\rm OE}} \sqrt{\frac{2t}{m}}+M_{\ell}\sqrt{\frac{2t}{n}}.
\end{split}
\end{equation}

By Lemmas \ref{L7} and \ref{L9}, we have that with the probability at least $1-2e^{-t}>0$, for any $\rho>0$,
\begin{equation}\label{T1.3}
\begin{split}
 &[R_{\rm I}(\mathbf{w}^*)-\widehat{R}_{\rm I}(\mathbf{w}^*)]
 +\alpha [{R}_{\rm O}(\mathbf{w}^*;\rho)-\widehat{R}_{\rm O}(\mathbf{w}^*;\rho)]
 \leq  M_{\ell}\sqrt{\frac{2t}{n}} + 2\alpha M_{\ell_{\rm OE}} \sqrt{\frac{2t}{m}}.
 \end{split}
\end{equation}

Combining Eqs. \eqref{T1.1}, \eqref{T1.2} and \eqref{T1.3}, we have that with the probability at least $1-4e^{-t}>0$, for any $\rho>0$,
\begin{equation*}
\begin{split}
&{R}_D(\widehat{\mathbf{w}}; \rho)- {R}_D(\mathbf{w}^*; \rho)
\\ \leq & \frac{b_0 M_{\ell}}{\sqrt{n}}  \int_{0}^{1} \sqrt{\log\mathcal{N}(\mathcal{F},M_{\ell}\epsilon,L^{\infty})}d\epsilon 
\\
+ & \alpha b_1 \sqrt{\frac{M_{\ell_{\rm OE}}^3}{\rho^2 m}}\int_{0}^{1} \sqrt{ \log \mathcal{N}(\mathcal{F}_{\rm OE},M_{\ell_{\rm OE}}\epsilon,L^{\infty})}{\rm d} \epsilon + 2M_{\ell}\sqrt{\frac{2t}{n}}+ b_{2}\alpha M_{\ell_{\rm OE}} \sqrt{\frac{2t}{m}},
\end{split}
\end{equation*}
where $b_0$, $b_1$ and $b_2$ are uniform constants.
\end{proof}
\begin{remark}
     One can prove that the cross-entropy and the exponential losses are bounded and lipschitz w.r.t. $\mathbf{w}$ for deep models with softmax outputs~\citep{golowich2018size}, if 
     \begin{itemize}
         \item activation functions are 1-Lipschitz;
         \item inputs are from bounded feature space $\mathcal{X}$;
         \item the parameter space $\mathcal{W}$ is bounded (e.g., with regularization).
     \end{itemize}
     More specifically, when the F-norm bounds parameters, the softmax output is continuous and never attains infinity. If we further assume that inputs are from a bounded feature space, then the model is a continuous function over the bounded space, implying that model outputs have upper and lower bounds. Thus, the cross-entropy and the exponential loss can be bounded in practice, and our assumptions are practical.
\end{remark}

\subsection{Corollary~\ref{T2}}\label{C1inApp}

\begin{corollary}\label{T2} Given the same conditions in Theorem \ref{T1}, if
\begin{itemize}
    \item $\ell_{\rm OE}(\mathbf{f}_{\mathbf{w}}; \mathbf{x})$ is  $L_{\rm OE}$-Lipschitz w.r.t. norm $\|\cdot\|$, i.e.,
\begin{equation*}
|\ell_{\rm OE}(\mathbf{f}_{\mathbf{w}}; \mathbf{x})-\ell_{\rm OE}(\mathbf{f}_{\mathbf{w}'}; \mathbf{x})|\leq L_{\rm OE}\|\mathbf{w}-\mathbf{w}'\|,
\end{equation*}
    \item $\ell_{\rm OE}(\mathbf{f}_{\mathbf{w}}; \mathbf{x})$ is  $L_{c}$-Lipschitz w.r.t. $c(\cdot,\cdot)$, i.e., 
\begin{equation*}
|\ell_{\rm OE}(\mathbf{f}_{\mathbf{w}}; \mathbf{x})-\ell_{\rm OE}(\mathbf{f}_{\mathbf{w}}; \mathbf{x}')|\leq L_{c}c(\mathbf{x},\mathbf{x}'),
\end{equation*}
    \item $\ell(\mathbf{f}_{\mathbf{w}}; \mathbf{x},y)$ is  $L$-Lipschitz w.r.t. norm $\|\cdot\|$, i.e., 
\begin{equation*}
|\ell(\mathbf{f}_{\mathbf{w}}; \mathbf{x},y)-\ell(\mathbf{f}_{\mathbf{w}'}; \mathbf{x},y)|\leq L\|\mathbf{w}-\mathbf{w}'\|,
\end{equation*}
    \item the parameter space $\mathcal{W} \subset \mathbb{R}^{d'}$ satisfies that 
\begin{equation*}
{\rm diam}(\mathcal{W}) = \sup_{\mathbf{w},\mathbf{w}'\in \mathcal{W}} \|\mathbf{w}-\mathbf{w}'\|<+\infty.
\end{equation*}
\end{itemize}
Let $\widehat{\mathbf{w}}$ be the optimal solution of Eq. \eqref{Eq::Objective}, i.e.,
\begin{equation*}
\widehat{\mathbf{w}} \in \argmin_{\mathbf{w}\in \mathcal{W}}  \widehat{R}_D({\mathbf{w}}; \rho).
\end{equation*}
With the probability at least $1-4e^{-t}>0$, for any $\rho>0$,
\begin{equation*}
R_D(\widehat{\mathbf{w}};\rho) - \min_{\mathbf{w} \in \mathcal{W}} R_D({\mathbf{w}};\rho)\leq \tilde{\epsilon}(n,m;t),
\end{equation*}
where 
\begin{equation*}
\begin{split}
\tilde{\epsilon}(n,m;t)= &  {b_0}\sqrt{\frac{M_{\ell}{\rm diam}(\mathcal{W}) Ld'}{n}} 
\\
+ & \alpha b_1 \min \{L_c ,  \frac{M_{\ell_{\rm OE}} }{\rho} \} \sqrt{\frac{{\rm diam}(\mathcal{W}) L_{\rm OE}d'}{m}}
\\
+ &2M_{\ell}\sqrt{\frac{2t}{n}}+ \alpha b_{2}M_{\ell_{\rm OE}} \sqrt{\frac{2t}{m}},
\end{split}
\end{equation*}
here $b_0$, $b_1$ and $b_2$ are uniform constants.
\end{corollary}

\begin{proof}[Proof of Corollary \ref{T2}]
By Lemmas \ref{L5} and \ref{L10}, we have that with the probability at least $1-2e^{-t}>0$, for any $\rho>0$,
\begin{equation}\label{T2.1}
\begin{split}
&[R_{\rm I}(\widehat{\mathbf{w}})-\widehat{R}_{\rm I}(\widehat{\mathbf{w}})] +\alpha [R_{\rm O}(\widehat{\mathbf{w}};\rho)-\widehat{R}_{\rm O}(\widehat{\mathbf{w}};\rho)] 
\\
\leq & \big [{b_0}\sqrt{\frac{M_{\ell}{\rm diam}(\mathcal{W}) Ld'}{n}} + M_{\ell}\sqrt{\frac{2t}{n}}\big ]
 +  \alpha b_1 \min \{L_c ,  \frac{M_{\ell_{\rm OE}} }{\rho} \} \sqrt{\frac{{\rm diam}(\mathcal{W}) L_{\rm OE}d'}{m}} + \alpha b_2 M_{\ell_{\rm OE}} \sqrt{\frac{t}{m}}
\end{split}
\end{equation}
By Lemmas \ref{L7} and \ref{L9}, we have that with the probability at least $1-2e^{-t}>0$, for any $\rho>0$,
\begin{equation}\label{T2.2}
\begin{split}
 &[R_{\rm I}(\mathbf{w}^*)-\widehat{R}_{\rm I}(\mathbf{w}^*)]
 + \alpha [{R}_{\rm O}(\mathbf{w}^*;\rho)-\widehat{R}_{\rm O}(\mathbf{w}^*;\rho)] \leq   M_{\ell}\sqrt{\frac{2t}{n}} + 2\alpha M_{\ell_{\rm OE}} \sqrt{\frac{2t}{m}}.
 \end{split}
\end{equation}
Using  Eqs. \eqref{T2.1}, \eqref{T2.2} and Eq. \eqref{T1.1}, we know that with the probability at least $1-4e^{-t}>0$, for any $\rho>0$,
\begin{equation*}
\begin{split}
&R_D(\widehat{\mathbf{w}};\rho) - \min_{\mathbf{w} \in \mathcal{W}} R_D({\mathbf{w}};\rho)\\ \leq & {b_0}\sqrt{\frac{M_{\ell}{\rm diam}(\mathcal{W}) Ld'}{n}} + \alpha b_1 \min \{L_c ,  \frac{M_{\ell_{\rm OE}} }{\rho} \} \sqrt{\frac{{\rm diam}(\mathcal{W}) L_{\rm OE}d'}{m}} + 2 M_{\ell}\sqrt{\frac{2t}{n}}+ \alpha b_{2}M_{\ell_{\rm OE}} \sqrt{\frac{2t}{m}},
\end{split}
\end{equation*}
where $b_0$, $b_1$ and $b_2$ are uniform constants.
\end{proof}

\subsection{Proof of Theorem~\ref{T3}}

\begin{proof}[Proof of Theorem \ref{T3}]
Consider 
\begin{equation*}
R_D(\widehat{\mathbf{w}}) -R_D(\widehat{\mathbf{w}};\rho).
\end{equation*}
It is clear that
\begin{equation*}
R_D(\widehat{\mathbf{w}}) -R_D(\widehat{\mathbf{w}};\rho)= \alpha [R_{\rm O}(\widehat{\mathbf{w}})-R_{\rm O}(\widehat{\mathbf{w}};\rho)].
\end{equation*}
Let $D' = (1-u)D_{X_{\rm O}}+uD_{X_{\rm A}}$ and
\begin{equation*}
\rho_{\rm O} = {\rm W}_c(D_{X_{\rm O}},D_{X_{\rm A}}).
\end{equation*}
Because $c(\cdot,\cdot)$ is a continuous metric, Kantorovich–Rubinstein duality \cite{Villani2003TopicsIO} implies that
\begin{equation*}
\begin{split}
&{\rm W}_c(D', D_{X_{\rm A}}) \\= &\sup_{\|f\|_{\rm Lip}\leq 1} \int_{\mathcal{X}} f(\mathbf{x}) {\rm d}D'(\mathbf{x}) - \int_{\mathcal{X}} f(\mathbf{x}) {\rm d}D_{X_{\rm A}}(\mathbf{x}) 
\\ = & (1-u ) \sup_{\|f\|_{\rm Lip}\leq 1} \int_{\mathcal{X}} f(\mathbf{x}) {\rm d}D_{X_{\rm O}}(\mathbf{x}) - \int_{\mathcal{X}} f(\mathbf{x}) {\rm d}D_{X_{\rm A}}(\mathbf{x}) 
\\
=&(1-u){\rm W}_c(D_{X_{\rm O}},D_{X_{\rm A}})
\\
= &(1-u) \rho_{\rm O}
\end{split}
\end{equation*}
Let
\begin{equation*}
u = 1-\frac{\rho}{\rho_{\rm O}}.
\end{equation*}
\textbf{Case 1.} If $\rho \geq \rho_{\rm O}$, then
\begin{equation*}
R_D(\widehat{\mathbf{w}})\leq R_D(\widehat{\mathbf{w}};\rho)
\end{equation*}
\textbf{Case 2.} If $\rho < \rho_{\rm O}$, then by Lemma \ref{L5.0}
\begin{equation*}
\begin{split}
R_{\rm O}(\widehat{\mathbf{w}})- R_{\rm O}(\widehat{\mathbf{w}};\rho) &\leq R_{\rm O}(\widehat{\mathbf{w}}) - \mathbb{E}_{\mathbf{x}\sim D'} \ell_{\rm OE}(\mathbf{f}_\mathbf{w};\mathbf{x}) \leq L_c(\rho_{\rm O}-\rho),
\end{split}
\end{equation*}
By Cases $1$ and $2$, we have shown that
\begin{equation*}
\begin{split}
&R_D(\widehat{\mathbf{w}}) -R_D(\widehat{\mathbf{w}};\rho)
\leq \alpha L_c \max\{{\rm W}_c(D_{X_{\rm o}},D_{X_{\rm A}})-\rho,0\}.
\end{split}
\end{equation*}
Then by Theorem \ref{T1}, we complete this proof.
\end{proof}

\newpage
\section{Necessary Lemmas}

\begin{lemma}\label{L1}
Assume that 
\begin{itemize}
    \item $|\ell_{\rm OE}(\mathbf{f}_{\mathbf{w}};\mathbf{x})| \leq M_{\ell_{\rm OE}}$,
    \item $c:\mathcal{X}\times \mathcal{X} \rightarrow \mathbb{R}_{+}$ is a continuous metric, 
\end{itemize}
then
\begin{itemize}
    \item $|\phi_{\gamma}(\mathbf{w};\mathbf{x})| \leq M_{\ell_{\rm OE}}$,
    \item  for some $\mathbf{w}_0 \in \mathcal{W}$ and $\epsilon>0$, when $\gamma_{\epsilon}$ $(\gamma_{\epsilon}\geq 0)$ satisfies the following condition:
\begin{equation*}
\gamma_{\epsilon} \rho + \mathbb{E}_{\mathbf{x}\sim D}\phi_{\gamma_{\epsilon}}(\mathbf{w}_0;\mathbf{x}) \leq \inf_{\gamma \geq 0} [ \gamma \rho + \mathbb{E}_{\mathbf{x}\sim D}\phi_{\gamma}(\mathbf{w}_0;\mathbf{x})] + \epsilon,
\end{equation*}
then
\begin{equation*}
\gamma_{\epsilon} \leq \frac{2M_{\ell_{\rm OE}}+\epsilon}{\rho}.
\end{equation*}
\end{itemize}
\end{lemma}
\begin{proof}[Proof of Lemma \ref{L1}]
   \textbf{First}, we prove: $|\phi_{\gamma}(\mathbf{w};\mathbf{x})| \leq M_{\ell_{\rm OE}}$.

   Because $\phi_{\gamma}(\mathbf{w};\mathbf{x}) = \sup_{\mathbf{x}'\in \mathcal{X}} \{ \ell_{\rm OE}(\mathbf{f}_\mathbf{w}; \mathbf{x}') - \gamma c(\mathbf{x}',\mathbf{x})\}$ and $c(\mathbf{x},\mathbf{x}') \geq 0$, it is clear that
   \begin{equation*}
   \phi_{\gamma}(\mathbf{w};\mathbf{x}) \leq \sup_{\mathbf{x}'\in \mathcal{X}} \ell_{\rm OE}(\mathbf{f}_{\mathbf{w}}; \mathbf{x}') \leq M_{\ell_{\rm OE}}.
   \end{equation*}
   In addition, because $c(\mathbf{x},\mathbf{x})=0$, then
   \begin{equation*}
   \phi_{\gamma}(\mathbf{w};\mathbf{x}) \geq  \ell_{\rm OE}(\mathbf{w}; \mathbf{x}) \geq -M_{\ell_{\rm OE}}.
   \end{equation*}
   Above inequalities have indicated that
   \begin{equation*}
       |\phi_{\gamma}(\mathbf{w};\mathbf{x})| \leq M_{\ell_{\rm OE}}.
   \end{equation*}

\textbf{Second}, we prove that 
\begin{equation*}
\gamma_{\epsilon} \leq \frac{2M_{\ell_{\rm OE}}+\epsilon}{\rho}.
\end{equation*}

By the dual theorem in \citet{Blanchet2016RobustWP,Blanchet2016QuantifyingDM,SinhaND18}, we can obtain that
% \begin{equation*}
% \begin{split}
 % &\inf_{\gamma \geq 0} [ \gamma \rho + \mathbb{E}_{\mathbf{x}\sim D}\phi_{\gamma}(\mathcal{W}_0;\mathbf{x})]  \\= &\sup_{{\rm W}_c(D',D)\leq \rho} \mathbb{E}_{\mathbf{x}\sim D'} \ell_{\rm OE}(\mathcal{W}_0; \mathbf{x}) \leq M_{\ell_{\rm OE}},
 % \end{split}
% \end{equation*}
\begin{equation*}
    \inf_{\gamma \geq 0} [ \gamma \rho + \mathbb{E}_{\mathbf{x}\sim D}\phi_{\gamma}(\mathbf{w}_0;\mathbf{x})]  = \sup_{{\rm W}_c(D',D)\leq \rho} \mathbb{E}_{\mathbf{x}\sim D'} \ell_{\rm OE}(\mathbf{f}_{\mathbf{w}_0}; \mathbf{x}) \leq M_{\ell_{\rm OE}},
\end{equation*}
which implies that
% \begin{equation*}
% \begin{split}
% \gamma_{\epsilon}\rho &\leq M_{\ell_{\rm OE}}+\epsilon -  \mathbb{E}_{\mathbf{x}\sim D}\phi_{\gamma}(\mathcal{W}_0;\mathbf{x}) \\ & \leq 2M_{\ell_{\rm OE}}+\epsilon.
% \end{split}
% \end{equation*}
\begin{equation*}
    \gamma_{\epsilon}\rho \leq M_{\ell_{\rm OE}}+\epsilon -  \mathbb{E}_{\mathbf{x}\sim D}\phi_{\gamma}(\mathbf{w}_0;\mathbf{x})  \leq 2M_{\ell_{\rm OE}}+\epsilon.
\end{equation*}
Therefore,
\begin{equation*}
    \gamma_{\epsilon} \leq \frac{2M_{\ell_{\rm OE}}+\epsilon}{\rho}.
\end{equation*}
\end{proof}

\begin{lemma}[Theorem 3 in \citet{SinhaND18}]\label{L2}
Given the same assumptions in Theorem \ref{T1}, then with the probability at least $1-e^{-t}>0$, for any $\gamma\geq 0$, $\rho \geq 0$ and $\mathbf{w} \in \mathcal{W}$, 
\begin{equation*}
\begin{split}
&\sup_{{\rm W}_{c}(D_{X'},D_{X_{\rm A}}) \leq \rho} \mathbb{E}_{\mathbf{x}\sim D_{X'}} \ell_{\rm OE}(\mathbf{f}_{\mathbf{w}}; \mathbf{x}) \\ \leq & \gamma \rho + \frac{1}{m} \sum_{i=1}^m \phi_{\gamma}(\mathbf{w};\mathbf{x}_{\rm A}^i)
 +b_2 M_{\ell_{\rm OE}} \sqrt{\frac{t}{m}} + \gamma b_1 \sqrt{\frac{M_{\ell_{\rm OE}}}{m}}\int_{0}^{1} \sqrt{ \log \mathcal{N}(\mathcal{F}_{\rm OE},M_{\ell_{\rm OE}}\epsilon,L^{\infty})}{\rm d} \epsilon, 
\end{split}
\end{equation*}
where $b_1$ and $b_2$ are uniform constants.
\end{lemma}
\begin{proof}[Proof of Lemma \ref{L2}]
This lemma is following Theorem 3 in \citet{SinhaND18}.
\end{proof}

\begin{lemma}\label{L3}
Given the same assumptions in Theorem \ref{T1}, then with the probability at least $1-e^{-t}>0$, for any $\rho > 0$ and $\mathbf{w} \in \mathcal{W}$, we have
\begin{equation*}
 \begin{split}
 &\sup_{{\rm W}_{c}(D_{X'},D_{X_{\rm A}}) \leq \rho} \mathbb{E}_{\mathbf{x}\sim D_{{\rm O}}} \ell_{\rm OE}(\mathbf{f}_\mathbf{w}; \mathbf{x}) \\ \leq & \sup_{{\rm W}_{c}(D_{X'},\widehat{D}_{X_{\rm A}}) \leq \rho} \mathbb{E}_{\mathbf{x}\sim D_{{\rm O}}} \ell_{\rm OE}(\mathbf{f}_\mathbf{w}; \mathbf{x}) 
 +b_2 M_{\ell_{\rm OE}}\sqrt{\frac{t}{m}}
 +   b_1 \sqrt{\frac{M_{\ell_{\rm OE}}^3}{\rho^2 m}}\int_{0}^{1} \sqrt{ \log \mathcal{N}(\mathcal{F}_{\rm OE},M_{\ell_{\rm  OE}}\epsilon,L^{\infty})}{\rm d} \epsilon,
 \end{split}
\end{equation*}

where $b_1$ and $b_2$ are uniform constants.
\end{lemma}
\begin{proof}[Proof of Lemma \ref{L3}]
By Lemma \ref{L2}, we know that  with the probability at least $1-e^{-t}>0$, for any $\frac{3M_{\ell_{\rm OE}}}{\rho} \geq \gamma\geq 0$, $ \rho \geq 0$ and $\mathbf{w} \in \mathcal{W}$, we have
\begin{equation*}
\begin{split}
&\sup_{{\rm W}_{c}(D_{X'},D_{X_{\rm A}}) \leq \rho} \mathbb{E}_{\mathbf{x}\sim D_{X'}} \ell_{\rm OE}(\mathbf{f}_\mathbf{w}; \mathbf{x}) \\ \leq & \gamma \rho + \frac{1}{m} \sum_{i=1}^m \phi_{\gamma}(\mathbf{w};\mathbf{x}_{\rm A}^i)
 +b_2M_{\ell_{\rm OE}} \sqrt{\frac{t}{m}} + b_1 \sqrt{\frac{M_{\ell_{\rm OE}}^3}{\rho^2 m}} \int_{0}^{1} \sqrt{ \log \mathcal{N}(\mathcal{F}_{\rm OE},M_{\ell_{\rm OE}}\epsilon,L^{\infty})}{\rm d} \epsilon,
\end{split}
\end{equation*}
where $b_1$ and $b_2$ are uniform constants.

The above bound and Lemma \ref{L1} imply that
 with the probability at least $1-e^{-t}>0$, for any $ \rho \geq 0$ and $\mathbf{w} \in \mathcal{W}$, we have
\begin{equation*}
\begin{split}
&\sup_{{\rm W}_{c}(D_{X'},D_{X_{\rm A}}) \leq \rho} \mathbb{E}_{\mathbf{x}\sim D_{X'}} \ell_{\rm OE}(\mathbf{f}_\mathbf{w}; \mathbf{x}) \\ \leq & \inf_{\gamma \geq 0} \big(\gamma \rho + \frac{1}{m} \sum_{i=1}^m \phi_{\gamma}(\mathbf{w};\mathbf{x}_{\rm A}^i)\big )
 +b_2 M_{\ell_{\rm OE}} \sqrt{\frac{t}{m}} + b_1 \sqrt{\frac{M_{\ell_{\rm OE}}^3}{\rho^2 m}} \int_{0}^{1} \sqrt{ \log \mathcal{N}(\mathcal{F}_{\rm OE},M_{\ell_{\rm OE}}\epsilon,L^{\infty})}{\rm d} \epsilon.
\end{split}
\end{equation*}
Combining the above inequality with the following equation:
\begin{equation*}
\begin{split}
\inf_{\gamma \geq 0} \big(\gamma \rho + \frac{1}{m} \sum_{i=1}^m \phi_{\gamma}(\mathbf{w};\mathbf{x}_{\rm A}^i)\big ) =\sup_{{\rm W}_{c}(D_{X'},\widehat{D}_{X_{\rm A}}) \leq \rho} \mathbb{E}_{\mathbf{x}\sim D_{X'}} \ell_{\rm OE}(\mathbf{f}_\mathbf{w}; \mathbf{x}),
\end{split}
\end{equation*}
we complete this proof.
\end{proof}

\begin{lemma}\label{L4}
Given the same assumptions in Lemma \ref{L3}, if $\ell_{\rm OE}(\mathbf{f}_\mathbf{w}; \mathbf{x})$ is a $L_{\rm OE}$-Lipschitz function w.r.t. norm $\|\cdot\|$ for all $\mathbf{x}\in \mathcal{X}$ and the parameter space $\mathcal{W}\subset \mathbb{R}^{d'}$ satisfies that ${\rm diam}(\mathcal{W}) = \sup_{\mathbf{w},\mathbf{w}'\in \mathcal{W}} \|\mathbf{w}-\mathbf{w}'\|<+\infty$, then with the probability at least $1-e^{-t}>0$, for any $\rho > 0$ and $\mathbf{w} \in \mathcal{W}$, 
\begin{equation*}
\begin{split}
&\sup_{{\rm W}_{c}(D_{X'},D_{X_{\rm A}}) \leq \rho} \mathbb{E}_{\mathbf{x}\sim D_{{\rm O}}} \ell_{\rm OE}(\mathbf{f}_\mathbf{w}; \mathbf{x}) \\ \leq & \sup_{{\rm W}_{c}(D_{X'},\widehat{D}_{X_{\rm A}}) \leq \rho} \mathbb{E}_{\mathbf{x}\sim D_{{\rm O}}} \ell_{\rm OE}(\mathbf{f}_\mathbf{w}; \mathbf{x}) +  b_2 M_{\ell_{\rm OE}}\sqrt{\frac{t}{m}} + b_1 M_{\ell_{\rm OE}} \sqrt{\frac{{\rm diam}(\mathcal{W}) L_{\rm OE}d'}{\rho^2 m}},
\end{split}
\end{equation*}
where $b_1$ and $b_2$ are uniform constants.
\end{lemma}
\begin{proof}
The proof is similar to Corollary 1 in \citet{SinhaND18}. Note that 
\begin{equation*}
\mathcal{F}_{\rm OE} = \{\ell_{\rm OE}(\mathbf{f}_\mathbf{w};\mathbf{x}): \mathbf{w} \in \mathcal{W}\},
\end{equation*}
and $\ell_{\rm OE}(\mathbf{f}_\mathbf{w}; \mathbf{x})$ is $L_{\rm OE}$-Lipschitz w.r.t. norm $\|\cdot\|$, therefore,
\begin{equation*}
\begin{split}
\mathcal{N}(\mathcal{F}_{\rm OE},M_{\ell_{\rm OE}}\epsilon,L^{\infty}) \leq  & \mathcal{N}(\mathcal{W},M_{\ell_{\rm OE}}\epsilon/L_{\rm OE},\|\cdot\|)
\leq (1+ \frac{{\rm diam}(\mathcal{W}) L_{\rm OE}}{M_{\ell_{\rm OE}}\epsilon})^{d'},
\end{split}
\end{equation*}
which implies that
\begin{equation*}
\begin{split}
& \int_{0}^{1} \sqrt{\log(\mathcal{N}(\mathcal{F}_{\rm OE},M_{\ell_{\rm OE}}\epsilon,L^{\infty})} {\rm d} \epsilon
\\ 
\leq & \sqrt{d'} \int_{0}^1 \sqrt{\log (1+ \frac{{\rm diam}(\mathcal{W}) L_{\rm OE}}{M_{\ell_{\rm OE}}\epsilon})}  {\rm d} \epsilon
\\ 
\leq & \sqrt{d'} \int_{0}^1 \sqrt{\frac{{\rm diam}(\mathcal{W}) L_{\rm OE}}{M_{\ell_{\rm OE}}\epsilon}}  {\rm d} \epsilon
= 2\sqrt{\frac{{\rm diam}(\mathcal{W}) L_{\rm OE}d'}{M_{\ell_{\rm OE}}}}.
\end{split}
\end{equation*}
By Lemma \ref{L3}, we obtain that there exist two uniform constants such that with the probability at least $1-e^{-t}>0$,
\begin{equation*}
\begin{split}
&\sup_{{\rm W}_{c}(D_{X'},D_{X_{\rm A}}) \leq \rho} \mathbb{E}_{\mathbf{x}\sim D_{{\rm O}}} \ell_{\rm OE}(\mathbf{f}_\mathbf{w}; \mathbf{x}) \\ \leq & \sup_{{\rm W}_{c}(D_{X'},\widehat{D}_{X_{\rm A}}) \leq \rho} \mathbb{E}_{\mathbf{x}\sim D_{{\rm O}}} \ell_{\rm OE}(\mathbf{f}_\mathbf{w}; \mathbf{x}) +  b_2 M_{\ell_{\rm OE}}\sqrt{\frac{t}{m}} + b_1 M_{\ell_{\rm OE}} \sqrt{\frac{{\rm diam}(\mathcal{W}) L_{\rm OE}d'}{\rho^2 m}},
\end{split}
\end{equation*}
where $b_1$ and $b_2$ are uniform constants.
\end{proof}

\begin{lemma}\label{L5.0}
Given the same assumptions in Theorem \ref{T1}, and for any $\mathbf{w} \in \mathcal{W}$ and any $\mathbf{x}, \mathbf{x}'\in \mathcal{X}$,
\begin{equation*}
|\ell_{\rm OE}(\mathbf{f}_\mathbf{w}; \mathbf{x})-\ell_{\rm OE}(\mathbf{f}_\mathbf{w}; \mathbf{x}')|\leq L_{c}c(\mathbf{x},\mathbf{x}'),
\end{equation*}
 then for any $\delta\geq 0$,
\begin{equation*}
\begin{split}
&\sup_{{\rm W}_{c}(D_{X'},D_{X_{\rm A}}) \leq \rho+\delta} \mathbb{E}_{\mathbf{x}\sim D_{X'}} \ell_{\rm OE}(\mathbf{f}_\mathbf{w}; \mathbf{x}) - \sup_{{\rm W}_{c}(D_{X'},D_{X_{\rm A}}) \leq \rho} \mathbb{E}_{\mathbf{x}\sim D_{X'}} \ell_{\rm OE}(\mathbf{f}_\mathbf{w}; \mathbf{x})  \leq  L_c \delta.
\end{split}
\end{equation*}
\end{lemma}

\begin{proof}[Proof of Lemma \ref{L5.0}]
For each $\epsilon>0$, we set $D_{X'}^{\delta,\epsilon}$ satisfies that
\begin{equation*}
\begin{split}
\sup_{{\rm W}_{c}(D_{X'},D_{X_{\rm A}}) \leq \rho+\delta} &\mathbb{E}_{\mathbf{x}\sim D_{X'}} \ell_{\rm OE}(\mathbf{f}_\mathbf{w}; \mathbf{x}) \leq \mathbb{E}_{\mathbf{x}\sim D_{X'}^{\delta,\epsilon}} 
\ell_{\rm OE}(\mathbf{f}_\mathbf{w}; \mathbf{x})+\epsilon,
\end{split}
\end{equation*}
and
\begin{equation*}
{\rm W}_c(D_{X'}^{\delta,\epsilon},D_{X_{\rm A}}) \leq \rho+\delta.
\end{equation*}

\textbf{Case 1.} If 
\begin{equation*}
{\rm W}_c(D_{X'}^{\delta,\epsilon},D_{X_{\rm A}}) \leq \rho,
\end{equation*}
then
\begin{equation*}
\begin{split}
&\sup_{{\rm W}_{c}(D_{X'},D_{X_{\rm A}}) \leq \rho+\delta} \mathbb{E}_{\mathbf{x}\sim D_{X'}} \ell_{\rm OE}(\mathbf{f}_\mathbf{w}; \mathbf{x}) - \sup_{{\rm W}_{c}(D_{X'},D_{X_{\rm A}}) \leq \rho} \mathbb{E}_{\mathbf{x}\sim D_{X'}} \ell_{\rm OE}(\mathbf{f}_\mathbf{w}; \mathbf{x})  \leq  \epsilon.
\end{split}
\end{equation*}

\textbf{Case 2.} If
\begin{equation*}
{\rm W}_c(D_{X'}^{\delta,\epsilon},D_{X_{\rm A}}) > \rho,
\end{equation*}
then we consider a special distribution $D_{X'}'$, which is defined as follows: for some $u\in [0,1]$,
\begin{equation*}
D_{X'}' = (1-u) D_{X'}^{\delta,\epsilon} + u D_{X_{\rm A}}.
\end{equation*}
It is clear that
\begin{equation*}
{\rm W}_c(D_{X'}', D_{X_{\rm A}}) \leq (1-u) {\rm W}_c(D_{X'}^{\delta,\epsilon},D_{X_{\rm A}})\leq (1-u) (\rho +\delta).
\end{equation*}
hence, if we set $u = \delta/(\rho+\delta)$, 
\begin{equation*}
{\rm W}_c(D_{X'}', D_{X_{\rm A}}) \leq \rho.
\end{equation*}
Because $c(\cdot,\cdot)$ is a metric, Kantorovich–Rubinstein duality \cite{Villani2003TopicsIO} implies that
\begin{equation*}
\begin{split}
&{\rm W}_c(D_{X'}', D_{X'}^{\delta,\epsilon}) \\= &\sup_{\|f\|_{\rm Lip}\leq 1} \int_{\mathcal{X}} f(\mathbf{x}) {\rm d}D_{X'}'(\mathbf{x}) - \int_{\mathcal{X}} f(\mathbf{x}) {\rm d}D_{X'}^{\delta,\epsilon}(\mathbf{x}) 
\\ = & u \sup_{\|f\|_{\rm Lip}\leq 1} \int_{\mathcal{X}} f(\mathbf{x}) {\rm d}D_{X_{\rm A}}(\mathbf{x}) - \int_{\mathcal{X}} f(\mathbf{x}) {\rm d}D_{X'}^{\delta,\epsilon}(\mathbf{x}) 
\\
=&  u {\rm W}_c(D_{X_{\rm A}},D_{X'}^{\delta,\epsilon})
\\
= & \delta
\end{split}
\end{equation*}
By Kantorovich–Rubinstein duality \citep{Villani2003TopicsIO}, we also obtain that
\begin{equation*}
\begin{split}
&\sup_{{\rm W}_{c}(D_{X'},D_{X_{\rm A}}) \leq \rho+\delta} \mathbb{E}_{\mathbf{x}\sim D_{X'}} \ell_{\rm OE}(\mathbf{f}_\mathbf{w}; \mathbf{x}) - \sup_{{\rm W}_{c}(D_{X'},D_{X_{\rm A}}) \leq \rho} \mathbb{E}_{\mathbf{x}\sim D_{X'}} \ell_{\rm OE}(\mathbf{f}_\mathbf{w}; \mathbf{x}) \\ \leq &
 \mathbb{E}_{\mathbf{x}\sim D_{X'}^{\delta,\epsilon}} \ell_{\rm OE}(\mathbf{f}_\mathbf{w}; \mathbf{x}) - \mathbb{E}_{\mathbf{x}\sim D_{X'}'} \ell_{\rm OE}(\mathbf{f}_\mathbf{w}; \mathbf{x})+\epsilon
 \leq L_c \delta +\epsilon,
\end{split}
\end{equation*}
which implies that
\begin{equation*}
\begin{split}
\sup_{{\rm W}_{c}(D_{X'},D_{X_{\rm A}}) \leq \rho+\delta} \mathbb{E}_{\mathbf{x}\sim D_{X'}} \ell_{\rm OE}(\mathbf{f}_\mathbf{w}; \mathbf{x}) - \sup_{{\rm W}_{c}(D_{X'},D_{X_{\rm A}}) \leq \rho} \mathbb{E}_{\mathbf{x}\sim D_{X'}} \ell_{\rm OE}(\mathbf{f}_\mathbf{w}; \mathbf{x})
 \leq  L_c \delta.
\end{split}
\end{equation*}
By Cases 1 and 2, we prove this lemma.
\end{proof}

\begin{lemma}\label{L5}
Given the same assumptions in Theorem \ref{T1}, if
\begin{itemize}
\item{$\ell_{\rm OE}(\cdot; \mathbf{x})$ is  $L_{\rm OE}$-Lipschitz w.r.t. norm $\|\cdot\|$ for all $\mathbf{x}\in \mathcal{X}$;}
\item{ the parameter space $\mathcal{W}\subset \mathbb{R}^{d'}$ satisfies that 
\begin{equation*}
{\rm diam}(\mathcal{W}) = \sup_{\mathbf{w},\mathbf{w}'\in \mathcal{W}} \|\mathbf{w}-\mathbf{w}'\|<+\infty;
\end{equation*}}
\item{ for each $\mathbf{w} \in \mathcal{W}$ and any $\mathbf{x}, \mathbf{x}'\in \mathcal{X}$,
\begin{equation*}
|\ell_{\rm OE}(\mathbf{f}_\mathbf{w}; \mathbf{x})-\ell_{\rm OE}(\mathbf{f}_\mathbf{w}; \mathbf{x}')|\leq L_{c}c(\mathbf{x},\mathbf{x}'),
\end{equation*}}
\end{itemize}
then  with the probability at least $1-e^{-t}>0$, for any $\rho > 0$ and $\mathbf{w} \in \mathcal{W}$, 
\begin{equation}\label{Eq:2.4.3}
\begin{split}
&\sup_{{\rm W}_{c}(D_{X'},D_{X_{\rm A}}) \leq \rho} \mathbb{E}_{\mathbf{x}\sim D_{X'}} \ell_{\rm OE}(\mathbf{f}_\mathbf{w}; \mathbf{x}) \\ \leq &\sup_{{\rm W}_{c}(D_{X'},\widehat{D}_{X_{\rm A}}) \leq \rho} \mathbb{E}_{\mathbf{x}\sim D_{X'}} \ell_{\rm OE}(\mathbf{f}_\mathbf{w}; \mathbf{x})+b_2 M_{\ell_{\rm OE}} \sqrt{\frac{t}{m}} + b_1 \min \{L_c ,  \frac{M_{\ell_{\rm OE}} }{\rho} \} \sqrt{\frac{{\rm diam}(\mathcal{W}) L_{\rm OE}d'}{m}},
\end{split}
\end{equation}
where $b_1$ and $b_2$ are uniform constants. 
\end{lemma}

\begin{proof}
By Lemma \ref{L2} and the similar proving process in Lemma \ref{L4}, we obtain that 
 with the probability at least $1-e^{-t}>0$, for any $\gamma\geq 0$, $\rho \geq 0$ and $\mathbf{w} \in \mathcal{W}$, we have
\begin{equation}\label{Eq:2.4.1}
\begin{split}
&\sup_{{\rm W}_{c}(D_{X'},D_{X_{\rm A}}) \leq \rho} \mathbb{E}_{\mathbf{x}\sim D_{X'}} \ell_{\rm OE}(\mathbf{f}_\mathbf{w}; \mathbf{x}) \\ \leq & \gamma \rho + \frac{1}{m} \sum_{i=1}^m \phi_{\gamma}(\mathbf{w};\mathbf{x}_{\rm A}^i)
 +b_2 M_{\ell_{\rm OE}} \sqrt{\frac{t}{m}} +  \gamma b_1 \sqrt{\frac{{\rm diam}(\mathcal{W}) L_{\rm OE}d'}{m}},
\end{split}
\end{equation}
where $b_1$ and $b_2$ are uniform constants. 

Let
\begin{equation*}
\rho_m = \rho+ b_1 \sqrt{\frac{{\rm diam}(\mathcal{W}) L_{\rm OE}d'}{m}},
\end{equation*}
and
\begin{equation*}
\begin{split}
\Delta_m &= \sup_{{\rm W}_{c}(D_{X'},\widehat{D}_{X_{\rm A}}) \leq \rho_m} \mathbb{E}_{\mathbf{x}\sim D_{X'}} \ell_{\rm OE}(\mathbf{f}_\mathbf{w}; \mathbf{x}) - \sup_{{\rm W}_{c}(D_{X'},\widehat{D}_{X_{\rm A}}) \leq \rho} \mathbb{E}_{\mathbf{x}\sim D_{X'}} \ell_{\rm OE}(\mathbf{f}_\mathbf{w}; \mathbf{x}).
\end{split}
\end{equation*}

Note that
\begin{equation*}
\begin{split}
\inf_{\gamma \geq 0} \gamma \rho_m +\frac{1}{m} \sum_{i=1}^m \phi_{\gamma}(\mathbf{w};\mathbf{x}_{\rm A}^i)
=\sup_{{\rm W}_{c}(D_{X'},\widehat{D}_{X_{\rm A}}) \leq \rho_m} \mathbb{E}_{\mathbf{x}\sim D_{X'}} \ell_{\rm OE}(\mathbf{f}_\mathbf{w}; \mathbf{x}),
\end{split}
\end{equation*}
and by Lemma ~ \ref{L5.0},
\begin{equation*}
\Delta_m \leq L_c(\rho_m-\rho) = b_1 L_c \sqrt{\frac{{\rm diam}(\mathcal{W}) L_{\rm OE}d'}{m}},
\end{equation*}
hence, by Eq. \eqref{Eq:2.4.1}, we know that with the probability at least $1-e^{-t}>0$, for any $\rho \geq 0$ and $\mathbf{w} \in \mathcal{W}$, 
\begin{equation}\label{Eq:2.4.2}
\begin{split}
&\sup_{{\rm W}_{c}(D_{X'},D_{X_{\rm A}}) \leq \rho} \mathbb{E}_{\mathbf{x}\sim D_{X'}} \ell_{\rm OE}(\mathbf{f}_\mathbf{w}; \mathbf{x}) \\ \leq &\sup_{{\rm W}_{c}(D_{X'},\widehat{D}_{X_{\rm A}}) \leq \rho} \mathbb{E}_{\mathbf{x}\sim D_{X'}} \ell_{\rm OE}(\mathbf{f}_\mathbf{w}; \mathbf{x}) + b_1 L_c \sqrt{\frac{{\rm diam}(\mathcal{W}) L_{\rm OE}d'}{m}}+b_2 M_{\ell_{\rm OE}} \sqrt{\frac{t}{m}},
\end{split}
\end{equation}
where $b_1$ and $b_2$ are uniform constants. 

Combining Lemma \ref{L2} with Eq. \eqref{Eq:2.4.2}, we know that with the probability at least $1-e^{-t}>0$, for any $\rho > 0$ and $\mathbf{w} \in \mathcal{W}$, 
\begin{equation*}
\begin{split}
&\sup_{{\rm W}_{c}(D_{X'},D_{X_{\rm A}}) \leq \rho} \mathbb{E}_{\mathbf{x}\sim D_{X'}} \ell_{\rm OE}(\mathbf{f}_\mathbf{w}; \mathbf{x}) \\ \leq &\sup_{{\rm W}_{c}(D_{X'},\widehat{D}_{X_{\rm A}}) \leq \rho} \mathbb{E}_{\mathbf{x}\sim D_{X'}} \ell_{\rm OE}(\mathbf{f}_\mathbf{w}; \mathbf{x})+b_2 M_{\ell_{\rm OE}} \sqrt{\frac{t}{m}} + b_1 \min \{L_c ,  \frac{M_{\ell_{\rm OE}} }{\rho} \} \sqrt{\frac{{\rm diam}(\mathcal{W}) L_{\rm OE}d'}{m}},
\end{split}
\end{equation*}
where $b_1$ and $b_2$ are uniform constants. 
\end{proof}

\begin{lemma}\label{L7}
Given the same assumptions in Lemma \ref{L3}, for a fixed $\mathbf{w}_0\in \mathcal{W}$, then with the probability at least $1-e^{-t}>0$, for any $\rho \geq 0$
\begin{equation*}
\begin{split}
\sup_{{\rm W}_{c}(D_{X'},\widehat{D}_{X_{\rm A}}) \leq \rho} \mathbb{E}_{\mathbf{x}\sim D_{X'}} \ell_{\rm OE}(\mathbf{f}_{\mathbf{w}_0}; \mathbf{x}) \leq 
\sup_{{\rm W}_{c}(D_{X'},D_{X_{\rm A}}) \leq \rho} \mathbb{E}_{\mathbf{x}\sim D_{X'}} \ell_{\rm OE}(\mathbf{f}_{\mathbf{w}_0}; \mathbf{x})+2M_{\ell_{\rm OE}} \sqrt{\frac{2t}{m}}.
\end{split}
\end{equation*}
\begin{proof}[Proof of Lemma \ref{L7}] By \citet{SinhaND18}, it is clear that 
\begin{equation*}
\begin{split}
\sup_{{\rm W}_{c}(D_{X'},D_{X_{\rm A}}) \leq \rho} \mathbb{E}_{\mathbf{x}\sim D_{X'}} \ell_{\rm OE}(\mathbf{f}_{\mathbf{w}_0}; \mathbf{x}) = \inf_{\gamma \geq 0} [\gamma \rho + \mathbb{E}_{\mathbf{x}\sim D_{X_{\rm A}}} \phi_{\gamma}(\mathbf{w}_0;\mathbf{x})]
\end{split}
\end{equation*}
Therefore, for each $\epsilon>0$, there exists a constant $\gamma_{\epsilon}\geq 0$ such that
\begin{equation*}
\begin{split}
   \gamma_{\epsilon} \rho + \mathbb{E}_{\mathbf{x}\sim D_{X_{\rm A}}} \phi_{\gamma_{\epsilon}}(\mathbf{w}_0;\mathbf{x}) \leq \sup_{{\rm W}_{c}(D_{X'},D_{X_{\rm A}}) \leq \rho} \mathbb{E}_{\mathbf{x}\sim D_{X'}} \ell_{\rm OE}(\mathbf{f}_{\mathbf{w}_0}; \mathbf{x})+\epsilon.
   \end{split}
\end{equation*}
Combining the above inequality, Lemma \ref{L1} and McDiarmid's Inequality, then with the probability at least
\begin{equation*}
1- \exp{(\frac{-\epsilon_0^2m}{2M^2_{\ell_{\rm OE}}})}>0,
\end{equation*}
we have
\begin{equation*}
\begin{split}
 &\mathbb{E}_{\mathbf{x}\sim \widehat{D}_{X_{\rm A}}} \phi_{\gamma_{\epsilon}}(\mathbf{w}_0;\mathbf{x})
 \leq
 \mathbb{E}_{\mathbf{x}\sim D_{X_{\rm A}}} \phi_{\gamma_{\epsilon}}(\mathbf{w}_0;\mathbf{x})+ \epsilon_0.
 \end{split}
\end{equation*}
If we set $t = {\epsilon_0^2m}/{2M^2_{\ell_{\rm OE}}}$, then
\begin{equation*}
\epsilon_0 = M_{\ell_{\rm OE}} \sqrt{\frac{2t}{m}}.
\end{equation*}
Hence, with the probability at least $1-e^{-t}>0$,
we have
\begin{equation*}
\begin{split}
 &\gamma_{\epsilon}\rho + \mathbb{E}_{\mathbf{x}\sim \widehat{D}_{X_{\rm A}}} \phi_{\gamma_{\epsilon}}(\mathbf{w}_0;\mathbf{x})
 \leq
\sup_{{\rm W}_{c}(D_{X'},D_{X_{\rm A}}) \leq \rho} \mathbb{E}_{\mathbf{x}\sim D_{X'}} \ell_{\rm OE}(\mathbf{f}_{\mathbf{w}_0}; \mathbf{x})+\epsilon+ M_{\ell_{\rm OE}} \sqrt{\frac{2t}{m}},
 \end{split}
\end{equation*}
which implies that with the probability at least $1-e^{-t}>0$,
\begin{equation*}
\begin{split}
&\sup_{{\rm W}_{c}(D_{X'},\widehat{D}_{X_{\rm A}}) \leq \rho} \mathbb{E}_{\mathbf{x}\sim D_{X'}} \ell_{\rm OE}(\mathbf{f}_{\mathbf{w}_0}; \mathbf{x})  \leq 
\sup_{{\rm W}_{c}(D_{X'},D_{X_{\rm A}}) \leq \rho} \mathbb{E}_{\mathbf{x}\sim D_{X'}} \ell_{\rm OE}(\mathbf{f}_{\mathbf{w}_0}; \mathbf{x})+\epsilon+ M_{\ell_{\rm OE}} \sqrt{\frac{2t}{m}},
\end{split}
\end{equation*}
because
\begin{equation*}
\begin{split}
&\gamma_{\epsilon}\rho + \mathbb{E}_{\mathbf{x}\sim \widehat{D}_{X_{\rm A}}} \phi_{\gamma_{\epsilon}}(\mathbf{w}_0;\mathbf{x}) \geq  \sup_{{\rm W}_{c}(D_{X'},\widehat{D}_{X_{\rm A}}) \leq \rho} \mathbb{E}_{\mathbf{x}\sim D_{X'}} \ell_{\rm OE}(\mathbf{f}_{\mathbf{w}_0}; \mathbf{x}).
\end{split}
\end{equation*}
By setting $\epsilon = M_{\ell_{\rm OE}} \sqrt{{2t}/{m}}$, we complete this proof.
\end{proof}
\end{lemma}

\begin{lemma}\label{L8} If $0\leq \ell(\mathbf{f}_\mathbf{w};\mathbf{x},y)\leq M_{\ell}$, then with the probability at least $1-e^{-t}>0$, we have that for any $\mathbf{w}\in \mathcal{W}$,
\begin{equation*}
\begin{split}
 &\mathbb{E}_{(\mathbf{x},y) \sim D_{X_{\rm I}Y_{\rm I}}}  \ell(\mathbf{f}_\mathbf{w}; \mathbf{x}, y)-\frac{1}{n}\sum_{i=1}^n \ell(\mathbf{f}_\mathbf{w}; \mathbf{x}_{\rm I}^i, y_{\rm I}^i) \leq \frac{b_0 M_{\ell}}{\sqrt{n}}  \int_{0}^{1} \sqrt{\log\mathcal{N}(\mathcal{F},M_{\ell}\epsilon,L^{\infty})}  {\rm d}\epsilon+ M_{\ell}\sqrt{\frac{2t}{n}},
 \end{split}
\end{equation*}
where $b_0$ is a uniform constant.
\end{lemma}
\begin{proof}[Proof of Lemma \ref{L8}]
Let 
\begin{equation*}
X_{\ell(\mathbf{f}_\mathbf{w};\cdot)} = \mathbb{E}_{(\mathbf{x},y) \sim D_{X_{\rm I}Y_{\rm I}}}  \ell(\mathbf{f}_\mathbf{w}; \mathbf{x}, y)-\frac{1}{n}\sum_{i=1}^n \ell(\mathbf{f}_\mathbf{w}; \mathbf{x}_{\rm I}^i, y_{\rm I}^i).
\end{equation*}
Then, it is clear that
\begin{equation*}
\mathbb{E}_{S\sim D_{X_{\rm I}Y_{\rm I}}^n} X_{\ell(\mathbf{f}_\mathbf{w};\cdot)} = 0.
\end{equation*}
By Proposition 2.6.1 and Lemma 2.6.8 in \citet{Vershynin2018HighDimensionalP},
\begin{equation*}
 \|X_{\ell(\mathbf{f}_\mathbf{w};\cdot)}-X_{\ell(\mathbf{f}_{\mathbf{w}'};\cdot)}\|_{\Phi_2}\leq \frac{c_0}{\sqrt{n}}\|\ell(\mathbf{f}_\mathbf{w};\cdot) -\ell(\mathbf{f}_{\mathbf{w}'};\cdot) \|_{L^{\infty}},
\end{equation*}
where $\|\cdot\|_{\Phi_2}$ is the sub-gaussian norm and $c_0$ is a uniform constant.
Therefore, by Dudley’s entropy integral \citep{Vershynin2018HighDimensionalP}, we have
\begin{equation*}
\mathbb{E}_{S\sim D_{X_{\rm I}Y_{\rm I}}^n} \sup_{\mathbf{w} \in \mathcal{W}} X_{\ell(\mathbf{f}_\mathbf{w};\cdot)} \leq  \frac{b_0}{\sqrt{n}}\int_{0}^{+\infty} \sqrt{\log \mathcal{N}(\mathcal{F},\epsilon,L^{\infty}) }{\rm d} \epsilon,
\end{equation*}
where $b_0$ is a uniform constant and 
\begin{equation*}
\mathcal{F}=\{\ell(\mathbf{f}_\mathbf{w};\mathbf{x},y) : \mathbf{w} \in \mathcal{W}\}.
\end{equation*}
Note that
\begin{equation*}
\begin{split}
\mathbb{E}_{S\sim D_{X_{\rm I}Y_{\rm I}}^n} \sup_{\mathbf{w} \in \mathcal{W}} X_{\ell(\mathbf{f}_\mathbf{w};\cdot)}& \leq \frac{b_0}{\sqrt{n}} \int_{0}^{+\infty} \sqrt{\log \mathcal{N}(\mathcal{F},\epsilon,L^{\infty})} {\rm d} \epsilon
\\ & = \frac{b_0}{\sqrt{n}} \int_{0}^{M_{\ell}} \sqrt{\log \mathcal{N}(\mathcal{F},\epsilon,L^{\infty})} {\rm d} \epsilon\\ & = \frac{b_0}{\sqrt{n}} M_{\ell}  \int_{0}^{1} \sqrt{\log\mathcal{N}(\mathcal{F},M_{\ell}\epsilon,L^{\infty})} {\rm d} \epsilon.
\end{split}
\end{equation*}
Then, similar to the proof of Lemma \ref{L7}, we use McDiarmid's Inequality, then with the probability at least $1-e^{-t}>0$, for any $\mathbf{w} \in \mathcal{W}$, 
\begin{equation*}
X_{\ell(\mathbf{f}_\mathbf{w};\cdot)} \leq \frac{b_0}{\sqrt{n}} M_{\ell}  \int_{0}^{1} \sqrt{\log\mathcal{N}(\mathcal{F},M_{\ell}\epsilon,L^{\infty})}  {\rm d}\epsilon+ M_{\ell}\sqrt{\frac{2t}{n}}.
\end{equation*}
\end{proof}

\begin{lemma}\label{L9} If $0\leq \ell(\mathbf{f}_\mathbf{w};\mathbf{x},y)\leq M_{\ell}$, then for a fixed $\mathbf{w}_0 \in \mathcal{W}$, with the probability at least $1-e^{-t}>0$, 
\begin{equation*}
\begin{split}
 &\frac{1}{n}\sum_{i=1}^n \ell(\mathbf{f}_{\mathbf{w}_0}; \mathbf{x}_{\rm I}^i, y_{\rm I}^i)-\mathbb{E}_{(\mathbf{x},y) \sim D_{X_{\rm I}Y_{\rm I}}}  \ell(\mathbf{f}_{\mathbf{w}_0}; \mathbf{x}, y) \leq   M_{\ell}\sqrt{\frac{2t}{n}}.
 \end{split}
\end{equation*}
\end{lemma}
\begin{proof}[Proof of Lemma \ref{L9}]
Similar to the proof of Lemma \ref{L7}, McDiarmid's Inequality implies this result.
\end{proof}

\begin{lemma}\label{L10} If 
\begin{itemize}
    \item $0\leq \ell(\mathbf{f}_{\mathbf{w}};\mathbf{x},y)\leq M_{\ell}$,
    \item $\ell(\mathbf{f}_{\mathbf{w}}; \mathbf{x},y)$ is  $L$-Lipschitz w.r.t. norm $\|\cdot\|$, i.e., for any $(\mathbf{x},y)\in \mathcal{X}\times \mathcal{Y}$, and ${\mathbf{w}},{\mathbf{w}}'\in \mathcal{W}$,
\begin{equation*}
|\ell(\mathbf{f}_{\mathbf{w}}; \mathbf{x},y)-\ell(\mathbf{f}_{\mathbf{w}'}; \mathbf{x},y)|\leq L\|{\mathbf{w}}-{\mathbf{w}}'\|,
\end{equation*}
    \item  the parameter space $\mathcal{W}\subset \mathbb{R}^{d'}$ satisfies that 
\begin{equation*}
{\rm diam}(\mathcal{W}) = \sup_{{\mathbf{w}},{\mathbf{w}}'\in \mathcal{W}} \|{\mathbf{w}}-{\mathbf{w}}'\|<+\infty,
\end{equation*}
\end{itemize}
then with the probability at least $1-e^{-t}>0$, we have that for any ${\mathbf{w}}\in \mathcal{W}$,
\begin{equation*}
\begin{split}
 \mathbb{E}_{(\mathbf{x},y) \sim D_{X_{\rm I}Y_{\rm I}}}  \ell({\mathbf{w}}; \mathbf{x}, y)-\frac{1}{n}\sum_{i=1}^n \ell({\mathbf{w}}; \mathbf{x}_{\rm I}^i, y_{\rm I}^i) \leq {b_0}\sqrt{\frac{M_{\ell}{\rm diam}(\mathcal{W}) Ld'}{n}} + M_{\ell}\sqrt{\frac{2t}{n}},
 \end{split}
\end{equation*}
where $b_0$ is a uniform constant.
\end{lemma}
\begin{proof}[Proof of Lemma \ref{L10}]
The proof is similar to Corollary 1 in \citet{SinhaND18} and Lemma \ref{L4}. Note that 
\begin{equation*}
\mathcal{F} = \{\ell(\mathbf{f}_{\mathbf{w}};\mathbf{x},y): {\mathbf{w}} \in \mathcal{W}\},
\end{equation*}
and $\ell(\mathbf{f}_{\mathbf{w}}; \mathbf{x},y)$ is $L$-Lipschitz w.r.t. norm $\|\cdot\|$, therefore,
\begin{equation*}
\begin{split}
\mathcal{N}(\mathcal{F},M_{\ell}\epsilon,L^{\infty}) \leq  & \mathcal{N}(\mathcal{W},M_{\ell}\epsilon/L,\|\cdot\|)
\leq (1+ \frac{{\rm diam}(\mathcal{W}) L}{M_{\ell}\epsilon})^{d'},
\end{split}
\end{equation*}
which implies that
\begin{equation*}
\begin{split}
\int_{0}^{1} \sqrt{\log(\mathcal{N}(\mathcal{F},M_{\ell}\epsilon,L^{\infty})} {\rm d} \epsilon
\leq & \sqrt{d'} \int_{0}^1 \sqrt{\log (1+ \frac{{\rm diam}(\mathcal{W}) L}{M_{\ell}\epsilon})}  {\rm d} \epsilon
\\ 
\leq & \sqrt{d'} \int_{0}^1 \sqrt{\frac{{\rm diam}(\mathcal{W}) L}{M_{\ell}\epsilon}}  {\rm d} \epsilon
= 2\sqrt{\frac{{\rm diam}(\mathcal{W}) Ld'}{M_{\ell}}}.
\end{split}
\end{equation*}
By Lemma \ref{L8}, we obtain this result.
\end{proof}

\newpage

\section{Further Discussions}

\label{app:alg}

As discussed in Section~\ref{sec: alg}, we realize the dual optimization objective following Eq.~\eqref{eq: dro_oe2}, searching for the worst OOD data in a finite-dimensional space to ease the computation. Furthermore, directly searching in the input space is typically hard for optimization~\citep{MadryMSTV18,wang2021probabilistic}, where the results can easily stuck at sub-optimal solutions and the computation is time-consuming. Therefore, we suggest perturbing the worst OOD data in the embedding space. Denote the embedding features of an input by $\mathbf{e}{(\mathbf{x})}$, we consider the bi-level optimization problem:
\begin{equation*}
\begin{aligned}
\inf_{\gamma \geq 0}\big \{\gamma \rho & + \frac{1}{m}\sum_{i=1}^m \phi_{\gamma}(\mathbf{w}; \mathbf{e}{(\mathbf{x}_{\rm A}^i)}) \big \}, \\
    \text{where}~\phi_{\gamma}(\mathbf{w}; \mathbf{e}{(\mathbf{x}_{\rm A}^i)}) = \sup_{\mathbf{p}^i\in\mathcal{E}} &  \left\{ 
 \ell_{\rm OE}(\mathbf{h}(\mathbf{e}(\mathbf{x}_{\rm A}^i)+\mathbf{p}^i); \mathbf{e}(\mathbf{x}_{\rm A}^i))-\gamma  \|\mathbf{p}^i\|_1\right\}.
\end{aligned}
\end{equation*}
Such an bi-level problem can be solved by alternative optimization~\citep{huang2023robust,liu2021investigating}: we first find the proper value of $\mathbf{p}^i$ that approaches to the true value of $\phi_{\gamma}(\mathbf{w}; \mathbf{e}{(\mathbf{x}_{\rm A}^i)})$, and then we update the value of $\gamma$ that leads to the smallest value of $\gamma \rho + \frac{1}{m}\sum_{i=1}^m \phi_{\gamma}(\mathbf{w}; \mathbf{e}{(\mathbf{x}_{\rm A}^i)})$. The gradient descent/ascent can be adopted for optimization. Specifically, for the perturbation $\mathbf{p}^i$, each optimization step is \begin{equation*}
\begin{aligned}
    \phi_{\gamma}(\mathbf{w}; \mathbf{e}{(\mathbf{x}_{\rm A}^i)}) &\leftarrow\ell_\text{OE}\left(\textbf{h}(\textbf{e}(\mathbf{x}_{\rm A}^i)+\mathbf{p}^i);\textbf{e}(\mathbf{x}_{\rm A}^i)\right)-\gamma \left\| \mathbf{p}^i \right\|_1, \\
    \mathbf{p}^i& \leftarrow \mathbf{p}^i + \texttt{ps} \nabla_{\mathbf{p}^i} \phi_{\gamma}(\mathbf{w}; \mathbf{e}{(\mathbf{x}_{\rm A}^i)}),
\end{aligned}    
\end{equation*}
where $\texttt{ps}$ is the learning rate. For $\gamma$, the optimization step follows:
\begin{equation*}
\gamma\leftarrow\gamma-\beta\{\rho-\frac{1}{m}\sum_{i=1}^{m} \left\| \mathbf{p}^i \right\|_1 \},
\end{equation*}
with $\beta$ the learning rate. Furthermore, to avoid the extreme value and/or the negative value of $\gamma$, we should conduct value clipping for $\gamma$, which is given by
\begin{equation*}
    \gamma\leftarrow\min(\max(\gamma, \gamma_\text{max}),0),
\end{equation*}
where we constrain the minimum of $\gamma$ to be $0$ and the maximum of $\gamma$ to be $\gamma_\text{max}$.

\subsection{Understanding Theoretical Results} \label{sec: understand theorem}

Theorem~\ref{T1} justifies that when the model complexity and the sample size are large enough, the empirical solution given by our DAL risk will converge to its optimal value, i.e., $\min_\mathbf{w} R_D(\mathbf{w};\rho)$. Therefore, the difference between the expected and the empirical error is bounded w.r.t. the DAL risk. Theorem~\ref{T3} goes one step further, studying the detection risk w.r.t. (unseen) real OOD data. It states that the open-world performance of our DAL depends on both the approximate risk and the estimation error. The former models the best performance (i.e., Bayes optimal) that our DAL can achieve, and the latter depends on the OOD distribution gap, the radius, and the excess error introduced in Theorem~\ref{T1}. In summary, Theorem~\ref{T1} considers the convergence for DAL itself, while Theorem~\ref{T3} justifies that DAL can mitigate the OOD distribution discrepancy in the open world.

\subsection{Comparison with DOE}

A parallel work, named DOE~\citep{wang2023out}, also focuses on mitigating the OOD distribution discrepancy issue. Overall, they state that model perturbation can lead to input transformation, and thus learning from the perturbed model can make the predictor learn from diverse distributions with respect to auxiliary OOD cases. Moreover, to make the transformed data benefit the model most, DOE searches for the associated perturbation that leads to the worst OOD regret.

Similar to DOE, we also learn from the worst OOD cases to mitigate the distribution discrepancy, but DAL further enjoys the following two strengths. 1) From the theoretical perspective, our clear definition of the candidate OOD distribution space, i.e., the Wasserstein ball, allows us to investigate the impact of DAL for open-world OOD detection (cf., Theorem~\ref{T3}). In contrast, DOE only constrains the magnitude of the perturbation strength, making it limited to proving convergence w.r.t. their proposed learning objective (cf., Theorem~2 in~\cite{wang2023out}). 2) From the algorithmic perspective, we directly search for the worst OOD data in the embedding space, more effective than DAL, which requires searching the model perturbation for the whole model. As a result, our theoretically-driven framework, i,e, DAL, yields superior performance over DOE in Table~\ref{tab: full} while requiring less computation cost (DAL take only half the training time compared with DOE per epoch).

\subsection{Discussing about Limitations}

In theory, our main drawback lies in the trade-off between estimation and approximation errors (cf., Theorem~\ref{T3}), where we may not get a very tight bound for the real OOD risk. In algorithm, the worst OOD data are constrained in the ball around the auxiliary OOD data (cf., Algorithm~\ref{alg: DOE}), of which the coverage may not include real OOD data. Moreover, we conduct distribution augmentation in the embedding space, where our Theorem~\ref{T3} can only be applied. Other data generation approaches, which can lead to more complex forms of distribution augmentation in the input space, are also of interest. Our future studies will focus on advanced learning schemes that address the above issues, e.g., modeling the data generation process through the causality perspective~\citep{zhang2021adversarial}.

\section{Further Evaluations}

We provide more information about evaluation setups and conduct additional experiments on DAL.

\subsection{Hardware Configurations} 

All experiments are realized by Pytorch $1.81$ with CUDA $11.1$, using machines equipped with GeForce RTX $3090$ GPUs and AMD Threadripper $3960$X Processors.

\subsection{ID Accuracy}

We report the ID accuracy for those methods that require model training on the CIFAR benchmarks, of which the results are summarized in Table~\ref{tab: id accu}. We also list the results for the model conventionally trained on ID data with empirical risk minimization (ERM). Overall, most of the considered methods can preserve relatively high ID accuracy. Moreover, those methods that regularize predictors by auxiliary OOD data, such as OE and DAL, can even show further improvements. It indicates that learning with auxiliary data can achieve high detection performance and maintain good ID accuracy. 

\begin{table}[t]
\caption{ID accuracy on the CIFAR benchmarks for those methods that require model training. } \label{tab: id accu}
\centering
\scriptsize
\resizebox{.82\linewidth}{!}{
\begin{tabular}{c|cccccccc|c}
\toprule[1.5pt]
Method    & ERM    & CSI & VOS   & OE    & Energy-OE & ATOM  & DOE & POEM & DAL \\
\midrule[1pt]
CIFAR-10  &  94.28 & 94.33 & 94.58      & 95.22 &  94.84         & 95.12 & 94.28     & 93.32     &  95.01    \\
CIFAR-100 &  73.98 & 74.30 & 75.50      & 75.90 &  71.61         & 74.04 & 74.51    & 74.85      & 76.13    \\
\bottomrule[1.5pt]  
\end{tabular}}
\end{table}

\subsection{Other Scoring Functions}
\label{app: other scoring}

We further claim that many advanced scoring strategies other than MSP can also benefit from DAL. In Table~\ref{tab: diff score}, we compare the OOD detection performance before (w/o train) and after (w/ DAL) DAL training across a set of representative scoring strategies, including MSP, Fee Energy, ASH, Mahalanobis, and KNN. We also compare the results after OE training (w/ OE). As we can see, both OE and DAL can lead to much better results than before training, and DAL can further boost detection performance over OE. It indicates that the benefits of our proposal are not limited to the specific scoring function of MSP. However, Mahalanobis fails (FPR95 more than 95) on CIFAR-100 after OE and DAL training, which may require further exploration.

\begin{table}[t]
\caption{Comparison on the CIFAR benchmarks with different scoring strategies. } \label{tab: diff score}
\centering
\scriptsize
\resizebox{.82\linewidth}{!}{
\begin{tabular}{c|cc|cc|cc|cc|cc}
\toprule[1.5pt]
\multirow{2}{*}{} & \multicolumn{2}{c|}{MSP} & \multicolumn{2}{c|}{Free Energy} & \multicolumn{2}{c|}{ASH} & \multicolumn{2}{c|}{Mahalanobis} & 
\multicolumn{2}{c}{KNN} \\
\cline{2-11}
& FPR & AUROC & FPR & AUROC & FPR & AUROC & FPR & AUROC & FPR & AUROC \\
\midrule[1pt]
\multicolumn{11}{c}{CIFAR-10}  \\
\midrule[0.6pt]
w/o train & 50.15 & 91.02 & 33.21 & 91.01 & 32.98 & 91.85 & 46.64 & 88.59 & 33.38 & 93.76 \\
w/ OE     & 4.67  & 98.88 & 3.40  & 98.98 & 3.35  & 98.99 & 15.80 & 94.32 & 5.50  & 98.71  \\
w/ DAL    & 2.68  & 99.01 & 2.59  & 98.99 & 2.50  & 98.70 & 12.75 & 95.55 & 5.04  & 97.58 \\
\midrule[1pt]
\multicolumn{11}{c}{CIFAR-100} \\
\midrule[0.6pt]
w/o train & 78.61 & 75.95 & 69.84 & 75.20 & 59.31 & 84.46 & 72.37 & 82.70 & 59.31 & 84.46 \\
w/ OE     & 43.14 & 90.27 & 36.98 & 92.66 & 33.82 & 93.36 &  -    & -     & 53.14 & 83.50 \\
w/ DAL    & 29.68 & 93.92 & 29.63 & 93.83 & 29.73 & 94.05 &  -    & -     & 50.46 & 84.75 \\
\bottomrule[1.5pt]
\end{tabular}}
\end{table}

\subsection{Mean and Standard Deviation}

This section validates the experiments on CIFAR benchmarks with five individual trials (random seeds), comparing between our DAL and OE. Besides the individual results, we also summarize the mean performance and standard deviation across the trails for both CIFAR-10 and CIFAR-100. We summarize the experimental results in Tables~\ref{tab: cifar10 mean+std}-\ref{tab: cifar100 mean+std}. As we can see, our DAL can not only lead to improved performance in OOD detection, but our performance is also very stable across different choices of ID datasets and real OOD datasets.

\begin{table*}[t]
\caption{Comparison of DAL and outlier exposure on CIFAR-10 with 5 individual trails. $\downarrow$ (or $\uparrow$) indicates smaller (or larger) values are preferred; and a bold font indicates the best results in the corresponding column.  } \label{tab: cifar10 mean+std}
\centering
\resizebox{.93\linewidth}{!}{
\begin{tabular}{c|cccccccccc|cc}
\toprule[1.5pt]
\multirow{2}{*}{Trails} & \multicolumn{2}{c}{SVHN} & \multicolumn{2}{c}{LSUN} & \multicolumn{2}{c}{iSUN} & \multicolumn{2}{c}{Textures} & \multicolumn{2}{c}{Places365} & \multicolumn{2}{|c}{\textbf{Average}} \\
\cline{2-13}
& FPR95 $\downarrow$ & AUROC $\uparrow$ & FPR95 $\downarrow$ & AUROC $\uparrow$ & FPR95 $\downarrow$ & AUROC $\uparrow$ & FPR95 $\downarrow$ & AUROC $\uparrow$ & FPR95 $\downarrow$ & AUROC $\uparrow$ & FPR95 $\downarrow$ & AUROC $\uparrow$ \\
\midrule[1pt]
\multicolumn{13}{c}{OE} \\
\midrule[0.6pt]
\#1        & 1.50 & 99.23 & 1.10 & 99.33 & 1.70 & 99.18 & 4.00 & 98.64 & 11.30 & 97.09 & 3.92 & 98.69  \\
\#2        & 1.25 & 99.15 & 1.05 & 99.49 & 2.20 & 98.88 & 4.15 & 98.59 & 11.60 & 97.08 & 4.05 & 98.63  \\
\#3        & 1.25 & 99.38 & 1.05 & 99.42 & 1.75 & 99.01 & 4.00 & 98.82 & 11.10 & 97.04 & 3.83 & 98.73  \\
\#4        & 1.70 & 99.13 & 1.05 & 99.52 & 2.10 & 99.12 & 4.20 & 98.55 & 11.65 & 97.08 & 4.14 & 98.68  \\
\#5        & 1.35 & 99.17 & 1.30 & 99.49 & 1.40 & 99.26 & 4.60 & 99.00 & 11.75 & 97.03 & 1.08 & 98.79  \\
\hline
\makecell{mean \\ $\pm$ std} 
           & \makecell{1.41 \\ $\pm$ 0.17}  & \makecell{99.21 \\ $\pm$ 0.10}  & \makecell{ 1.10\\ $\pm$ 0.02} & \makecell{99.45\\ $\pm$ 0.07} & \makecell{1.83 \\ $\pm$ 0.28} & \makecell{99.09 \\ $\pm$ 0.15} & \makecell{4.19 \\ $\pm$ 0.22} & \makecell{98.72 \\ $\pm$ 0.18} & \makecell{11.48 \\ $\pm$ 0.24} & \makecell{97.06 \\ $\pm$ 0.03} & \makecell{3.40 \\ $\pm$ 1.16} & \makecell{98.70 \\ $\pm$ 0.05} \\
\midrule[0.6pt]
\multicolumn{13}{c}{{DAL}} \\
\midrule[0.6pt]
\#1        & 0.80 & 99.84 & 0.40 & 99.59 & 0.95 & 99.29 & 2.65 & 98.85 & 7.75 & 97.37 & 2.51 & 98.86  \\
\#2        & 0.90 & 99.24 & 0.60 & 99.57 & 1.20 & 99.26 & 2.65 & 98.84 & 8.15 & 97.43 & 2.70 & 98.87  \\
\#3        & 0.90 & 99.16 & 0.65 & 99.52 & 1.15 & 99.14 & 2.40 & 98.75 & 8.20 & 97.35 & 2.66 & 98.78  \\
\#4        & 0.80 & 99.37 & 0.55 & 99.63 & 0.95 & 99.34 & 2.85 & 98.89 & 7.95 & 97.39 & 2.62 & 98.93  \\
\#5        & 1.25 & 99.39 & 0.40 & 99.61 & 0.85 & 99.39 & 2.75 & 98.90 & 7.70 & 97.46 & 2.59 & 98.95  \\
\hline
\makecell{mean \\ $\pm$ std} 
           & \makecell{\textbf{0.93} \\ $\pm$ \textbf{0.17}} & \makecell{\textbf{99.40} \\ $\pm$ \textbf{0.24}} & \makecell{\textbf{0.52} \\ $\pm$ \textbf{0.10}} & \makecell{\textbf{99.58} \\ $\pm$ \textbf{0.04}} & \makecell{\textbf{1.02} \\ $\pm$ \textbf{0.13}} & \makecell{\textbf{99.28} \\ $\pm$ \textbf{0.08}} & \makecell{\textbf{2.65} \\ $\pm$ \textbf{0.15}} & \makecell{\textbf{98.84} \\ $\pm$ \textbf{0.05}} & \makecell{\textbf{7.95} \\ $\pm$ \textbf{0.20}} & \makecell{\textbf{97.40} \\ $\pm$ \textbf{0.04}} & \makecell{\textbf{2.61} \\ $\pm$ \textbf{0.06}} & \makecell{\textbf{98.87} \\ $\pm$ \textbf{0.06}}  \\
\bottomrule[1.5pt]   
\end{tabular}}
\end{table*}

\begin{table*}[t]
\caption{Comparison of DAL and outlier exposure on CIFAR-100 with 5 individual trails. $\downarrow$ (or $\uparrow$) indicates smaller (or larger) values are preferred, and a bold font indicates the best results in the corresponding column.  } \label{tab: cifar100 mean+std}
\centering
\resizebox{.93\linewidth}{!}{
\begin{tabular}{c|cccccccccc|cc}
\toprule[1.5pt]
\multirow{2}{*}{Trails} & \multicolumn{2}{c}{SVHN} & \multicolumn{2}{c}{LSUN} & \multicolumn{2}{c}{iSUN} & \multicolumn{2}{c}{Textures} & \multicolumn{2}{c}{Places365} & \multicolumn{2}{|c}{\textbf{Average}} \\
\cline{2-13}
& FPR95 $\downarrow$ & AUROC $\uparrow$ & FPR95 $\downarrow$ & AUROC $\uparrow$ & FPR95 $\downarrow$ & AUROC $\uparrow$ & FPR95 $\downarrow$ & AUROC $\uparrow$ & FPR95 $\downarrow$ & AUROC $\uparrow$ & FPR95 $\downarrow$ & AUROC $\uparrow$ \\
\midrule[1pt]
\multicolumn{13}{c}{OE} \\
\midrule[0.6pt]
\#1        & 44.45 & 91.76 & 15.75 & 97.26 & 45.95 & 88.80 & 47.35 & 89.80 & 54.10 & 87.90 & 41.52 & 91.10  \\
\#2        & 42.75 & 91.93 & 15.85 & 97.22 & 46.85 & 88.91 & 46.75 & 89.78 & 53.05 & 88.04 & 41.05 & 91.18  \\
\#3        & 43.75 & 91.88 & 15.95 & 97.34 & 52.25 & 87.62 & 47.15 & 89.49 & 54.10 & 88.03 & 42.64 & 90.87  \\
\#4        & 41.30 & 92.23 & 16.15 & 97.27 & 46.90 & 88.76 & 47.00 & 89.73 & 54.40 & 87.91 & 41.15 & 91.18  \\
\#5        & 42.55 & 91.92 & 16.20 & 97.22 & 44.70 & 89.66 & 47.35 & 89.47 & 54.35 & 87.82 & 41.03 & 91.22  \\
\hline
\makecell{mean \\ $\pm$ std} 
           & \makecell{42.96 \\ $\pm$ 1.07}  & \makecell{91.94 \\ $\pm$ 0.15} & \makecell{15.97 \\ $\pm$ 0.17} & \makecell{97.26 \\ $\pm$ 0.04} & \makecell{47.33 \\ $\pm$ 2.58} & \makecell{88.75 \\ $\pm$ 0.65} & \makecell{47.12 \\ $\pm$ 0.23} & \makecell{89.65 \\ $\pm$ 0.14} & \makecell{54.00 \\ $\pm$ 0.49} & \makecell{87.94 \\ $\pm$ 0.08} & \makecell{41.47 \\ $\pm$ 0.61} & \makecell{91.10 \\ $\pm$ 0.13}  \\
\midrule[0.6pt]
\multicolumn{13}{c}{{DAL}} \\
\midrule[0.6pt]
\#1        & 19.35 & 96.21 & 16.05 & 96.78 & 26.05 & 94.23 & 37.60 & 91.57 & 49.35 & 88.81 & 29.68 & 93.52  \\
\#2        & 22.65 & 95.55 & 16.30 & 96.73 & 26.35 & 94.23 & 36.20 & 91.91 & 48.50 & 88.74 & 30.00 & 93.43  \\
\#3        & 20.15 & 96.15 & 16.20 & 96.91 & 29.85 & 93.55 & 37.85 & 91.60 & 47.90 & 88.95 & 30.39 & 93.43   \\
\#4        & 14.50 & 96.72 & 16.75 & 96.58 & 33.75 & 92.68 & 37.60 & 91.63 & 49.70 & 88.80 & 30.46 & 93.28  \\
\#5        & 22.70 & 95.90 & 15.20 & 96.91 & 27.15 & 94.58 & 37.00 & 91.82 & 49.65 & 88.73 & 30.34 & 93.59  \\
\hline
\makecell{mean \\ $\pm$ std} 
           & \makecell{\textbf{19.87} \\ $\pm$ \textbf{2.99}} & \makecell{\textbf{96.11} \\ $\pm$ \textbf{0.39}} & \makecell{\textbf{16.10} \\ $\pm$ \textbf{0.51}} & \makecell{\textbf{96.78} \\ $\pm$ \textbf{0.12}} & \makecell{\textbf{28.63} \\ $\pm$ \textbf{2.89}} & \makecell{\textbf{93.85} \\ $\pm$ \textbf{0.68}} & \makecell{\textbf{37.25} \\ $\pm$ \textbf{0.60}} & \makecell{\textbf{91.70} \\ $\pm$ \textbf{0.13}} & \makecell{\textbf{49.01} \\ $\pm$ \textbf{0.71}} & \makecell{\textbf{88.80} \\ $\pm$ \textbf{0.08}} & \makecell{\textbf{30.17} \\ $\pm$ \textbf{0.29}} & \makecell{\textbf{93.45} \\ $\pm$ \textbf{0.10}}  \\
\bottomrule[1.5pt]   
\end{tabular}}
\end{table*}

\subsection{Effects of Hyper-parameters}
We further test the impacts of other hyper-parameters on the performance in OOD detection, where we consider $\gamma_\text{max}$, $\beta$, $\texttt{num}\_\texttt{search}$, and $\texttt{ps}$ on CIFAR benchmarks.

\textbf{Impacts of $\gamma$.} The exact values of $\gamma$ are dynamically determined by $\gamma_\text{max}$, $\rho$, $\beta$, and the current model $\textbf{f}_\textbf{w}$. To evaluation the effects of $\gamma$, we conduct experiments on CIFAR benchmarks with different $\gamma_\text{max}$, $\rho$, and $\beta$, of which the results are summarized in Table~\ref{tab: gamma_cifar10}-\ref{tab: gamma_cifar100}. Overall, our method is pretty robust to different choices of hyper-parameters, in that the results for most of the hyper-parameter setups can lead to promising improvement over the original outlier exposure. Specifically, for $\gamma_\text{max}$, most of its different choices can lead to effective OOD detection with the proper choices of $\rho$ and $\beta$, but its values should not be too small (e.g., $\gamma_\text{max}=0.1$). The reason is that if the value of $\gamma$ is too small, the distance between the worst-cases OOD features, i.e., $\mathbf{g}_\mathbf{w}(\mathbf{x})+\mathbf{p}$, and the original OOD features, i.e., $\mathbf{g}_\mathbf{w}(\mathbf{x})$,  can be very long, occupying the regions that should belong to ID data. It will make the model confused between ID and OOD cases and thus lead to unsatisfactory results. A similar conclusion can also be applied for $\beta$: when its value is too large (such as $\beta=5$), values of $\mathbf{g}_\mathbf{w}(\mathbf{x})+\mathbf{p}$ can also be arbitrarily large, making the current model hardly learn to discern ID and OOD patterns. 

\textbf{Impacts of} $\texttt{num}\_\texttt{search}$ \textbf{and} $\texttt{ps}$. We also provides the results on CIFAR benchmarks with different $\texttt{num}\_\texttt{search}$ and $\texttt{ps}$, and the results can be found in Tables~\ref{tab: hyper_search}-\ref{tab: hyper_ps}. As we can see, even with some extreme values, such as $\texttt{num}\_\texttt{search}=500$ and $\texttt{ps}=100$, the resultant models still enjoy the improvements over outlier exposure, indicating that our method is pretty robust to these hyper-parameters. The reason is that our proper selection of $\rho$ will constrain the resultant perturbation to be beneficial, avoiding the worst OOD distribution to not intrude the region that belongs to ID data.

\begin{table}[]
\caption{Detection Performance on CIFAR-$10$ dataset with different choices of $\beta$, $\rho$, and $\gamma_\text{max}$, where we report the FPR95 / AUROC for each individual trail setup.} \label{tab: gamma_cifar10}
\centering
\resizebox{.48\linewidth}{!}{
\begin{tabular}{ccccccc}
\toprule[1.5pt]
\multicolumn{7}{c}{\cellcolor{greyC} $\gamma_\text{max}$=50}               \\ \midrule
\multicolumn{2}{c}{\multirow{2}{*}{}} & \multicolumn{5}{|c}{$\rho$} \\ \cline{3-7} 
\multicolumn{2}{c|}{} & \multicolumn{1}{c|}{1e-2} & \multicolumn{1}{c|}{1e-1} & \multicolumn{1}{c|}{1} & \multicolumn{1}{c|}{10}  & 100 \\ \hline
\multicolumn{1}{c|}{\multirow{8}{*}{$\beta$}} & \multicolumn{1}{c|}{1e-3} & 2.95 / 99.07 & 2.80 / \textbf{99.07} & 2.96 / 99.02 & 2.80 / 98.31 & 91.85 / 64.43  \\ \cline{2-7} 
\multicolumn{1}{c|}{}                         & \multicolumn{1}{c|}{5e-3} & 2.95 / 99.01 & 3.05 / 99.00 & 2.97 / 99.04 & 2.69 / 98.16 & 92.79 / 55.39  \\ \cline{2-7} 
\multicolumn{1}{c|}{}                         & \multicolumn{1}{c|}{1e-2} & 2.79 / 98.95 & \textbf{2.68} / 99.05 & 2.84 / 98.88 & 2.71 / 98.67 & 96.48 / 44.97  \\ \cline{2-7} 
\multicolumn{1}{c|}{}                         & \multicolumn{1}{c|}{5e-2} & 3.08 / 98.98 & 3.03 / 98.98 & 2.85 / 98.98 & 2.88 / 98.79 & 95.58 / 45.79  \\ \cline{2-7} 
\multicolumn{1}{c|}{}                         & \multicolumn{1}{c|}{1e-1} & 2.79 / 98.96 & 2.98 / 98.99 & 2.75 / 99.02 & 10.22 / 96.37 & 88.96 / 72.68  \\ \cline{2-7} 
\multicolumn{1}{c|}{}                         & \multicolumn{1}{c|}{5e-1} & 2.82 / 99.00 & 2.95 / 99.01 & 2.81 / 99.05 & 4.34 / 96.58 & 95.51 / 46.60  \\ \cline{2-7} 
\multicolumn{1}{c|}{}                         & \multicolumn{1}{c|}{1}    & 2.94 / 98.93 & 2.88 / 99.02 & 3.19 / 99.00 & 52.76 / 94.36 & 95.05 / 53.89  \\ \cline{2-7} 
\multicolumn{1}{c|}{}                         & \multicolumn{1}{c|}{5}    & 2.98 / 98.91 & 2.77 / 98.96 & 3.00 / 99.06 & 94.44 / 62.90 & 95.57 / 52.51 \\  \midrule
\multicolumn{7}{c}{\cellcolor{greyC} $\gamma_\text{max}$=5}                                                                                                                                                        \\ \midrule
\multicolumn{2}{c}{\multirow{2}{*}{}} & \multicolumn{5}{|c}{$\rho$} \\ \cline{3-7} 
\multicolumn{2}{c|}{} & \multicolumn{1}{c|}{1e-2} & \multicolumn{1}{c|}{1e-1} & \multicolumn{1}{c|}{1} & \multicolumn{1}{c|}{10}  & 100 \\ \hline
\multicolumn{1}{c|}{\multirow{8}{*}{$\beta$}} & \multicolumn{1}{c|}{1e-3} & 3.05 / 99.04 & 2.84 / 98.98 & 2.81 / 98.99 & 2.79 / 98.20 & 93.77 / 64.09  \\ \cline{2-7} 
\multicolumn{1}{c|}{}                         & \multicolumn{1}{c|}{5e-3} & 2.82 / 99.02 & 2.85 / \textbf{99.20} & 2.97 / 99.00 & 2.76 / 98.37 & 93.50 / 63.44 \\ \cline{2-7} 
\multicolumn{1}{c|}{}                         & \multicolumn{1}{c|}{1e-2} & 2.82 / 98.99 & 2.95 / 98.92 & 2.93 / 99.03 & 2.68 / 98.48 & 94.42 / 52.51  \\ \cline{2-7} 
\multicolumn{1}{c|}{}                         & \multicolumn{1}{c|}{5e-2} & 2.99 / 98.98 & 2.81 / 98.99 &2.87 / 98.91 & 5.21 / 95.95 & 87.32 / 71.25  \\ \cline{2-7} 
\multicolumn{1}{c|}{}                         & \multicolumn{1}{c|}{1e-1} & 2.91 / 98.95 & 2.70 / 99.06 & 2.90 / 99.03 & 3.84 / 96.46 & 95.17 / 57.61  \\ \cline{2-7} 
\multicolumn{1}{c|}{}                         & \multicolumn{1}{c|}{5e-1} & 3.03 / 99.01 & 3.06 / 99.00 & \textbf{2.75} / 99.02 & 88.76 / 69.56 & 94.68 / 41.66  \\ \cline{2-7} 
\multicolumn{1}{c|}{}                         & \multicolumn{1}{c|}{1}    & 2.67 / 99.00 & 2.86 / 99.03 & 2.94 / 99.00 & 97.40 /45.45 &95.03 / 48.58  \\ \cline{2-7} 
\multicolumn{1}{c|}{}                         & \multicolumn{1}{c|}{5}    & 2.87 / 98.93 & 2.82 / 98.97 & 63.22 / 87.13 & 98.97 / 58.50 & 95.49 / 50.51  \\ \midrule
\multicolumn{7}{c}{\cellcolor{greyC} $\gamma_\text{max}$=0.5}                                                                                                                                                        \\\midrule
\multicolumn{2}{c}{\multirow{2}{*}{}} & \multicolumn{5}{|c}{$\rho$} \\ \cline{3-7} 
\multicolumn{2}{c|}{} & \multicolumn{1}{c|}{1e-2} & \multicolumn{1}{c|}{1e-1} & \multicolumn{1}{c|}{1} & \multicolumn{1}{c|}{10}  & 100\\ \hline
\multicolumn{1}{c|}{\multirow{8}{*}{$\beta$}} & \multicolumn{1}{c|}{1e-3} & 2.90 / 98.99 & 2.74 / 98.89 & 2.74 / 98.87 & 2.56 / 98.20 & 94.75 / 65.65  \\ \cline{2-7} 
\multicolumn{1}{c|}{}                         & \multicolumn{1}{c|}{5e-3} & 2.64 / 98.94 & 2.79 / 98.89 & 2.53 / 98.81 & 2.54 / 98.32 & 97.22 / 41.88  \\ \cline{2-7} 
\multicolumn{1}{c|}{}                         & \multicolumn{1}{c|}{1e-2} & 2.65 / 98.93 & 2.74 / 98.94 & 2.78 / 98.85 & 2.70 / 98.50 & 94.86 / 56.57  \\ \cline{2-7} 
\multicolumn{1}{c|}{}                         & \multicolumn{1}{c|}{5e-2} & 2.68 / \textbf{98.98} & 2.73 / 98.95 & 2.89 / 98.80 & 3.50 / 96.78 & 92.13 / 65.63  \\ \cline{2-7} 
\multicolumn{1}{c|}{}                         & \multicolumn{1}{c|}{1e-1} & 2.68 / 98.94 & 2.63 / 98.91 & 2.89 / 98.86 & 11.31 / 96.11 & 91.44 / 70.20 \\ \cline{2-7} 
\multicolumn{1}{c|}{}                         & \multicolumn{1}{c|}{5e-1} & 3.07 / 98.89 & 2.74 / 98.89 & \textbf{2.50} / 98.68 & 18.10 / 95.47 & 93.15 / 59.78 \\ \cline{2-7} 
\multicolumn{1}{c|}{}                         & \multicolumn{1}{c|}{1}    & 2.82 / 98.85 & 2.61 / 98.89 & 2.76 / 98.81 & 12.58 / 95.80  & 94.61 / 61.49 \\ \cline{2-7} 
\multicolumn{1}{c|}{}                         & \multicolumn{1}{c|}{5}    & 2.61 / 98.84 & 2.73 / 98.92 & 2.74 / 98.73 & 83.18 / 90.90 & 94.14 / 60.82 \\  
\bottomrule[1.5pt]
\end{tabular}}
\resizebox{.49\linewidth}{!}{
\begin{tabular}{ccccccc}
\toprule[1.5pt]
\multicolumn{7}{c}{\cellcolor{greyC} $\gamma_\text{max}$=10}                                                                                                                                                        \\ \midrule
\multicolumn{2}{c}{\multirow{2}{*}{}} & \multicolumn{5}{|c}{$\rho$} \\ \cline{3-7} 
\multicolumn{2}{c|}{} & \multicolumn{1}{c|}{1e-2} & \multicolumn{1}{c|}{1e-1} & \multicolumn{1}{c|}{1} & \multicolumn{1}{c|}{10}  & 100 \\ \hline
\multicolumn{1}{c|}{\multirow{8}{*}{$\beta$}} & \multicolumn{1}{c|}{1e-3} & 3.06 / 99.05 & 2.82 / 99.06 & 2.84 / 98.97 & \textbf{2.41} / 97.95 & 94.65 / 46.60  \\ \cline{2-7} 
\multicolumn{1}{c|}{}                         & \multicolumn{1}{c|}{5e-3} & 2.85 / 99.00 & 2.80 / \textbf{99.09} & 2.93 / 99.04 & 2.56 / 98.21 & 94.69 / 57.61 \\ \cline{2-7} 
\multicolumn{1}{c|}{}                         & \multicolumn{1}{c|}{1e-2} & 2.91 / 98.98 & 2.98 / 99.04 & 2.56 / 99.02 & 2.58 / 98.28 & 95.48 / 50.70  \\ \cline{2-7} 
\multicolumn{1}{c|}{}                         & \multicolumn{1}{c|}{5e-2} & 2.94 / 99.01 & 2.91 / 99.06 & 2.71 / 99.03 & 2.81 / 98.53 & 97.80 / 43.86  \\ \cline{2-7} 
\multicolumn{1}{c|}{}                         & \multicolumn{1}{c|}{1e-1} & 3.02 / 99.04 & 2.77 / 99.04 & 2.87 / 99.02 & 14.11 / 95.47 & 96.00 / 55.75  \\ \cline{2-7} 
\multicolumn{1}{c|}{}                         & \multicolumn{1}{c|}{5e-1} & 2.90 / 98.98 & 2.89 / 99.04 & 2.73 / 99.06 & 67.48 / 90.50 & 94.64 / 56.57  \\ \cline{2-7} 
\multicolumn{1}{c|}{}                         & \multicolumn{1}{c|}{1}    & 2.81 / 98.96 & 2.89 / 99.04 & 2.82 / 99.04 & 90.97 / 70.65 & 87.12 / 56.72 \\ \cline{2-7} 
\multicolumn{1}{c|}{}                         & \multicolumn{1}{c|}{5}    & 2.73 / 98.98 & 2.90 / 99.03 & 14.04 / 95.64 & 93.01 / 45.50 & 88.61 / 69.65 \\ \midrule
\multicolumn{7}{c}{\cellcolor{greyC} $\gamma_\text{max}$=1}                                                                                                                                                        \\ \midrule
\multicolumn{2}{c}{\multirow{2}{*}{}} & \multicolumn{5}{|c}{$\rho$} \\ \cline{3-7} 
\multicolumn{2}{c|}{} & \multicolumn{1}{c|}{1e-2} & \multicolumn{1}{c|}{1e-1} & \multicolumn{1}{c|}{1} & \multicolumn{1}{c|}{10}  & 100\\ \hline
\multicolumn{1}{c|}{\multirow{8}{*}{$\beta$}} & \multicolumn{1}{c|}{1e-3} & 2.74 / 98.92 & \textbf{2.62} / 98.95 & 2.64 / 98.93 & 2.70 / 98.93 & 92.36 / 57.40  \\ \cline{2-7} 
\multicolumn{1}{c|}{}                         & \multicolumn{1}{c|}{5e-3} & 2.64 / 99.00 & 2.78 / 98.98 & 2.86 / 98.82 & 2.68 / 98.16 & 94.72 / 58.59  \\ \cline{2-7} 
\multicolumn{1}{c|}{}                         & \multicolumn{1}{c|}{1e-2} & 2.69 / 99.02 & 2.69 / 99.01 & 2.65 / 98.97 & 2.79 / 98.19 & 93.39 / 57.55  \\ \cline{2-7} 
\multicolumn{1}{c|}{}                         & \multicolumn{1}{c|}{5e-2} & 2.67 / 98.94 & 2.75 / 98.93 & 2.71 / 98.97 & 2.74 / 98.64 & 94.60 / 52.14  \\ \cline{2-7} 
\multicolumn{1}{c|}{}                         & \multicolumn{1}{c|}{1e-1} & 2.90 / 98.99 & 2.71 / \textbf{99.07} & 2.77 / 98.96 & 22.93 / 94.15 & 92.55 / 67.68  \\ \cline{2-7} 
\multicolumn{1}{c|}{}                         & \multicolumn{1}{c|}{5e-1} & 2.66 / 98.99 & 2.64 / 98.99 & 2.90 / 98.87 & 79.75 / 89.56 & 93.76 / 55.94  \\ \cline{2-7} 
\multicolumn{1}{c|}{}                         & \multicolumn{1}{c|}{1}    & 2.78 / 99.01 & 2.98 / 98.94 & 2.56 / 98.96 & 52.66 / 90.50 & 92.57 / 90.03 \\ \cline{2-7} 
\multicolumn{1}{c|}{}                         & \multicolumn{1}{c|}{5}    & 2.68 / 99.03 & 2.68 / 98.96 & 2.69 / 98.96 & 25.31 / 94.53 & 97.42 / 43.10  \\ \midrule
\multicolumn{7}{c}{\cellcolor{greyC} $\gamma_\text{max}$=0.1}                                                                                                                                                        \\\midrule
\multicolumn{2}{c}{\multirow{2}{*}{}} & \multicolumn{5}{|c}{$\rho$} \\ \cline{3-7} 
\multicolumn{2}{c|}{} & \multicolumn{1}{c|}{1e-2} & \multicolumn{1}{c|}{1e-1} & \multicolumn{1}{c|}{1} & \multicolumn{1}{c|}{10}  & 100 \\ \hline
\multicolumn{1}{c|}{\multirow{8}{*}{$\beta$}} & \multicolumn{1}{c|}{1e-3} & 2.75 / \textbf{98.59}                  & 53.32 / 94.12 & 82.19 / 92.03 & 2.52 / 98.41 & 95.24 / 51.41  \\ \cline{2-7} 
\multicolumn{1}{c|}{}                         & \multicolumn{1}{c|}{5e-3} & 2.83 / 98.55 & \textbf{2.46} / 98.55 & 91.80 / 59.84 & 90.23 / 70.99 & 95.65 / 69.73 \\ \cline{2-7} 
\multicolumn{1}{c|}{}                         & \multicolumn{1}{c|}{1e-2} & 2.47 / 98.48 & 2.56 / 98.55 & 87.64 / 88.24 & 97.56 / 71.35 & 88.87 / 77.29  \\ \cline{2-7} 
\multicolumn{1}{c|}{}                         & \multicolumn{1}{c|}{5e-2} & 91.10 / 91.70 & \textbf{2.46} / 98.53 & 87.24 / 90.28 & 90.04 / 81.11 & 96.49 / 44.84  \\ \cline{2-7} 
\multicolumn{1}{c|}{}                         & \multicolumn{1}{c|}{1e-1} & 87.87 / 91.66 & 3.00 / 95.54 & 86.08 / 71.34 & 89.70 / 72.59 & 73.47 / 79.13  \\ \cline{2-7} 
\multicolumn{1}{c|}{}                         & \multicolumn{1}{c|}{5e-1} & 90.51 / 84.74 & 92.42 / 92.06 & 91.78 / 90.73 & 89.48 / 81.07 & 94.27 / 59.21  \\ \cline{2-7} 
\multicolumn{1}{c|}{}                         & \multicolumn{1}{c|}{1}    & 93.82 / 89.51 & 92.29 / 90.44 & 96.46 / 83.90 & 79.04 / 86.78 & 94.06 / 50.42 \\ \cline{2-7} 
\multicolumn{1}{c|}{}                         & \multicolumn{1}{c|}{5}    & 80.22 / 91.03 & 87.43 / 89.21 & 91.98 / 59.04 & 91.55 / 63.98 &  94.18 / 61.55 \\
\bottomrule[1.5pt]
\end{tabular}}
\end{table}

\begin{table}[]
\caption{Detection Performance on CIFAR-$100$ dataset with different choices of $\beta$, $\rho$, and $\gamma_\text{max}$, where we report the FPR95 / AUROC for each individual trail setup.} \label{tab: gamma_cifar100}
\centering
\resizebox{.49\linewidth}{!}{
\begin{tabular}{ccccccc}
\toprule[1.5pt]
\multicolumn{7}{c}{\cellcolor{greyC} $\gamma_\text{max}$=50}                                                                                                                                                        \\ \midrule
\multicolumn{2}{c}{\multirow{2}{*}{}} & \multicolumn{5}{|c}{$\rho$} \\ \cline{3-7} 
\multicolumn{2}{c|}{} & \multicolumn{1}{c|}{1e-2} & \multicolumn{1}{c|}{1e-1} & \multicolumn{1}{c|}{1} & \multicolumn{1}{c|}{10} & 100\\ \hline
\multicolumn{1}{c|}{\multirow{8}{*}{$\beta$}} & \multicolumn{1}{c|}{1e-3} & 31.72 / 93.37 & 30.84 / 93.51 & 31.50 / 93.27 & 31.08 / 92.47 & 87.70 / 86.84  \\ \cline{2-7} 
\multicolumn{1}{c|}{}                         & \multicolumn{1}{c|}{5e-3} & 33.36 / 92.70 & 30.23 / 93.61 & 31.27 / 93.32 & 29.98 / 92.94 & 96.98 / 36.34 \\ \cline{2-7} 
\multicolumn{1}{c|}{}                         & \multicolumn{1}{c|}{1e-2} & 31.73 / 93.32 & 35.42 / 92.42 & 30.86 / 93.61 & 31.22 / 92.34 & 87.17 / 86.80  \\ \cline{2-7} 
\multicolumn{1}{c|}{}                         & \multicolumn{1}{c|}{5e-2} & 33.69 / 92.66 & 32.50 / 93.18 &31.29 / 93.40 & 30.92 / 92.68 & 87.23 / 86.57  \\ \cline{2-7} 
\multicolumn{1}{c|}{}                         & \multicolumn{1}{c|}{1e-1} & 32.21 / 93.18 & 32.42 / 93.31 & 33.77 / 92.98 & 29.48 / 93.04 & 87.14 / 87.27  \\ \cline{2-7} 
\multicolumn{1}{c|}{}                         & \multicolumn{1}{c|}{5e-1} & 33.44 / 92.93 & 36.19 / 92.71 & 33.55 / 92.85 & 34.83 / 91.70 & 86.91 / 87.22  \\ \cline{2-7} 
\multicolumn{1}{c|}{}                         & \multicolumn{1}{c|}{1}    & 33.01 / 93.16 & 33.80 / 92.64 & \textbf{28.99} / \textbf{93.87} & 36.22 / 91.41 & 92.34 / 52.72   \\ \cline{2-7} 
\multicolumn{1}{c|}{}                         & \multicolumn{1}{c|}{5}    & 34.62 / 92.55 & 32.90 / 93.15 & 36.00 / 92.14 & 95.59 / 89.09 & 93.69 / 96.87 \\  \midrule
\multicolumn{7}{c}{\cellcolor{greyC} $\gamma_\text{max}$=5}  \\ \midrule
\multicolumn{2}{c}{\multirow{2}{*}{}} & \multicolumn{5}{|c}{$\rho$} \\ \cline{3-7} 
\multicolumn{2}{c|}{} & \multicolumn{1}{c|}{1e-2} & \multicolumn{1}{c|}{1e-1} & \multicolumn{1}{c|}{1} & \multicolumn{1}{c|}{10} & 100 \\ \hline
\multicolumn{1}{c|}{\multirow{8}{*}{$\beta$}} & \multicolumn{1}{c|}{1e-3} & 33.58 / 92.84 & 34.61 / 92.35 & 32.90 / 92.91 & 30.55 / 92.66 & 95.80 / 54.59  \\ \cline{2-7} 
\multicolumn{1}{c|}{}                         & \multicolumn{1}{c|}{5e-3} & 33.80 / 92.87 & 34.98 / 92.61 & 31.87 / 93.39 & \textbf{27.36} / 93.18 & 89.37 / 51.52  \\ \cline{2-7} 
\multicolumn{1}{c|}{}                         & \multicolumn{1}{c|}{1e-2} & 36.96 / 92.48 & 34.87 / 93.01 & 30.87 / 93.01 & 30.57 / 92.82 & 85.67 / 88.23 \\ \cline{2-7} 
\multicolumn{1}{c|}{}                         & \multicolumn{1}{c|}{5e-2} & 36.00 / 92.49 & 32.10 / 93.24 & 31.31 / \textbf{93.55} & 30.70 / 92.95 & 85.80 / 86.38 \\ \cline{2-7} 
\multicolumn{1}{c|}{}                         & \multicolumn{1}{c|}{1e-1} & 33.64 / 92.75 & 31.82 / 93.09 & 32.55 / 92.97 & 32.42 / 92.22 & 91.16 / 84.87 \\ \cline{2-7} 
\multicolumn{1}{c|}{}                         & \multicolumn{1}{c|}{5e-1} & 34.03 / 92.69 & 33.03 / 93.09 & 32.99 / 93.19 & 81.64 / 89.24 & 89.67 / 87.08  \\ \cline{2-7} 
\multicolumn{1}{c|}{}                         & \multicolumn{1}{c|}{1}    & 34.85 / 92.61 & 32.73 / 93.31 & 29.53 / 93.74 & 63.16 / 90.50 & 87.71 / 87.06 \\ \cline{2-7} 
\multicolumn{1}{c|}{}                         & \multicolumn{1}{c|}{5}    & 34.59 / 92.69 & 33.14 / 93.16 & 72.10 / 90.79 & 40.65 / 91.30 & 89.45 / 86.70  \\ \midrule
\multicolumn{7}{c}{\cellcolor{greyC} $\gamma_\text{max}$=0.5}                                                                                                                                                        \\\midrule
\multicolumn{2}{c}{\multirow{2}{*}{}} & \multicolumn{5}{|c}{$\rho$} \\ \cline{3-7} 
\multicolumn{2}{c|}{} & \multicolumn{1}{c|}{1e-2} & \multicolumn{1}{c|}{1e-1} & \multicolumn{1}{c|}{1} & \multicolumn{1}{c|}{10} & 100\\ \hline
\multicolumn{1}{c|}{\multirow{8}{*}{$\beta$}} & \multicolumn{1}{c|}{1e-3} & 33.10 / 92.59 & 32.25 / 92.84 & 30.40 / 92.97 & 31.70 / 92.39 & 87.47 / 86.53  \\ \cline{2-7} 
\multicolumn{1}{c|}{}                         & \multicolumn{1}{c|}{5e-3} & 34.15 / 92.33 & 31.80 / 92.33 & 31.97 / 92.49 & 31.75 / 92.48 & 90.14 / 86.43 \\ \cline{2-7} 
\multicolumn{1}{c|}{}                         & \multicolumn{1}{c|}{1e-2} & 34.35 / 92.28 & 33.90 / 92.29 & 29.72 / \textbf{93.29} & \textbf{29.66} / 92.75 & 88.49 / 86.76  \\ \cline{2-7} 
\multicolumn{1}{c|}{}                         & \multicolumn{1}{c|}{5e-2} & 32.61 / 93.04 & 33.19 / 92.41 & 33.73 / 92.23 & 30.28 / 92.99 & 84.51 / 88.00  \\ \cline{2-7} 
\multicolumn{1}{c|}{}                         & \multicolumn{1}{c|}{1e-1} & 33.23 / 92.48 & 35.78 / 91.90 & 33.89 / 91.97 & 75.36 / 90.70 & 96.66 / 57.33   \\ \cline{2-7} 
\multicolumn{1}{c|}{}                         & \multicolumn{1}{c|}{5e-1} & 32.48 / 92.33 & 30.56 / 93.19 & 32.74 / 92.37 & 84.06 / 88.13 & 91.47 / 46.58  \\ \cline{2-7} 
\multicolumn{1}{c|}{}                         & \multicolumn{1}{c|}{1}    & 33.90 / 92.25 & 32.89 / 92.52 & 31.26 / 92.89 & 62.99 / 90.64 & 88.53 / 87.17  \\ \cline{2-7} 
\multicolumn{1}{c|}{}                         & \multicolumn{1}{c|}{5}    & 32.62 / 92.33 & 32.60 / 92.48 & 30.04 / 92.94 & 71.77 / 90.34 & 96.58 / 49.03  \\ 
\bottomrule[1.5pt]
\end{tabular}}
\resizebox{.49\linewidth}{!}{
\begin{tabular}{ccccccc}
\toprule[1.5pt]
\multicolumn{7}{c}{\cellcolor{greyC} $\gamma_\text{max}$=10}                                                                                                                                                        \\ \midrule
\multicolumn{2}{c}{\multirow{2}{*}{}} & \multicolumn{5}{|c}{$\rho$} \\ \cline{3-7} 
\multicolumn{2}{c|}{} & \multicolumn{1}{c|}{1e-2} & \multicolumn{1}{c|}{1e-1} & \multicolumn{1}{c|}{1} & \multicolumn{1}{c|}{10} & 100\\ \hline
\multicolumn{1}{c|}{\multirow{8}{*}{$\beta$}} & \multicolumn{1}{c|}{1e-3} & 33.07 / 92.92 & 33.93 / 92.85 & 35.09 / 92.44 & 30.39 / 92.59 & 94.25 / 48.00  \\ \cline{2-7} 
\multicolumn{1}{c|}{}                         & \multicolumn{1}{c|}{5e-3} & 34.03 / 92.81 & 33.64 / 92.88 & 30.43 / 93.66 & 33.75 / 92.00 & 96.58 / 48.51  \\ \cline{2-7} 
\multicolumn{1}{c|}{}                         & \multicolumn{1}{c|}{1e-2} & 35.80 / 92.46 & 33.40 / 93.16 & 34.52 / 92.90 & 32.02 / 91.76 & 88.11 / 86.65   \\ \cline{2-7} 
\multicolumn{1}{c|}{}                         & \multicolumn{1}{c|}{5e-2} & 33.38 / 93.19 & 32.95 / 93.29 & 29.74 / \textbf{93.63} & 29.13 / 92.71 & 95.18 / 50.19  \\ \cline{2-7} 
\multicolumn{1}{c|}{}                         & \multicolumn{1}{c|}{1e-1} & 31.77 / 93.33 & 31.94 / 93.31 & 35.34 / 92.78 & \textbf{27.80} / 93.02 & 95.48 / 56.81  \\ \cline{2-7} 
\multicolumn{1}{c|}{}                         & \multicolumn{1}{c|}{5e-1} & 34.12 / 92.71 & 35.26 / 92.86 & 31.71 / 93.23 & 79.79 / 89.56 & 89.26 / 86.60 \\ \cline{2-7} 
\multicolumn{1}{c|}{}                         & \multicolumn{1}{c|}{1}    & 32.90 / 93.05 & 34.13 / 92.86 & 31.64 / 93.05 & 47.37 / 91.31 & 96.10 / 46.88\\ \cline{2-7} 
\multicolumn{1}{c|}{}                         & \multicolumn{1}{c|}{5}    & 32.65 / 93.13 & 34.41 / 92.86 & 33.02 / 91.80 & 57.19 / 90.72 & 89.12 / 87.75 \\ \midrule
\multicolumn{7}{c}{\cellcolor{greyC} $\gamma_\text{max}$=1}                                                                                                                                                        \\ \midrule
\multicolumn{2}{c}{\multirow{2}{*}{}} & \multicolumn{5}{|c}{$\rho$} \\ \cline{3-7} 
\multicolumn{2}{c|}{} & \multicolumn{1}{c|}{1e-2} & \multicolumn{1}{c|}{1e-1} & \multicolumn{1}{c|}{1} & \multicolumn{1}{c|}{10}  & 100\\ \hline
\multicolumn{1}{c|}{\multirow{8}{*}{$\beta$}} & \multicolumn{1}{c|}{1e-3} & 37.24 / 91.87 & 30.42 / 93.44 & 33.45 / 92.88 & 30.17 / 92.77 & 95.60 / 57.62  \\ \cline{2-7} 
\multicolumn{1}{c|}{}                         & \multicolumn{1}{c|}{5e-3} & 32.75 / 92.88 & 32.24 / 93.05 & 32.13 / 92.79 & 31.37 / 92.68 & 95.96 / 48.31  \\ \cline{2-7} 
\multicolumn{1}{c|}{}                         & \multicolumn{1}{c|}{1e-2} & 36.66 / 91.81 & 30.40 / 93.45 & 29.47 / 93.54 & 31.14 / 92.37 & 88.33 / 86.39 \\ \cline{2-7} 
\multicolumn{1}{c|}{}                         & \multicolumn{1}{c|}{5e-2} & 31.00 / 93.39 & 30.88 / \textbf{93.46} & 30.33 / 92.97 & 31.43 / 92.50 & 87.83 / 88.46  \\ \cline{2-7} 
\multicolumn{1}{c|}{}                         & \multicolumn{1}{c|}{1e-1} & 31.18 / 93.42 & 33.03 / 93.08 & 31.98 / 93.17 & 52.46 / 90.95 & 95.97 / 53.86 \\ \cline{2-7} 
\multicolumn{1}{c|}{}                         & \multicolumn{1}{c|}{5e-1} & 35.14 / 92.79 & 29.55 / 93.34 & \textbf{29.19} / 93.34 & 82.19 / 89.22 & 89.30 / 85.40  \\ \cline{2-7} 
\multicolumn{1}{c|}{}                         & \multicolumn{1}{c|}{1}    & 36.93 / 92.33 & 34.03 / 93.00 & 35.45 / 91.74 & 77.26 / 90.46 & 97.33 / 48.51 \\ \cline{2-7} 
\multicolumn{1}{c|}{}                         & \multicolumn{1}{c|}{5}    & 31.00 / 93.37 & 30.48 / 93.37 & 31.17 / 93.21 & 83.14 / 89.83 & 97.43 / 50.74  \\ \midrule
\multicolumn{7}{c}{\cellcolor{greyC} $\gamma_\text{max}$=0.1}                                                                                                                                                        \\\midrule
\multicolumn{2}{c}{\multirow{2}{*}{}} & \multicolumn{5}{|c}{$\rho$} \\ \cline{3-7} 
\multicolumn{2}{c|}{} & \multicolumn{1}{c|}{1e-2} & \multicolumn{1}{c|}{1e-1} & \multicolumn{1}{c|}{1} & \multicolumn{1}{c|}{10}  & 100 \\ \hline
\multicolumn{1}{c|}{\multirow{8}{*}{$\beta$}} & \multicolumn{1}{c|}{1e-3} & 93.84 / 50.41 & 99.39 / 51.61 & 76.69 / 89.44 & 75.19 / 89.36 & 97.09 / 44.87  \\ \cline{2-7} 
\multicolumn{1}{c|}{}                         & \multicolumn{1}{c|}{5e-3} & 81.55 / 87.48 & 80.81 / 88.63 & 75.63 / 90.32 & 96.69 / 43.59 &  96.15 / 50.92 \\ \cline{2-7} 
\multicolumn{1}{c|}{}                         & \multicolumn{1}{c|}{1e-2} & 95.32 / 51.45 & 91.92 / 55.05 & 75.40 / 89.75 & 96.55 / 45.81 & 95.24 / 48.11 \\ \cline{2-7} 
\multicolumn{1}{c|}{}                         & \multicolumn{1}{c|}{5e-2} & \textbf{42.65} / 90.84 & 93.80 / 45.56 & 99.58 / 50.06 & 97.08 / 45.02 & 83.93 / 87.53  \\ \cline{2-7} 
\multicolumn{1}{c|}{}                         & \multicolumn{1}{c|}{1e-1} & 95.32 / 51.45 & 91.92 / 55.05 & 75.40 / 89.75 & 96.55 / 45.81 & 95.24 / 48.11  \\ \cline{2-7} 
\multicolumn{1}{c|}{}                         & \multicolumn{1}{c|}{5e-1} & 90.51 / 84.74 & 92.42 / \textbf{92.06} & 91.78 / 90.73 & 89.48 / 81.07 & 94.27 / 59.21 \\ \cline{2-7} 
\multicolumn{1}{c|}{}                         & \multicolumn{1}{c|}{1}    & 81.14 / 90.18 & 82.34 / 89.98 & 80.38 / 90.42 & 88.00 / 53.88 & 95.43 / 46.75 \\ \cline{2-7} 
\multicolumn{1}{c|}{}                         & \multicolumn{1}{c|}{5}    & 96.57 / 45.61 & 98.16 / 47.34 & 100.0 / 49.86 & 94.89 / 51.80 & 92.12 / 87.10 \\
\bottomrule[1.5pt]
\end{tabular}}
\end{table}

\begin{table}[t]
\centering
\parbox{.48\linewidth}{
\centering
    \caption{The hyper-parameter effects of $\texttt{num}\_\texttt{search}$ on the CIFAR benchmarks.} \label{tab: hyper_search}
    \scriptsize\resizebox{.98\linewidth}{!}{
    \begin{tabular}{c|ccccccccc}
    \toprule[1.5pt]
    & 0 & 1 & 2 & 5 & 10 & 20 & 50 & 100 & 200 \\
    \midrule 
    \multicolumn{10}{c}{\cellcolor{greyC} CIFAR-10} \\
    \midrule[0.6pt]
    FPR95 & 3.33 & 2.90 & \textbf{2.41} & 2.61 & 2.62 & 2.46 & 2.74 & 2.86 &  3.00 \\
    AUROC & 98.59 & \textbf{99.10} & 98.96 & 98.91 & 98.92 & 98.56 & 98.95 & 99.07 & 98.80 \\  
    \midrule[0.6pt]
    \multicolumn{10}{c}{\cellcolor{greyC} CIFAR-100} \\
    \midrule[0.6pt]
    FPR95 & 36.47 & 34.12 & 33.55 & 33.30 & \textbf{30.38} & 31.27 & 32.01 & 33.07 & 31.73  \\
    AUROC & 91.75 & 92.60 & 92.98 & 93.14 & \textbf{93.62} & 93.36 & 93.18 & 92.91 & 93.22  \\    
    \bottomrule[1.5pt]  
    \end{tabular}}}
\centering
\parbox{.48\linewidth}{
    \centering
    \caption{The hyper-parameter effects of $\texttt{ps}$ on the CIFAR benchmarks.} \label{tab: hyper_ps}
    \scriptsize
    \resizebox{.98\linewidth}{!}{\begin{tabular}{c|ccccccccc}
    \toprule[1.5pt]
     & $1e^{-2}$ & $5e^{-2}$ & $1e^{-1}$ & $5e^{-1}$ & 1 & 5 & 10 & 50 & 100 \\
    \midrule 
    \multicolumn{10}{c}{\cellcolor{greyC} CIFAR-10} \\
    \midrule[0.6pt]
    FPR95 & 2.97 & 2.76 & 2.80 & \textbf{2.49} & 2.57 & 2.92 & 3.01 & 3.04 & 2.92 \\
    AUROC & 99.00 & \textbf{99.02} & 98.95 & 98.94 & 98.82 & 98.90 & 98.97 & 98.81 & 98.30 \\    
    \midrule[0.6pt]
    \multicolumn{10}{c}{\cellcolor{greyC} CIFAR-100} \\
    \midrule[0.6pt]
    FPR95 & 35.74 & 35.75 & 32.64 & \textbf{29.00} & 31.03 & 33.63 & 32.93 & 37.61 & 95.07 \\
    AUROC & 92.82 & 92.45 & 93.14 & \textbf{93.95} & 93.18 & 92.74 & 92.74 & 93.09 & 91.17 \\  
    \bottomrule[1.5pt]  
    \end{tabular}}
}
\end{table}
\clearpage

\subsection{Aligning Training Epochs}

\begin{wraptable}{r}{0.45\textwidth}
\caption{Comparison between OE and DAL with 50 epochs training.  } \label{tab: 50eps}
\vspace{7pt}
\scriptsize
{
\begin{tabular}{c|cc|cc}
\toprule[1.5pt]
\multirow{2}{*}{} & \multicolumn{2}{c|}{CIFAR-10} & \multicolumn{2}{c}{CIFAR-100} \\ \cline{2-5}
                  & FPR95  $\downarrow$         & AUROC    $\uparrow$     & FPR95    $\downarrow$       & AUROC   $\uparrow$       \\
                  \midrule
OE                & 3.07          & 98.97        & 37.35         & 92.00         \\
DAL               & \textbf{2.68}          & \textbf{99.01}        & \textbf{29.68}         & \textbf{93.92}         \\ \bottomrule[1.5pt]
\end{tabular}
}
\end{wraptable}

In our experiments, we follow the suggested hyper-parameters for the baselines, running OE with 10 epochs on the CIFAR benchmarks. However, our DAL, due to distribution augmentation, is run for 50 epochs to fully fit the augmented distribution. To demonstrate that our improvement is not dominated by longer training time, we also list the results of OE with 50 epochs training, summarizing the results on the CIFAR benchmarks in Table~\ref{tab: 50eps}. As we can see, although OE can produce better results with 50 epochs of training, our DAL can still demonstrate its superiority in OOD detection. For example, on CIFAR-100, our DAL improves OE by $7.67$ w.r.t. FPR95 and $1.92$ w.r.t. AUROC.

\subsection{Other Norms}

\begin{wraptable}{r}{0.55\textwidth}
\caption{Using $\ell_1$ and $\ell_2$ norms.  } \label{tab: different norms}
\vspace{7pt}
\scriptsize
{
\begin{tabular}{c|cc|cc}
\toprule[1.5pt]
          & \multicolumn{2}{c|}{$\ell_1$ norm} & \multicolumn{2}{c}{$\ell_2$ norm} \\ \cline{2-5}
          & FPR95 $\downarrow$ & AUROC $\uparrow$ & FPR95 $\downarrow$ & AUROC  $\uparrow$ \\ \midrule[1pt]
CIFAR-10  & \textbf{2.68}  & \textbf{99.01} & 2.81  & 98.98 \\
CIFAR-100 & \textbf{29.68} & 93.92 & 30.20 & \textbf{93.95} \\ \bottomrule[1.5pt] 
\end{tabular}}
\end{wraptable}

We can also use the $l_2$ norm and the associated  Wasserstein-2 distance. Therefore, we conduct the related experiments on the CIFAR benchmarks in comparing between $l_1$ and $l_2$ norms, and the results are summarized in Table~\ref{tab: different norms}. As we can see, we do not observe an obvious difference between using $\ell_1$ and $\ell_2$ norms, so it is proper to use the $\ell_1$ norm and the Wasserstein-1 distance in our DAL by default. 

\subsection{Linear Probing}

\begin{wraptable}{r}{0.42\textwidth}
\caption{Comparison between fully fine-tuning and linear probing.  } \label{tab:linear probe}
\vspace{7pt}
\scriptsize
\centering
{
\begin{tabular}{c|c|c}
\toprule[1.5pt]
FPR95 $\downarrow$ & linear probe & fine tune \\ \midrule[1pt]
OE     & 50.09        & 43.14     \\
DAL    & \textbf{43.37}        & \textbf{29.68}      \\ \bottomrule[1.5pt] 
\end{tabular}
}
\end{wraptable}

In many applications, the costs of re-training and re-deployment can be prohibitively high, where we should assume a fixed feature extractor $\mathbf{e}$ and allow only the classifier $\mathbf{h}$ (i.e., the fully connected layer) to be tuned. DAL is also adaptable for such a restricted setting, with improved detection performance over the OE counterpart. Table~\ref{tab:linear probe} summarizes the results on CIFAR-100, comparing OE and DAL under the settings of full training (fine tuning) and training with only the classifier (linear probe). As we can see, for the linear probe setup, DAL can still improve the OE counterpart, while the performance gain is largely limited compared to that of the full training.

\subsection{False Negative Rate}

\begin{wraptable}{r}{0.45\textwidth}
\caption{Experiments measured by FNR95.} \label{tab:fnr}
\vspace{7pt}
\scriptsize
\centering
{
\begin{tabular}{c|c|c|c}
\toprule[1.5pt]
FNR95 $\downarrow$ & CIFAR-10 & CIFAR-100 & \makecell{CIFAR-10 vs. \\ CIFAR-100} \\ \midrule[1pt]
MSP    & 33.02        & 64.83    & 43.01 \\
OE     & 5.04        &  41.31    & 26.38 \\ 
DAL    & \textbf{3.89}        & \textbf{26.87}  & \textbf{22.81}     \\ \bottomrule[1.5pt] 
\end{tabular}
}
\end{wraptable}

We further consider the metric of false negative rate (FNR95) for ID data when the true positive rate of ID data is at 95\%. We summarize the results on the CIFAR benchmarks in Table~\ref{tab:fnr}, where we consider the common OOD detection setups as in Table~\ref{tab: full} and the challenging CIFAR-10 vs. CIFAR-100 setup as in Table~\ref{tab: cifar hard}. As we can see, the FNR decreases for all three considered cases, further demonstrating the effectiveness of our method.

\clearpage
\subsection{ImageNet Evaluations}
\label{app: imagenet results}
We also conduct experiments on the ImageNet benchmarks, demonstrating the effectiveness of our DAL when facing this very challenging OOD detection task.

\textbf{OOD Datasets.} We adopt a subset of ImageNet-21K-P dataset~\citep{ridnik2021imagenet} as the auxiliary OOD data, which is cleansed to avoid those classes that coincide with ID cases. Furthermore, iNaturalist~\citep{HornASCSSAPB18}, SUN~\citep{xu2015turkergaze}, Places$365$, and Textures are adopted as the real OOD datasets, where we eliminate those data whose labels coincide with ID cases.

\textbf{Hyper-parameter Selection.} The hyper-parameters are also tuned on the validation data. We adopt the random search that follows the following three steps. Step 1: we randomly select a hyper-parameter (e.g., $\beta$) and fix the values of all other hyper-parameters to be their current optimal values. Step 2: we choose the best $\beta$ from the candidate set. Step 3: do Steps 1-2 again. We repeat Steps 1 and 2 for 50 times in our experiments. For the backbone model, we use ResNet-50 with the pre-trained parameters published by the PyTorch official repository. 

\textbf{Hyper-parameters Setups.} Our DAL is run for $5$ epochs. The batch size is $64$ for both the ID and the OOD cases. We have the initial learning rate $1e^{-4}$, $\gamma_\text{max}=10$, $\beta=0.5$, $\rho=0.1$, and $\texttt{ps}=0.1$. Furthermore, we employ cosine decay~\citep{LoshchilovH17} for the learning rate.

\textbf{ImageNet evaluations.} Due to the large semantic space and complex image patterns, OOD detection on the ImageNet dataset is a challenging task~\citep{HuangL21}. However, similar to the CIFAR benchmarks, DAL can also reveal the best detection performance over all the considered baseline methods. Moreover, it is well-known that MSP scoring can easily fail on the ImageNet benchmark~\citep{hendrycks2022scaling}, so we also report the results after DAL training using ASH (DAL-ASH) and Free Energy (DAL-Energy), which can further improve the detection performance.

\begin{table*}[t]
\centering
\caption{Comparison between our method and advanced methods on ImageNet.  $\downarrow$ (or $\uparrow$) indicates smaller (or larger) values are preferred, and a bold font indicates the best results in the column. } \vspace{3pt}
\label{tab: imagenet full}
\resizebox{0.9\linewidth}{!}{
\begin{tabular}{c|cccccccccc}
\toprule[1.5pt]
\multirow{2}{*}{Method} & \multicolumn{2}{c}{Textures} & \multicolumn{2}{c}{Places365} & \multicolumn{2}{c}{iNaturalist} & \multicolumn{2}{c}{SUN} & \multicolumn{2}{c}{\textbf{Average}} \\
\cline{2-11}
& FPR95 $\downarrow$ & AUROC $\uparrow$ & FPR95 $\downarrow$ & AUROC $\uparrow$ & FPR95 $\downarrow$ & AUROC $\uparrow$ & FPR95 $\downarrow$ & AUROC $\uparrow$ & FPR95 $\downarrow$ & AUROC $\uparrow$ \\
\midrule[0.8pt]
\multicolumn{11}{c}{Using ID data only} \\ \hline
MSP         & 66.58 & 80.03 & 74.15 & 78.97 & 72.72 & 77.19 & 78.70 & 75.15 & 73.04 & 77.84 \\
Free Energy & 52.84 & 86.36 & 70.64 & 81.67 & 73.98 & 75.97 & 76.92 & 78.08 & 68.60 & 80.52 \\
ASH         & 15.93 & 96.00 & 63.08 & 82.43 & 52.05 & 83.67 & 71.68 & 77.71 & 50.68 & 85.35 \\
Mahalanobis & 40.52 & 91.41 & 97.10 & 53.11 & 96.15 & 53.62 & 96.95 & 52.74 & 82.68 & 62.72 \\
KNN         & 26.54 & 93.49 & 78.64 & 76.82 & 75.78 & 69.51 & 74.30 & 78.85 & 63.82 & 79.66 \\
VOS         & 94.83 & 57.69 & 98.72 & 38.50 & 87.75 & 65.65 & 70.20 & 83.62 & 87.87 & 61.36 \\
\midrule[0.8pt]
\multicolumn{11}{c}{Using ID data and auxiliary OOD data} \\ \hline
OE          & 57.34 & 82.97 &  7.92 & 98.04 & 73.87 & 76.94 & 52.60 & 77.31 & 52.60 & 83.81 \\
Energy-OE   & 42.46 & 88.27 &  1.88 & 99.49 & 73.81 & \textbf{78.34} & 69.45 & 79.54 & 46.90 & 86.41 \\
ATOM        & 60.20 & 90.60 &  7.07 & 98.25 & 74.30 & 77.00 & 55.87 & 75.80 & 49.36 & 85.41 \\
DOE         & 35.11 & 92.15 &  0.72 & 99.79 & 72.55 & 78.00 & 59.06 & 85.67 & 41.86 & 88.90\\
POEM        & 40.80 & 89.78 &  0.26 & 99.70 & 73.23 & 68.83 & 65.45 & 82.08 & 44.93 & 85.10\\
\hline
% \rowcolor[HTML]{EFEFEF}  
DAL         & 55.49 & 85.29 &  5.83 & 99.09 & 74.23 & 76.70 & 50.76 & 79.21 & 46.57 & 85.08 \\
DAL-ASH     & \textbf{14.10} & \textbf{97.00} &  \textbf{0.23} & \textbf{99.85} & \textbf{67.38} & 78.20 & \textbf{45.14} & \textbf{85.90} & \textbf{31.71} & \textbf{90.24} \\
DAL-Energy  & 33.83 & 90.44 &  0.47 & 99.82 & 74.37 & 67.68 & 49.12 & 80.28 & 39.45 & 84.55 \\
\bottomrule[1.5pt]   
\end{tabular}}
\end{table*}

\end{document}